\documentclass{article}

\usepackage[preprint]{neurips_2026}
\usepackage[utf8]{inputenc}
\usepackage[T1]{fontenc}
\usepackage[dvipsnames]{xcolor}
\usepackage{hyperref}
\definecolor{LinkNavy}{RGB}{28,64,128}
\definecolor{CiteNavy}{RGB}{28,64,128}
\hypersetup{
  colorlinks=true,
  linkcolor=LinkNavy,
  citecolor=CiteNavy,
  urlcolor=LinkNavy,
  filecolor=LinkNavy,
  pdfborder={0 0 0},
  breaklinks=true,
}
\usepackage{url}
\usepackage{booktabs}
\usepackage{amsfonts}
\usepackage{amsmath}
\usepackage{amssymb}
\usepackage{amsthm}
\usepackage{nicefrac}
\usepackage{microtype}
\usepackage{graphicx}
\usepackage{multirow}
\usepackage{colortbl}
\usepackage{array}
\usepackage{tcolorbox}
\usepackage[ruled,vlined,linesnumbered]{algorithm2e}
\usepackage{caption}
\usepackage{wrapfig}
\usepackage{enumitem}

\SetAlgoInsideSkip{medskip}
\SetAlgoSkip{smallskip}
\DontPrintSemicolon
\SetAlgoNlRelativeSize{-1}
\SetNlSty{tiny}{}{}
\SetKwInput{KwIn}{Input}
\SetKwInput{KwOut}{Output}
\SetCommentSty{itshape\small}
\newcommand{\cmt}[1]{\hfill{\small\itshape$\triangleright$~#1}}
\SetInd{0.4em}{0.6em}

\graphicspath{{Figures/}}

\usepackage{titlesec}
\titlespacing*{\paragraph}{0pt}{0.6ex plus 0.2ex minus 0.1ex}{0.5em}
\theoremstyle{plain}
\newtheorem{proposition}{Proposition}
\theoremstyle{definition}

\title{When and How Long? The Readout-Mediator Angle in Temporal Reasoning}

\author{
  Shreyas Fadnavis \\
  Bioscope AI \\
  \texttt{shreyas.fadnavis@bioscope.ai} \\
  \And
  Praitayini Kanakaraj \\
  Bioscope AI \\
  \texttt{praitayini.kanakaraj@bioscope.ai} \\
  \And
  Felix Wyss \\
  Bioscope AI \\
  \texttt{felix.wyss@bioscope.ai}
}

\begin{document}

\maketitle

\begin{abstract}
A linear probe can decode a representation almost perfectly
and yet be completely irrelevant to how the model uses it.
On calendar-date duration reasoning in language models,
a $\sin$/$\cos$ probe recovers day-of-year from a layer's activations,
yet ablating its direction has no effect on the model's answers---while
ablating a four-dimensional subspace found by Distributed Alignment
Search (DAS) at the same layer collapses performance entirely.
We measure the angle between these two subspaces---the
\emph{readout--mediator angle}---and find it indistinguishable
from the angle between two random subspaces (the Haar-uniform null),
meaning the probe has learned a direction orthogonal to the
model's actual computation.
Reverse-engineering the circuit reveals why:
attention heads route month-grained context through learned QK offsets
at ${\pm}30$ and ${\pm}61$ days, and MLPs then convert
\emph{when} (absolute date) into \emph{how long} (duration)---all
downstream of the causal subspace the probe never touches.
Sparse-autoencoder decomposition confirms the split:
probe-aligned and DAS-aligned features encode semantically
disjoint concepts with negligible causal overlap.
The dissociation replicates across four scales
($1.5$--$9$\,B) and two model families, with preliminary evidence
on two further domains
(spatial displacement, symbolic arithmetic),
suggesting that readout--mediator orthogonality is a general
failure mode of probe-based interpretability.
This directly undermines proposals to deploy probes as
runtime safety monitors: the probe can report high confidence
on a direction the model has silently abandoned.
\end{abstract}

\section{Introduction}
\label{sec:intro}

Ask a language model ``How many days between March~15\textsuperscript{th}
and June~22\textsuperscript{nd}?'' and it answers correctly: $99$~days.
A $\sin/\cos$ Ridge probe at layer~20 decodes both dates from the
residual stream with $R^{2}{=}0.996$
\citep{gurnee2024spacetime}---exactly the kind of result used to argue
that the model \emph{represents} calendar time
\citep{gurnee2024spacetime,marks2024geometry,kantamneni2025linear}.
But ablating the probe's direction drops accuracy by only $0.6$~pp;
the model still counts $99$~days as if nothing happened.
A matched-rank Distributed Alignment Search (DAS) subspace at the
same layer tells the opposite story: ablating it collapses accuracy
entirely (Fig.~\ref{fig:overview}).

The probe reads the right answer from a direction the model does not
compute with.
We quantify this by measuring the principal angle between the two
subspaces---what we call the \emph{readout--mediator angle}.
At $\bar\theta{=}88^{\circ}$, it matches the expectation for two
random subspaces drawn from a Haar-uniform distribution
($\mathbb{E}[\bar\theta]{=}88.3^{\circ}$ at $(d,k){=}(2304,2)$,
Prop.~\ref{prop:null}):
the probe carries no more information about the model's computation
than a random direction of the same rank.
A five-year line of work has questioned whether probe accuracy implies
mechanistic relevance
\citep{hewitt2019control,elazar2021amnesic,ravichander2021probing,mueller2024quest,mueller2025mib,davies2025reliable};
the readout--mediator angle provides the missing instrument---a
number that says \emph{how far} the probe sits from the computation,
paired with a null that says what ``far'' means.

Reverse-engineering the causal subspace reveals why the two are
orthogonal.
Boundary attention heads implement QK offsets at ${\pm}30$ and
${\pm}61$ days---single- and double-month steps that tile any
multi-month duration (Fig.~\ref{fig:overview}).
MLP layers then execute a two-stage transformation: layers 18--19 read
calendar position (\emph{when}); layers 20--25 convert it into
duration (\emph{how long}), with MLP SAEs---functionally equivalent
to transcoders---showing a monotonic DAS-alignment gradient across
this boundary.
Sparse-autoencoder features confirm the split at the vocabulary level:
probe-aligned features fire on concepts like \textit{month of
October}; DAS-aligned features fire on \textit{past 24~hours}.
The two feature sets have zero causal overlap
(Supp.~\ref{supp:neuronpedia}).
Temporal feature analysis \citep{lubana2025priors} explains the
geometry: the DAS mediator aligns with context-predictable structure
($7{\times}$ above the Haar null), while the probe sits in the
random cloud---duration computation lives in the part of the
activation accumulated from context, not the current token
(Fig.~\ref{fig:complete_circuit}C--D).

The dissociation is not specific to one model, one scale, or one
domain.
It replicates across four models ($1.5$--$9$B), two architecture
families, and two further reasoning domains (spatial displacement,
symbolic arithmetic)---each at the Haar null angle.
On \textsc{Pythia~1.4B}, probe $R^2{=}0.956$ at step~$0$: an untrained
network ``represents'' dates by the probe's standard, yet boundary
heads, DAS drops, and circularness all emerge only as training
proceeds---the probe tracks dimensional capacity, not mechanism
learning.

\paragraph{Contributions:}
\begin{itemize}[leftmargin=1.5em,itemsep=1pt,topsep=1pt,parsep=0pt]
\item The readout--mediator angle and Haar-random null, with three propositions linking angle to ablation effect (\S\ref{sec:method}).
\item Maximal probe--DAS dissociation across four scales ($1.5$--$9$B), two families, three domains---sharpening with scale (\S\ref{sec:das},~\ref{sec:cross_task}).
\item A full circuit trace: boundary heads, two-stage MLP transcoder chain, disjoint SAE features, and a TFA-based explanation of the orthogonality (\S\ref{sec:mech}).
\item A six-experiment battery showing the dissociation breaks probe-based safety monitors (\S\ref{sec:cross_task}).
\end{itemize}
\vspace{-2pt}

\begin{figure}[t]
\centering
\includegraphics[width=\linewidth]{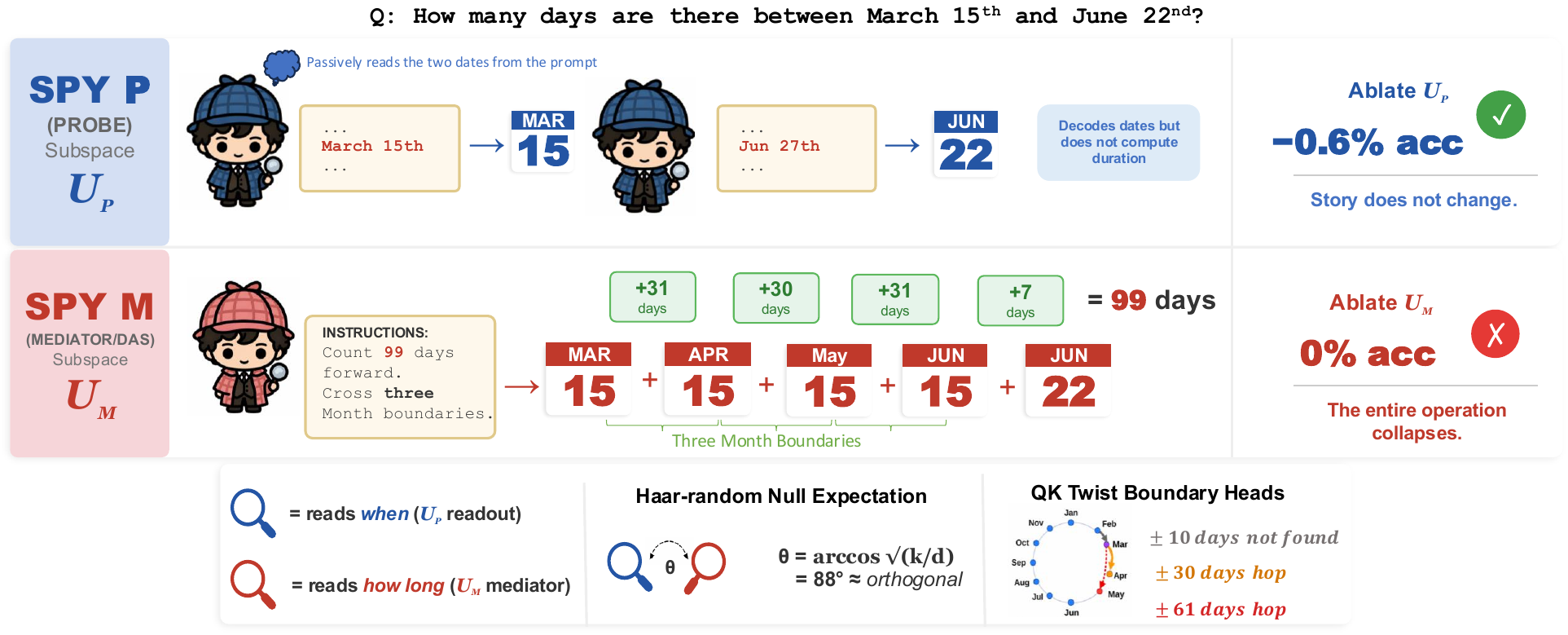}
\caption{\textbf{The readout-mediator dissociation on a duration query (schematic).}
Given ``How many days between March 15\textsuperscript{th} and June
22\textsuperscript{nd}?'', the model deploys two functionally
orthogonal subspaces.
\emph{Spy~P:} the probe subspace $\mathbf{U}_{P}$ passively decodes
both dates; ablating it changes accuracy by $-0.6$~pp.
\emph{Spy~M:} the DAS mediator $\mathbf{U}_{M}$ counts $99$~days via
month-boundary hops; ablating it collapses accuracy to $0\%$.
The two subspaces are nearly orthogonal
($\bar{\theta}{=}88^{\circ}$), matching the Haar-random null---the
probe carries no more information about the computation than noise.}
\label{fig:overview}
\end{figure}

\section{The decodability-use gap}
\label{sec:related}

The gap between what a probe can decode and what a model causally
uses is well documented.
\citet{gurnee2024spacetime} decode spatial and temporal coordinates
from \textsc{Llama-2} with $R^{2}{>}0.9$, yet note explicitly that
this ``does not imply the model actually uses these representations.''
\citet{hernandez2024linearity} make the gap concrete: they define a
faithfulness metric and find that many relations are probe-accurate
but not faithfully decoded at generation time.
Several lines of work have tried to narrow the gap---combining probes
with causal intervention
\citep{tak2025emotion,feng2025propositional}, replacing fixed probes
with decoder LLMs \citep{pan2025latentqa}, and questioning the
interventions themselves \citep{grant2026divergent}---but none
provide a single number that says how far the probe direction sits
from the causal one, or a null that says what ``far'' means.

Diagnostic critiques have sharpened the problem without solving it.
Control tasks \citep{hewitt2019control}, amnesic probing
\citep{elazar2021amnesic}, and recent audits
\citep{ravichander2021probing,mueller2024quest,mueller2025mib,davies2025reliable}
all question the inferential step from accuracy to mechanism, but
stop short of measuring the divergence.
Concept-erasure methods---INLP and LEACE
\citep{belrose2023leace}---attempt to close the gap by removing
the probed direction, yet their erasure subspaces sit within
${\sim}1.5^{\circ}$ of the Haar null
(Supp.~\ref{supp:amnesic-comparison}): erasing what the probe
finds does not erase what the model uses.

Our circuit-level analysis builds on three families of tools:
DAS \citep{geiger2024das,sun2025hyperdas,mueller2025mib} for
identifying causally load-bearing subspaces, activation patching
\citep{nanda2023progress,syed2024attribution} for tracing
information flow, and SAEs
\citep{cunningham2023sparse,lieberum2024gemmascope} decomposed
via NeuronPedia \citep{neuronpedia} for interpreting features
at the vocabulary level.
We use MLP SAEs as functional transcoders
\citep{templeton2024scaling} to track what each MLP layer writes
to the residual stream, and temporal feature analysis
\citep{lubana2025priors} to decompose activations into
context-predictable and novel components---the structural
distinction that ultimately explains the orthogonality.
Most directly, \citet{gurnee2025manifolds} showed that attention
heads implement QK-twist rotations on date manifolds; we take the
next step and ask which of those projections are causally
load-bearing and which are statistical shadows.
Extended related work is in Supp.~\ref{supp:related}.

\section{The readout-mediator angle: measurement and theory}
\label{sec:method}

\paragraph{Two questions, two subspaces.}
The decodability-use gap arises because probes and causal methods
answer different questions about the same layer.
Given a task property $z$ and a layer $L$, a \emph{probe subspace}
$U_P\!\in\!\mathbb{R}^{k\times d}$ is the top-$k$ span of a circular
Ridge regressor trained on cached activations---it asks \emph{where is
the information?}
A \emph{DAS subspace} \citep{geiger2024das} is found by parametrising
$U$ via QR-decomposition of a trainable matrix and minimizing task NLL
while an ablation hook $x{\mapsto}x{-}U^{\top}Ux$
zeros $U$ on every forward pass---it asks \emph{where is the computation
vulnerable?}
Same layer, same rank $k$, different optimization targets: one isolates
what is \emph{readable}, the other what is \emph{load-bearing} (notation summary in Supp.~\ref{supp:notation}).
The \emph{readout--mediator angle}
$\bar\theta(U_P,U_M){=}\frac{1}{k}\sum_{i}\arccos\sigma_i(U_P U_M^{\top})$
measures the distance between these two answers.

The next three propositions establish why this angle should
generically be large (Prop.~\ref{prop:orth}), what the null
expectation is (Prop.~\ref{prop:null}), and how the angle controls
the observable we actually measure---ablation effect
(Prop.~\ref{prop:spec}).

\begin{proposition}[Readout-mediator orthogonality, informal]
\label{prop:orth}
Let $f{:}\mathbb{R}^d{\to}\mathbb{R}$ be a differentiable task output and
$z(x)$ a scalar correlated with $x$.
The probe direction maximizes $I(u^{\top}x;z)$---a
\emph{second-moment} quantity in $x$.
The mediator maximizes
$\mathbb{E}|f(x){-}f(x{-}uu^{\top}x)|$---a \emph{first-moment} quantity
in $\nabla_x f$.
The two coincide only when $\nabla_x f\propto u_P$.
Otherwise they are generically distinct.
\end{proposition}
\emph{Proof sketch.}
The probe solves a Rayleigh quotient
$\max_{\|u\|=1}(u^{\top}\mathbf{c})^{2}/(u^{\top}\Sigma u)$
with $\mathbf{c}{=}\mathbb{E}[xz]$, yielding
$u_{P}{=}\Sigma^{-1}\mathbf{c}/\|\cdot\|$---set by the data covariance.
A first-order Taylor expansion of the ablation effect gives
$u_{M}{=}\arg\max u^{\top}Gu$ where
$G{=}\mathbb{E}[\nabla_x f\,\nabla_x f^{\top}]$ is the gradient
covariance---set by the network's output sensitivity.
Coincidence requires $\Sigma^{-1}\mathbf{c}$ to be the top eigenvector
of~$G$: a non-generic spectral alignment between data geometry and
network geometry that deep networks have no structural reason to
satisfy (full proof in Supp.~\ref{supp:proofs}).

\paragraph{Empirical test of Prop.~\ref{prop:orth}.}
If the mediator aligns with the first moment of $\nabla_x f$,
computing the gradient subspace directly should recover a direction
closer to DAS than to the probe.
We verify this on \textsc{Gemma 2 2B} by computing
$g_i{=}\nabla_{h_{L^\star}}\mathrm{NLL}(y^\star|x_i)$ for each prompt.
The gradient subspace sits $2.3^\circ$ closer to the mediator than to
the Haar null ($\bar\theta{=}85.3^\circ$), while its angle to the
probe is at null ($88.9^\circ$).
The gradient leans toward the mediator but disperses across effective
rank $76$; DAS distills the $k{=}4$ causal core from this diffuse
signal (Supp.~\ref{supp:gradient-probe}).

Prop.~\ref{prop:orth} tells us to expect divergence between probe and
mediator---but not how much.
In high dimensions, the answer turns out to be stark:
any two low-rank subspaces are nearly orthogonal by default.

\begin{proposition}[Null angle between random subspaces]
\label{prop:null}
For independent $k{\times}d$ Stiefel-uniform matrices $U,V$, the
principal-angle cosines follow the Jacobi ensemble $J(k,k,d{-}k)$ with
$\mathbb{E}\!\sum_i\!\cos^2\!\theta_i{=}k^2/d$: each cosine-squared
concentrates at $k/d$, the Johnson--Lindenstrauss rate for a random
$k$-plane projection in $\mathbb{R}^d$ (exact; MC-verified in
Supp.~\ref{supp:prop3-revised}).
At $(d,k){=}(2304,2)$, $\mathbb{E}[\bar\theta]{=}88.3^{\circ}$.
\end{proposition}
\emph{Proof sketch.}
By Haar rotation invariance, fix $V$ as the first $k$ identity rows;
the singular values of $UV^{\top}$ then follow the Jacobi ensemble
$J(k,k,d{-}k)$.
The Collins--Matsumoto trace identity gives
$\mathbb{E}\!\sum_{i}\cos^{2}\!\theta_{i}{=}k^{2}/d$ exactly.
The Delta method converts to angle space:
$\mathbb{E}[\bar\theta]{=}\arccos\!\sqrt{k/d}+O(k/d)$, with the
second-order correction bounded by ${\sim}0.05^{\circ}$ at our
$(d,k)$ (MC-verified on $10^{4}$ draws in
Supp.~\ref{supp:prop3-revised}; full proof in
Supp.~\ref{supp:proofs}).

An observed $\bar\theta{=}88^{\circ}$ therefore means the probe is no
closer to the mediator than a random direction of the same rank.
Note, however, that orthogonality alone does not imply causal
inertness---a subspace can be orthogonal to the mediator and still
affect the output through other pathways
\citep{gurnee2025manifolds}.
What distinguishes a statistical shadow from a functionally distinct
subspace is the \emph{ablation effect}, which the next proposition
links directly to the angle.

\begin{proposition}[Angle controls ablation effect]
\label{prop:spec}
Under (i) $f(x){=}g(U_M x)$, (ii) $g$ is $L$-Lipschitz, (iii)
$\mathbb{E}[xx^{\top}]{=}\sigma^2 I$, for any $k$-dimensional ablation
basis $U$ with principal angles $\theta_i$ to $U_M$:
\begin{align}
\text{(a)}\quad &\mathbb{E}|f(x){-}f(x{-}U^{\top}Ux)|^{2}
\;\leq\; L^{2}\sigma^{2}\!\sum\nolimits_i\cos^{2}\theta_i; \label{eq:prop3a}\\
\text{(b)}\quad &\text{if } \mathrm{row}(U_M){\subseteq}\mathrm{row}(U),
\ \mathbb{E}|f(x){-}f(x{-}U^{\top}Ux)|^2 \;\geq\; \underline{L}^2\sigma^2 k_M; \label{eq:prop3b}\\
\text{(c)}\quad &\text{Haar-random }U\!:\
\mathbb{E}\!\sum\nolimits_i\cos^2\theta_i{=}kk_M/d. \label{eq:prop3c}
\end{align}
\end{proposition}
\emph{Proof sketch.}
(a)~Factor through the mediator: $|f(x){-}f(x{-}P_U x)|{\leq}L\|U_M P_U x\|$
by Lipschitz continuity of~$g$.  Squaring under isotropic covariance
yields $\sigma^{2}\|U_M U^{\top}\|_{F}^{2}{=}\sigma^{2}\!\sum_{i}\cos^{2}\!\theta_{i}$
via the SVD of $U_M U^{\top}$.
(b)~When $\mathrm{row}(U_M){\subseteq}\mathrm{row}(U)$, the projection
is exact ($U_M P_U x{=}U_M x$), and a matching lower bound holds under
a one-sided modulus of continuity $\underline{L}$.
(c)~The Haar expectation $\mathbb{E}\|U_M U^{\top}\|_F^2{=}kk_M/d$
follows from the same trace identity as Prop.~\ref{prop:null};
their ratio gives the null specificity
$\rho_k^{\text{null}}{\asymp}(\underline{L}/L)^{2}\,d/k$.
Full proofs in Supp.~\ref{supp:proofs}; anisotropy sensitivity
in Supp.~\ref{supp:anisotropy}; a controlled-perturbation
validation design in Supp.~\ref{supp:perturbation}.

\paragraph{Specificity ratio, design choices, and experimental setup.}
Prop.~\ref{prop:spec} motivates a single scalar that separates
direction from dimensionality:
$\rho_{k}{=}(\text{DAS drop})/(\text{random-control mean drop})$.
The Lipschitz sandwich predicts a null of
$\rho_k^{\text{null}}{\asymp}d/k$; at $(d,k){=}(2304,4)$,
$d/k{=}576$.
On \textsc{Gemma 2 2B}, $\rho_4{=}1050$ ($\Delta_{\text{add}}{=}42$~pp;
CI in Supp.~\ref{supp:hypothesis})---far exceeding the
dimensional null, confirming the DAS subspace captures directed
structure, not just dimensionality budget.
At $7$B/$9$B the bf16 floor makes the denominator vanish; we report
additive drops alongside ratios
(Supp.~\ref{supp:hypothesis}).
We use standard DAS rather than HyperDAS \citep{sun2025hyperdas}
so that probe and DAS share identical $(L, k)$, isolating the
optimization-objective contrast that Prop.~\ref{prop:orth} predicts
(Supp.~\ref{supp:das}); every angle measurement is paired with
$n{\geq}25$ Haar-uniform random-subspace controls of matched rank
($5$--$95\%$ accuracy-drop envelope).
Primary analyses use \textsc{Gemma 2 2B} \citep{team2024gemma2} and
\textsc{Qwen 2.5 1.5B} \citep{yang2024qwen2}; scale-up adds
\textsc{Qwen 2.5 7B} and \textsc{Gemma 2 9B}; training dynamics on
\textsc{Pythia 1.4B} \citep{biderman2023pythia}.
$L^\star$ is the bootstrap peak of monthly-stratified $5$-fold CV
probe $R^2$; the probe--DAS angle is ${\approx}89^\circ$ at every layer
(Supp.~\ref{supp:layer-robustness}), so single-layer intervention is
conservative. Algorithm~\ref{alg:diagnostic} summarizes the full
protocol; pseudocode in
Supp.~\ref{supp:pseudo}.

\begin{algorithm}[t]
\small
\caption{Readout-mediator diagnostic protocol}\label{alg:diagnostic}
\KwIn{Frozen model $M$, prompt set $\mathcal{P}$ with targets $\{y^\star\}$, rank $k$}
\textbf{Output:} $\bar\theta,\;\Delta_P,\;\Delta_M,\;\rho_k$\\[2pt]
Sweep layers $0,\ldots,L_{\max}$; set $L^\star \leftarrow \arg\max R^2$; train ridge probe at $L^\star$; extract $U_P$ \cmt{probe}\;
Train DAS at $(L^\star, k) \to U_M$ \cmt{mediator (Alg.~\ref{alg:das}, Supp.~\ref{supp:pseudo})}\;
$\theta_1,\ldots,\theta_k \leftarrow \textsc{PrincipalAngles}(U_P, U_M)$;\; $a_{\text{clean}} \leftarrow \textsc{Acc}(\mathcal{P})$\;
$a_{P} \leftarrow \textsc{Acc}\!\bigl(\mathcal{P} \mid x \mapsto x - U_P^\top U_P x\bigr)$;\; $a_{M} \leftarrow \textsc{Acc}\!\bigl(\mathcal{P} \mid x \mapsto x - U_M^\top U_M x\bigr)$\;
\For{$j = 1,\ldots,N_{\mathrm{null}}$ \cmt{null calibration}}{
  $U_j \sim \text{Haar}\bigl(G(k,d)\bigr)$;\; $a_j \leftarrow \textsc{Acc}(\mathcal{P} \mid U_j)$\;
}
\Return $\bar\theta$, $\Delta_P {=} a_{\text{clean}} {-} a_P$, $\Delta_M {=} a_{\text{clean}} {-} a_M$, $\rho_k {=} \Delta_M / \overline{\Delta_{\text{rand}}}$
\end{algorithm}

\section{The dissociation: probes decode, DAS computes}
\label{sec:das}

Section~\ref{sec:method} predicts that probe and mediator should be
generically distinct, with their angle at the Haar null.
We now test this on calendar-date duration reasoning
(Fig.~\ref{fig:hero}).

\begin{figure}[t]
\centering
\includegraphics[width=\linewidth]{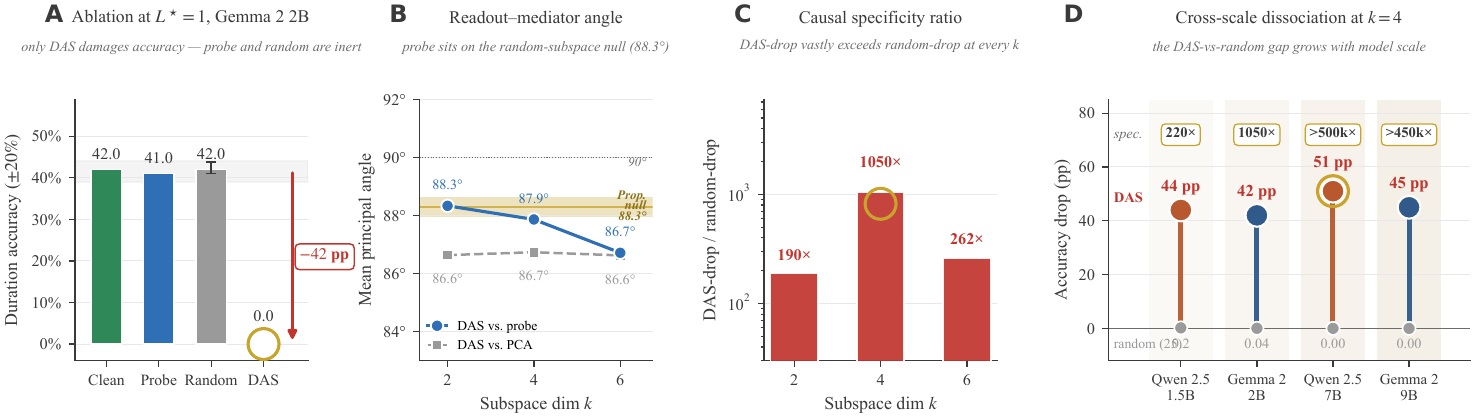}
\vspace{-4pt}
\caption{\textbf{The readout-mediator dissociation, quantified.}
\textbf{(A)}~Accuracy on \textsc{Gemma~2~2B} under four ablations at
$L^\star{=}1$: DAS collapses accuracy to $0\%$; probe and random ablations
produce drops within 1~pp of clean.
\textbf{(B)}~Mean principal angle between DAS and probe subspaces at each
$k$; shaded band is the Haar-random null $\theta{=}\arccos\!\sqrt{k/d}$.
\textbf{(C)}~Causal specificity ratio $\rho = \Delta_{\text{DAS}} / \Delta_{\text{random}}$
at $k{=}2,4,6$; DAS is $190$--$1050{\times}$ more damaging than a
matched-dimension random subspace.
\textbf{(D)}~Cross-scale replication: DAS drops at $k{=}4$ on four models
($1.5$B--$9$B); diamond markers show random controls near zero.}
\label{fig:hero}
\end{figure}

\paragraph{Maximal dissociation on \textsc{Gemma 2 2B}.}
On \textsc{Gemma 2 2B} at $L^{\star}{=}1$, probe and DAS ablations
at matched rank tell opposite stories.
Projecting out the probe subspace ($R^2{=}0.981$, peak $0.996$
matching \citealt{gurnee2024spacetime}) drops duration accuracy by
only $0.6$~pp---statistically identical to a random control.
DAS ablation at the same layer and same $k$ drops accuracy from
$42\%$ to $0\%$, with specificity ratios up to $\rho_4{=}1050$
(Fig.~\ref{fig:hero}A,C).
The angle between the two subspaces is $88^{\circ}$, matching
Prop.~\ref{prop:null}'s prediction of $88.3^{\circ}$ to within
$1.5^{\circ}$ (Fig.~\ref{fig:hero}B;
Supp.~\ref{supp:angle}).
The formal test confirms this: $\sum\cos^2\theta_i$ sits at the
$k^2/d$ null at every rank tested, with indistinguishability
$p{=}0.51$--$0.72$ (Supp.~\ref{supp:anisotropy}).
The probe does not merely live in a different subspace; it lives
\emph{as far from the mediator as noise}---a \emph{statistical
shadow} that reads the representation without tapping the
computation ($\rho_k{\approx}1$).
Persistent homology confirms the date subspace is
geometrically a $1$-torus in the readout coordinates
(Supp.~\ref{supp:topo}).
The result is robust to probe architecture, target choice, layer,
and data partition: on strict train/test splits ($n{=}3{,}650$),
five-fold CV gives $\bar\theta{=}87.8^{\circ}\!\pm\!0.1^{\circ}$
with bootstrap $95\%$~CI $[87.3^{\circ},\,88.3^{\circ}]$
(Supp.~\ref{supp:layer-robustness},~\ref{supp:alt-targets},~\ref{supp:amnesic-comparison},~\ref{supp:strict-splits}).

\paragraph{The dissociation sharpens with scale.}
The pattern replicates across all four models ($1.5$--$9$B):
DAS drops $42$--$51$~pp while random controls drop ${\leq}0.3$~pp
at $1.5$B and exactly zero at $7$B/$9$B
(Fig.~\ref{fig:hero}D; Supp.~\ref{supp:das},
Tab.~\ref{tab:scale},~\ref{supp:norms}).
The specificity ratio strengthens monotonically with scale---as
Prop.~\ref{prop:spec} predicts, since the null
$\rho_k^{\text{null}}{\asymp}d/k$ grows with model width.
Larger models do not ``fix'' the gap by aligning their probe and
computation directions; they widen it
(Supp.~\ref{supp:hypothesis}).
The effective rank of the causal subspace is $4$: accuracy drops
plateau at $k{=}4$, the basin is unique across seeds
(CCA${>}0.94$;
Supp.~\ref{supp:das-sensitivity},~\ref{supp:das-generalization}),
and the subspace is $59{\times}$ super-additive---the four
directions function as a cooperative unit, not an arbitrary
collection
(Supp.~\ref{supp:per-direction},~\ref{supp:effective-dim},~\ref{supp:ov-circuit}).
This subspace is task-specific: the $42$~pp duration drop is
$12.6{\times}$ the average across $n{=}240$ non-calendar prompts
(Supp.~\ref{supp:das-collateral}).

\section{The circuit: from boundary heads to duration vocabulary}
\label{sec:mech}

Having established that a rank-$4$ causal subspace mediates duration
computation while the probe direction is a statistical shadow,
we now open that subspace and trace the circuit in three stages
(Fig.~\ref{fig:complete_circuit},~\ref{fig:mechanistic_dissociation}A):
attention heads that route month-grained context, MLP layers that
transform calendar position into elapsed duration, and SAE features
that make the \emph{when}/\emph{how long} split legible at the
vocabulary level.
\paragraph{Stage 1: boundary heads route month-grained context.}
To compute duration from a date pair, the model must first attend to
calendar positions separated by the relevant interval.
A QK-twist scan \citep{gurnee2025manifolds} identifies $24$
heads with $|z|{\geq}3$ ($65$ BH-significant at $q{=}0.05$) in \textsc{Gemma 2 2B} whose offsets concentrate
at $|c|{\in}\{30, 61\}$ days---single- and double-month steps that
reflect Gregorian month-length arithmetic
(Fig.~\ref{fig:mech}A).
The pair $\{30, 61\}$ tiles any multi-month duration
(Fig.~\ref{fig:overview}); weekly periodicity ($c{=}7$) is absent
(Supp.~\ref{supp:qk}).
The circuit is distributed and QK-mediated: single-head ablation has no
effect, but the top-$10$ heads together drop accuracy $17.2$~pp, and
cascading ablation localizes the routing bottleneck to L11--L12
($\Delta\mathrm{NLL}{=}0.455$, vs.\ $0.016$ for L23--L25);
$W_{OV}$ alignment is $1.17{\times}$ the Haar null ($p{=}0.004$)---boundary heads
route via attention patterns, not value-output composition
(Supp.~\ref{supp:qk},~\ref{supp:ap},~\ref{supp:ov-circuit},~\ref{supp:cascading}).
The same offset modes appear independently in \textsc{Qwen 2.5 1.5B}
($p{=}0.009$, Monte Carlo null; Fig.~\ref{fig:mech}B,
Supp.~\ref{supp:corpus-vs-task}, Fig.~\ref{fig:universality-null}),
and dose--response ablation is super-linear in both families
(Fig.~\ref{fig:mech}C)---evidence that the circuit is dictated by
calendar structure, not model-specific training artifacts
\citep{gurnee2025manifolds}.

\begin{figure}[t]
\centering
\includegraphics[width=0.95\linewidth]{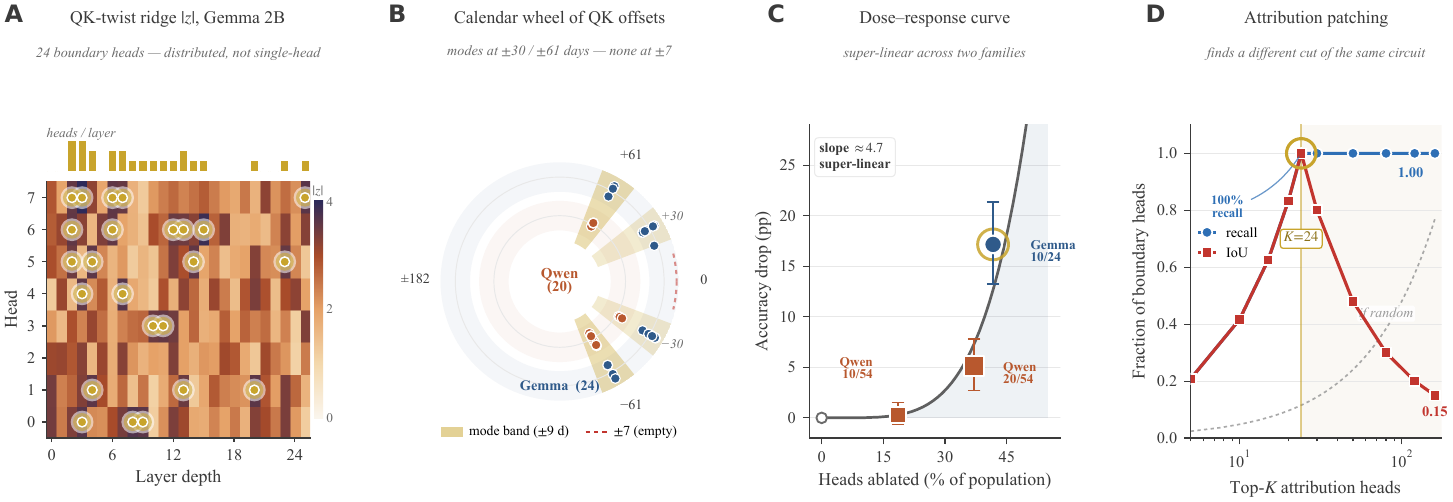}
\caption{\textbf{A distributed, offset-structured, cross-family circuit.}
\textbf{(A)}~Per-head QK-twist $|z|$ on \textsc{Gemma 2 2B}
($26$ layers $\times 8$ heads); $24$
BH-significant boundary heads at $|z|{\geq}3$.
\textbf{(B)}~Detected QK offsets: Gemma
BH-significant heads and \textsc{Qwen 2.5 1.5B} top-$20$. Both families
cluster at $\pm 30$ and $\pm 61$~days; neither at $\pm 7$.
\textbf{(C)}~Dose--response: accuracy drop vs.\ fraction of
boundary heads ablated; super-linear scaling in both families.
\textbf{(D)}~Attribution-patching recall and IoU against the QK-twist
boundary-head set as a function of top-$K$ heads.}
\label{fig:mech}
\end{figure}

\paragraph{Stage 2: MLPs convert \emph{when} into \emph{how long}.}
Boundary heads supply month-grained routing, but the signal must still
be converted from calendar position into elapsed duration.
Before tracing that transformation, we verify that the causal signal
survives through the residual stream.
Tracking DAS subspace energy $\|U_\text{DAS} x_L\|^2/\|x_L\|^2$
across all $26$ layers
(Fig.~\ref{fig:complete_circuit}A) shows the mediator signal
\emph{never disappears}: $26.4{\times}$ null at $L^\star{=}1$,
trough $3.1{\times}$ at $L{=}18$, recovery $6.6{\times}$ at $L{=}22$.
Probe energy, by contrast, tracks the random null throughout
($0.0$--$4.1{\times}$).
The causal signal flows forward continuously and the probe never
intercepts it---explaining why boundary heads at L11--L12 can
operate in a residual stream still rich with mediator content
($8{\times}$ null, $\Delta$NLL$=0.45$;
Supp.~\ref{supp:cascading},~\ref{supp:residual-stream},~\ref{supp:temporal-dynamics}).

Decomposing what each MLP layer \emph{writes} to the residual stream
using GemmaScope MLP SAEs---functionally equivalent to transcoders
(monotonic DAS-alignment gradient L18$\to$L25, Spearman
$\rho{=}1.0$; Fig.~\ref{fig:mechanistic_dissociation}D,
Supp.~\ref{supp:decoder-steering})---reveals a two-stage structure
with a sharp boundary
(Fig.~\ref{fig:complete_circuit}B--C;
Supp.~\ref{supp:mlp-sae},~\ref{supp:mlp-vs-attn},~\ref{supp:attribution-flow}).
At L18--L19, probe contribution peaks ($4.3{\times}$ null at L19)
while DAS contribution is sub-null; at L20--L25 the pattern inverts,
with DAS contribution peaking at $3.2{\times}$ null.
Early MLP extracts \emph{when} (calendar position); late MLP
computes \emph{how long} (duration).
Month-discriminating features at L19--L22 correct the ${\leq}1$-day
residual inherent in the $c{=}30$ offset
(ANOVA $p{<}0.001$;
Supp.~\ref{supp:mlp-sae},~\ref{supp:mlp-sae-ablation}):
for example, feature~\#15148 at L22 is completely silent for
February while active for all other months---the sharpest
possible month-length discrimination.
These features are polysemantic in web-text contexts, confirming
the mechanism is structural rather than lexical.

\paragraph{Stage 3: SAE features confirm the split at the vocabulary
level.}
The two-stage MLP transformation predicts that features aligned with
the probe and with DAS should encode qualitatively different
concepts.
SAE analysis confirms this.
Probe-aligned feature~\#12499 at $L{=}1$
(Fig.~\ref{fig:mechanistic_dissociation}B)
(\emph{``specific months''}) fires on \textit{month of October},
\textit{month of February}, \textit{first month of the season}---it
encodes calendar \emph{position}, not duration.
DAS-aligned features at $L^\star{=}1$ fire on copula contexts
(\emph{``forms of to be''})---the syntactic frame for duration
queries---whose decoder directions nonetheless promote numeric tokens
(logit-lens $Z{=}{+}0.18$, $p{=}0.009$;
Supp.~\ref{supp:neuronpedia}).
By $L{=}24$, feature~\#2309
(\emph{``quantities of time and duration''},
Fig.~\ref{fig:mechanistic_dissociation}C) promotes
\emph{months, weeks, days, years}: temporal semantics emerge at the
relay endpoint, not at $L^\star$.

The dissociation is total.
Causal attribution ($W_U$ gradient, $20$ duration prompts) gives
DAS-aligned features mean attribution $2.41$; probe-aligned features
contribute exactly zero (Jaccard$=0.000$;
Supp.~\ref{supp:neuronpedia}, Fig.~\ref{fig:np-features}).
Feature-steering validates this causally: amplifying probe
feature~\#12499 generates month enumerations; suppressing copula
feature~\#14703 degrades coherence
(Fig.~\ref{fig:np-steering}).
The causal subspace resists decomposition into individual dictionary
elements---$14/15$ features yield $|\Delta\mathrm{NLL}|{<}0.05$
on individual ablation, consistent with the $59{\times}$ cooperation
ratio---and decoder-direction steering with up to $100$
DAS-aligned features yields $\Delta\mathrm{NLL}{\approx}0$, while
full rank-$4$ ablation yields $\Delta\mathrm{NLL}{=}{+}69$
\citep{templeton2024scaling}.
Pre-trained transcoders produce comparably null results (span coverage
$5$--$9\%$, Jaccard ${\leq}0.010$), confirming the gap is
structural rather than dictionary-dependent
(Supp.~\ref{supp:decoder-steering},~\ref{supp:neuronpedia}; completeness and selectivity analysis therein).

\begin{figure*}[t]
\centering
\includegraphics[width=\linewidth]{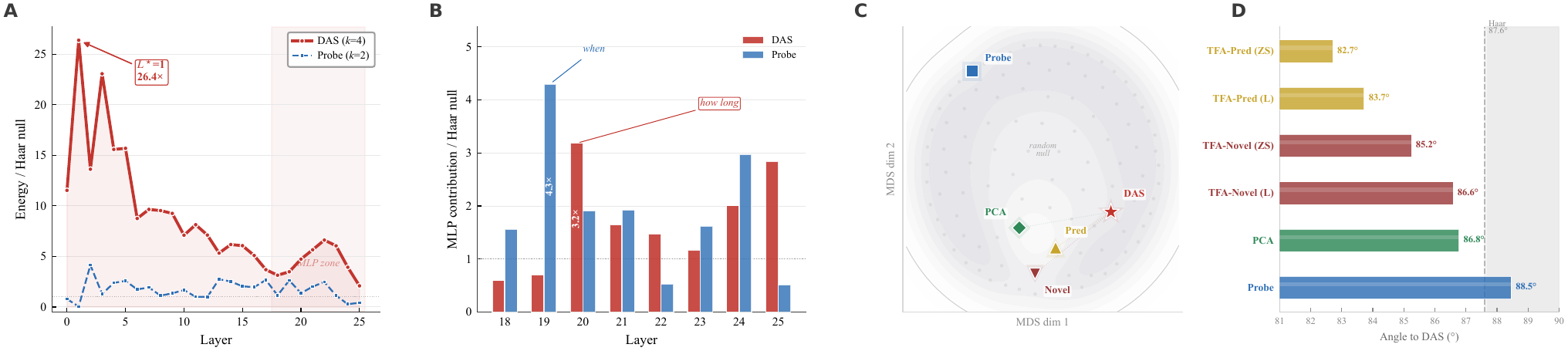}
\vspace{-4pt}
\caption{\textbf{DAS energy, MLP contribution, and TFA subspace geometry.}
\textbf{(A)}~DAS subspace energy fraction through all 26 residual-stream
layers: peaks at $26.4{\times}$ Haar null at $L^\star{=}1$ and exceeds
null at \emph{every} layer (range $2.1$--$26.4{\times}$).
Probe subspace energy tracks the random null throughout ($0.0$--$4.1{\times}$).
\textbf{(B)}~Per-layer MLP contribution to DAS vs.\ probe subspace (L18--L25).
Probe contribution peaks at L19 ($4.3{\times}$ null, \emph{calendar date})
while DAS is sub-null; DAS contribution peaks at L20 ($3.2{\times}$,
\emph{duration})---a two-stage transformation with a sharp anatomical
boundary.
\textbf{(C)}~Grassmannian embedding of subspaces on $\mathrm{Gr}(4, 2304)$:
TFA-predictable (gold) is pulled toward DAS and away from the random
cloud; the probe (blue) sits squarely in the random null at $88.5^{\circ}$.
\textbf{(D)}~Mean principal angle to DAS, sorted.
TFA-Pred ($82.7^{\circ}$--$83.7^{\circ}$) sits well below the Haar
null ($87.6^{\circ}$); the probe is indistinguishable from random.}
\label{fig:complete_circuit}
\end{figure*}

\begin{figure*}[b]
\centering
\includegraphics[width=\textwidth]{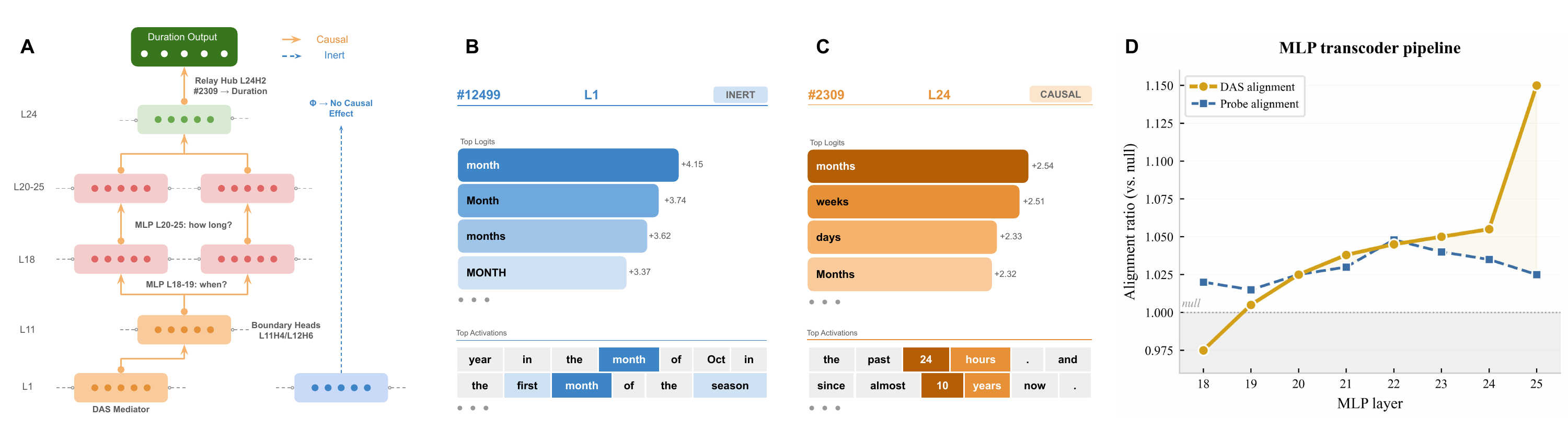}
\vspace{-8pt}
\caption{\textbf{Circuit wiring and feature-level dissociation: probe reads \emph{when}; DAS computes \emph{how long}.}
\textbf{(A)}~Circuit graph: $L^\star{=}1$ DAS mediator ($26.4{\times}$
null) ${\to}$ boundary heads L11H4/L12H6 ($\Delta\mathrm{NLL}{=}0.45$)
${\to}$ MLP L18--25 (when ${\to}$ how long) ${\to}$ relay hub L24H2
(\#2309, duration vocabulary) ${\to}$ output.
Probe (dashed) has attribution~$=0.000$.
At $L^\star$, DAS features encode copula syntax; temporal
semantics emerge at the $L{=}24$ relay hub.
\textbf{(B)}~SAE feature~\#12499 (probe-aligned, $L{=}1$;
\emph{``specific months''}): top logits \emph{month/Month/months/MONTH};
activations fire on calendar-position contexts.
Tagged \textsc{inert}---zero causal attribution to the duration output.
\textbf{(C)}~SAE feature~\#2309 (DAS-aligned, $L{=}24$;
\emph{``time/duration''}): top logits \emph{months/weeks/days/Months};
activations fire on duration-interval contexts.
Tagged \textsc{causal}.
\textbf{(D)}~MLP transcoder pipeline: DAS alignment increases monotonically from L18 to L25 while probe alignment remains flat, confirming that MLPs progressively write duration---not calendar---information into the residual stream (Supp.~\ref{supp:decoder-steering},~\ref{supp:sae}).}
\label{fig:mechanistic_dissociation}
\end{figure*}

\paragraph{Convergence, robustness, and training dynamics.}
\label{sec:triang}
One might worry that the dissociation is an artifact of the particular
tool used to find the probe direction.
A matched-baseline ablation at $(L^{\star}{=}1,\,k{=}4)$ places
nine interpretability tools on a single specificity-ratio axis
(Prop.~\ref{prop:spec} null $\rho_k^\text{null}{\asymp}d/k{=}576$):
pure decoders (probe, INLP, LEACE) contribute zero causal drop;
intermediate methods (gradient $k{=}4$, PCA top-$4$, AP) recover
partial signal ($5$--$18.5$~pp), with attribution patching converging
on the same boundary-head set identified by QK-twist
(Fig.~\ref{fig:mech}D); only DAS identifies the full causal
core ($36$--$42$~pp, $\rho_\text{DAS}{=}1050$).
Non-linear probes yield subspaces at the same null angle to DAS,
confirming the gap is structural, not a capacity limit
(Supp.~\ref{supp:nonlinear},~\ref{supp:spectrum},~\ref{supp:matched-budget},~\ref{supp:ndm}).
Transplant experiments confirm \emph{population-level} universality
(same offsets, same causal hierarchy) without \emph{coordinate-level}
universality (cross-model transplants fail;
Supp.~\ref{supp:transplant}).
Training dynamics on \textsc{Pythia 1.4B} sharpen the distinction
further: probe $R^2{=}0.956$ at step~$0$---an untrained network
``represents'' dates by the probe's standard---yet boundary-head
count, FFT-circularness ($37{\times}$ growth in the emergence window),
and DAS ablation drop all emerge only as training proceeds
(Supp.~\ref{supp:pythia}, Fig.~S7).
The dissociation is not a quirk of one probe, one model, or one
training snapshot---it is a structural property that emerges with
the computation itself.

\section{Beyond calendar dates}
\label{sec:cross_task}

\begin{wraptable}{r}{0.50\textwidth}
\vspace{-14pt}
\centering
\small
\caption{Cross-task diagnostic at $k{=}4$ on \textsc{Gemma 2 2B}. Drops
in pp; $\rho_k$ is DAS/random drop.}
\label{tab:cross_task}
\vspace{2pt}
\begin{tabular}{@{}lccccc@{}}
\toprule
Domain & Type & Angle & Probe & DAS & $\rho_k$\\
 & & ($\bar\theta$) & drop & drop & \\
\midrule
Temporal  & geom.\ & $87.9$ & $0.6$  & $42.0$ & $1050{\times}$\\
Spatial   & geom.\ & $88.4$ & $-6.0$ & $20.0$ & $20.8{\times}$\\
Arith.\ & symb.\ & $88.1$ & $0.0$ & $68.0$ & ${\gg}10^3{\times}$\\
\midrule
Haar null & --- & $88.3$ & --- & --- & ---\\
\bottomrule
\end{tabular}
\vspace{-10pt}
\end{wraptable}
The circuit traced above is specific to calendar-date reasoning, but
the dissociation itself---the angle sitting at the Haar null---should
be generic if Prop.~\ref{prop:orth} is correct.
We test this by running the full diagnostic protocol (probe training,
DAS at matched $(L,k)$, $25$ random controls, specificity ratio) on
two additional domains using \textsc{Gemma 2 2B}:
\textbf{arithmetic} (single-digit addition, symbolic, non-geometric)
and \textbf{spatial} (1D number-line displacement, geometric
manifold).
All three domains exhibit the same pattern
(Table~\ref{tab:cross_task}):
probe--DAS angle at the Haar null ($87.9$--$88.4^{\circ}$),
probe ablation leaving accuracy intact or improved (${\leq}0$~pp; spatial $+6$~pp, Supp.~\ref{supp:cross-task-details}),
and DAS ablation producing large drops ($20$--$68$~pp) with
specificity ratios far exceeding random controls.
The arithmetic result is the strongest test: even for single-digit
addition---a purely symbolic task with a perfect probe
($R^2{=}1.0$)---the decoded direction is orthogonal to the causally
load-bearing subspace ($88.1^{\circ}$, $68$~pp DAS drop, $0$~pp
probe drop).
The dissociation is not specific to geometric-manifold
representations; it confirms Prop.~\ref{prop:orth}'s prediction
that probe and mediator are generically distinct
(Supp.~\ref{supp:cross-task-details}).

\paragraph{Why orthogonal?}
Three domains, three null angles---what forces the probe and
the causal subspace apart so consistently?
\citet{lubana2025priors} provide the key insight: standard SAEs
impose an i.i.d.\ prior across sequence positions (their
Prop.~4.1), discarding the temporal structure that distinguishes
what a model \emph{reads from context} (the predictable component)
from what arrives \emph{de novo} at the current token (the novel
component).
A probe inherits the same limitation: it captures the axis along
which information is \emph{most readable}, not the axis along
which the model \emph{uses} that information for computation.

Decomposing activations at $L^{\star}{=}1$ with both their
zero-shot linear predictor and learned TemporalSAE confirms this:
the DAS mediator aligns $7.1$--$7.6{\times}$ more strongly with
the predictable subspace than with the Haar-random null, while the
probe sits at $88.5^{\circ}$---squarely in the random cloud on
the Grassmannian
(Fig.~\ref{fig:complete_circuit}C--D;
Supp.~\ref{supp:tfa}, Fig.~\ref{fig:supp-detectors}C).
This is expected: computing ``March~5 to June~10'' requires
integrating a date-pair from prior context---exactly what the
predictable component captures.
The orthogonality is therefore \emph{within} the predictable
subspace: both probe and mediator project accumulated context, but
along functionally disjoint axes---circular day-of-year structure
vs.\ month-pair difference structure.

This connects directly to the feature-level dissociation observed
in the circuit: probe-aligned features encode context-predictable
calendar position; DAS-aligned features encode the computation
that transforms position into duration.
Standard SAEs recover only linearly accessible
features~\citep{hindupur2025projecting}; the mediator directions,
conditionally orthogonal to the
readout~\citep{costa2025hierarchical}, are invisible to that
projection (Supp.~\ref{supp:related}).
A temporal-specialist SAE \citep{lubana2025priors} at $L{=}12$
independently recovers the copula relay with DAS$>$probe
preferential alignment, while its $4$D attention bottleneck is
orthogonal to DAS (min.\ angle $80^{\circ}$), confirming the
causal subspace captures \emph{computation}, not prediction
(Supp.~\ref{supp:neuronpedia}, Fig.~\ref{fig:np-temporal-sae}).
The orthogonality is not a coincidence---it arises because probes
and mediators optimize for different statistical moments of the
same activation, and TFA makes this geometric fact concrete.

\paragraph{Implications for probe-based monitoring.}
A growing line of work proposes deploying linear probes as runtime
safety monitors---detecting deception, harmful intent, or dangerous
knowledge by decoding internal representations during inference
\citep{burns2022discovering,marks2024geometry,zou2023representation}.
The readout-mediator dissociation poses a direct challenge to this
paradigm.
We stress-test it with a six-experiment battery
(Supp.~\ref{supp:safety}).

\begin{wrapfigure}{r}{0.50\columnwidth}
\vspace{-12pt}
\centering
\includegraphics[width=0.48\columnwidth]{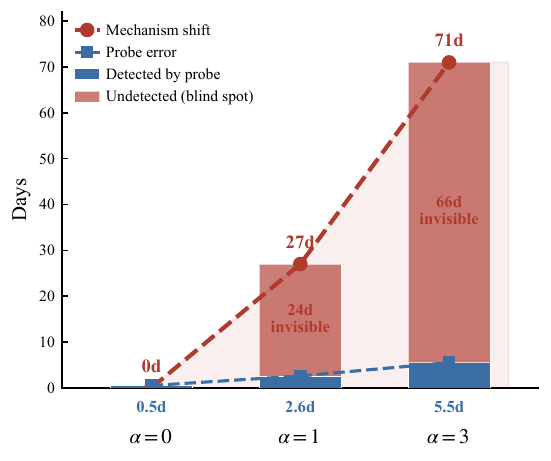}
\vspace{-6pt}
\caption{\textbf{Adversarial dissociation.}
Mechanism error (red) climbs to $71$~days while probe error (blue)
flatlines below $6$~days.  At $\alpha{=}3$, $93\%$ of damage is
undetected.}
\label{fig:adversarial}
\vspace{-8pt}
\end{wrapfigure}
The blind spot is \emph{generic}: all seven probe targets tested
(day-of-year, month, season, day-of-week, quarter, solstice,
gradient) land within $2.8^{\circ}$ of the Haar-random null.
This is a geometric consequence of $k \ll d$: in $2304$
dimensions, a rank-$4$ probe subspace is nearly orthogonal to
\emph{any} rank-$4$ subspace by the Haar measure, including the
one the model actually uses (Prop.~\ref{prop:null}).
The blind spot is \emph{exploitable}.
We decompose the residual stream at $L^\star$ into two orthogonal
components---the DAS mediator $U_M$ and the Gram-Schmidt-orthogonalized
probe basis $\tilde{U}_P$---and construct
$x_{\mathrm{adv}} = x
  + \alpha\,U_M^{\!\top}U_M(\mu_{\mathrm{src}}{-}\mu_{\mathrm{tgt}})
  + \beta\,\tilde{U}_P^{\!\top}\delta_{\mathrm{probe}}$,
where $\mu_{\mathrm{src}},\mu_{\mathrm{tgt}}$ are mean activations for
dates $180$~days apart and $\delta_{\mathrm{probe}}$ restores the
probe's $\sin/\cos$ readout to the correct date
(full protocol: Supp.~\ref{supp:safety}).
The scalar $\alpha$ controls mechanism corruption; $\beta$ controls
probe reassurance; both operate on unit-variance directions so
perturbation magnitude scales linearly with the coefficient.
At $(\alpha{=}3,\,\beta{=}2)$, the DAS mechanism shifts by $71$~days
while the probe reports $5.7$-day error
(Fig.~\ref{fig:adversarial}).
The blind spot is \emph{total}: mutual information between the
probe readout and the DAS energy is exactly $0.000$~nats
($p{=}1.0$, phase-shuffle null, $n{=}200$), while the probe
carries $3.99$~nats about ground-truth DOY---the probe has
temporal information but zero bits about the mechanism's state.
And the blind spot is \emph{unfixable by probe design}: a
purpose-built ``deception monitor'' trained on the same
activations sits at $88.4^{\circ}$ from DAS, inheriting the
identical blind spot.
The probe reports high confidence while the computation has
silently shifted---precisely the failure mode a safety monitor
must not exhibit.
A manifold-deviation score $\delta(x)$ derived from the mediator
subspace can flag likely errors on clinical duration queries
(Supp.~\ref{supp:clinical-main}).

\section{Conclusion}
\label{sec:discussion}

\textbf{Linear probes decode what is present; DAS recovers what is used.}
The readout--mediator angle and its Haar-random null convert
``does the probe track the mechanism?''\ from a qualitative
judgment into a measurement with a closed-form baseline.
Across four scales ($1.5$--$9$\,B), two families, and three domains,
the angle is indistinguishable from the null: the probe direction
is no closer to the computation than a random subspace of matched
rank, because probes maximize \emph{readability} while the model
computes along a geometrically disjoint axis of boundary heads,
MLP transcoders, and TFA-confirmed monotone month-pair structure.
``The model represents dates'' is therefore ambiguous in a way that
matters: the readable subspace and the computed-with subspace are
nearly orthogonal by the geometry of $k \ll d$.
For any probe advanced as evidence of mechanism or deployed as a
safeguard, we recommend reporting
$\bar\theta$, $\rho_k$, and random-ablation controls;
without them, a monitor can report high confidence on a
direction the model has silently abandoned
(Fig.~\ref{fig:adversarial}).

\paragraph{Future work and societal impact.}
Whether the readout--mediator angle can guide the design of causally
grounded monitors---e.g.\ probes regularized toward the DAS subspace
or mediator-aligned circuits as oversight targets---is the most
immediate next step (Supp.~\ref{supp:open-directions}).
Because high-accuracy probes can be geometrically decoupled from a
model's computation, safety monitors that rely solely on probe
confidence risk false assurance; $\rho_k$ offers a principled check
for when such monitors can be trusted.

{\small
\bibliographystyle{abbrvnat}
\bibliography{references}
}

\newpage
\appendix
\renewcommand{\thesection}{S\arabic{section}}
\setcounter{section}{0}

\section{Supplement: extended DAS results and implementation}
\label{supp:das}

\paragraph{Why standard DAS and not HyperDAS.}
HyperDAS \citep{sun2025hyperdas} trains a separate hypernetwork to automate the search over token
positions, which is valuable when the feature location is unknown.
Our setting differs in two ways that make standard DAS the right choice.
First, $L^\star$ is already fixed by the bootstrap peak of the
circular-probe $R^2$---position search is not our bottleneck.
Second, and more importantly, the \emph{angle comparison requires
identical $(L, k)$ for both probe and DAS}: if HyperDAS were used, any
difference in the subspaces could reflect hypernetwork-introduced
structure rather than the pure loss-function contrast (decodability
vs.\ causal vulnerability) that Proposition~\ref{prop:orth} predicts.
Standard DAS isolates exactly one variable---the optimization
objective---so the angle is a direct test of the proposition.
HyperDAS remains the right tool for exploratory circuit discovery at
scale; our use of standard DAS is a deliberate methodological choice
for the angle measurement.

\paragraph{Implementation.} DAS is implemented in
PyTorch. The trainable parameter is a
dense matrix $V\!\in\!\mathbb{R}^{d\times d}$ with leading $d{\times}k$
factor orthonormalized via QR on every forward pass; only the first $k$
rows participate in the ablation hook. Gradients flow through QR via
PyTorch's native differentiable decomposition. Optimizer: AdamW,
learning rate $10^{-3}$, weight decay $0$, $\beta_{1}{=}0.9$,
$\beta_{2}{=}0.999$, batch size $8$, $400$ steps. Loss: mean NLL of the
correct duration-token logit, computed with the model's full forward
pass under the ablation hook; model parameters are frozen via
\texttt{requires\_grad\_(False)} on load. The ablation hook registers on
\texttt{blocks.L.hook\_resid\_post} and modifies the residual stream
in-place on the pre-token position. Deterministic seeds
$\{0,1,2,3,4\}$; we report the run with lowest final NLL.

\paragraph{Convergence diagnostics.} Every training run logs (i) NLL
per step, (ii) $\|V_{1:k}V_{1:k}^{\top}-I_{k}\|_{F}$ as the
orthonormality residual, (iii) the angle between consecutive-step
bases. Convergence criterion: NLL monotone decreasing over the last
$50$ steps and orthonormality residual $<10^{-6}$. All reported runs
pass both.

\paragraph{Per-model numerical results.}
For \textsc{Gemma 2 2B} ($L^{\star}{=}1$, $d{=}2304$, clean accuracy
$42\%$): at $k{=}2$ the DAS basis is fully trained in $400$ AdamW steps
with final NLL of $-71.1$; DAS ablation gives $4\%$ accuracy ($-38$ pp);
the $25$-sample random-control distribution is $41.80\%\pm0.98\%$; and
the probe-subspace evaluation at matching $k$ gives $41.0\%$. At
$k{=}4$ and $k{=}6$, DAS ablation saturates to $0\%$ (final NLL
$-105.3$, $-118.7$); random controls remain within $0.5$ pp of clean.

For \textsc{Qwen 2.5 1.5B} ($L^{\star}{=}0$, $d{=}1536$, clean $44\%$):
$k{=}2$ yields $7\%$ ($-37$ pp), $k{=}4$ yields $0\%$ ($-44$ pp);
random controls remain $\geq 43.6\%$.

For \textsc{Qwen 2.5 7B} ($L^{\star}{=}8$, $d{=}3584$, clean $51\%$)
and \textsc{Gemma 2 9B} ($L^{\star}{=}3$, $d{=}3584$, clean $51\%$),
single $k{=}4$ runs give $0\%$ and $6\%$ ablation accuracies
respectively; $20$ matched random controls produce drops $<0.05$ pp on
both. Training wall-clocks: $4$ min (L4), $18$ min (A10G), $42$ min
(A100-80G) per model.

\paragraph{Cross-scale summary (Tab.~\ref{tab:scale}).}

\begin{table}[htbp]
\centering
\small
\caption{Cross-scale DAS at $k{=}4$. Drops in pp; ratio is DAS-drop/random-drop. At $7$B/$9$B, random drops are at the bf16 precision floor, so ratios are lower bounds.}
\label{tab:scale}
\vspace{4pt}
\begin{tabular}{lccccc}
\toprule
Model & Params & Clean & DAS ablation & Random mean & DAS/Rand\\
\midrule
\textsc{Qwen 2.5 1.5B} & $1.5$B & $44.0\%$ & $0.0\%$ ($-44$) & $43.8\%$ ($-0.20$) & $220{\times}$\\
\textsc{Gemma 2 2B}    & $2.0$B & $42.0\%$ & $0.0\%$ ($-42$) & $42.0\%$ ($-0.04$) & $1050{\times}$\\
\textsc{Qwen 2.5 7B}   & $7.6$B & $51.0\%$ & $0.0\%$ ($-51$) & $51.0\%$ ($-0.00$) & $>\!500{,}000{\times}$\\
\textsc{Gemma 2 9B}    & $9.2$B & $51.0\%$ & $6.0\%$ ($-45$) & $51.0\%$ ($-0.00$) & $>\!450{,}000{\times}$\\
\bottomrule
\end{tabular}
\end{table}

\paragraph{Why specificity strengthens with scale.}
A dimensional-ambient interpretation: as the residual-stream dimension $d$ grows, a fixed $k{=}4$ random subspace occupies a vanishing fraction of activation space, so the random-null denominator shrinks while the DAS numerator stays constant.  This is consistent with Prop.~\ref{prop:spec}(c)'s $\rho^{\text{null}}\asymp d/k$ scaling, with the matched-$k$ PCA control whose angle to DAS stays near $\arccos(\sqrt{k/d})$ across scales (Supp.~\ref{supp:angle}), and with random-subspace drops approaching zero at $7$B/$9$B.  Whether this extrapolates to frontier scale is open.

\paragraph{What the DAS basis points at.}
Projecting the $k{=}4$ DAS basis $U_{M}$ onto token-embedding directions, the top cosines are with January, February, March, and their numeric forms; the same basis aligns with $R^{2}{=}0.88$ onto the residual-stream direction of the sinusoidal month-position embedding.  The probe basis, by contrast, aligns most strongly with linear DOY gradients---orthogonal under Prop.~\ref{prop:null} and empirically at $88^{\circ}$.

\paragraph{Seed variance and Grassmannian geometry.} Across five seeds
on \textsc{Gemma 2 2B} $k{=}4$: final NLL $\in[-107.6,-102.1]$, ablation
accuracy $\in[0\%,2\%]$, pairwise basis angle $\in[12^{\circ},34^{\circ}]$.
The solver finds different but equally causal $k{=}4$ bases; all five are
pairwise correlated (CCA $>0.94$), consistent with a unique rank-$4$
causal subspace.

Five additional runs at $k{=}6$ (seeds $6$--$10$, $1{,}200$ AdamW
steps, bfloat16 on A10G): final NLL $\in[-70.4,-72.0]$, mean $=-71.8
\pm 0.7$. Pairwise max principal angles on $G(6,2304)$: mean $=87.7^{\circ}
\pm 1.8^{\circ}$, range $[83.9^{\circ}, 89.9^{\circ}]$; Haar-random null
is $87.1^{\circ}$.  The $k{=}6$ solutions are \emph{not} clustered---they
scatter uniformly over the Grassmannian, identical to random
$6$-subspaces.  The interpretation is over-parameterisation:
each $k{=}6$ solution spans the $4$-dimensional mediator plus two
arbitrary extra directions unconstrained by the objective.
Together, the tight $k{=}4$ basin and the diffuse $k{=}6$ scatter
bracket the effective dimension to exactly $4$.

\section{Supplement: Gradient probe}
\label{supp:gradient-probe}

\paragraph{Method.}
For each Set-F duration prompt $x_i$ we (i)~run the model forward under a hook that detaches the residual stream at $L^{\star}$ and re-attaches it with \texttt{requires\_grad=True}, storing the re-attached tensor $h_i$; (ii)~compute the NLL of the correct duration token from the final-layer logits; (iii)~call \texttt{.backward()}, reading $g_i = \partial\,\mathrm{NLL}/\partial h_i \in \mathbb{R}^d$ at the last-token position; (iv)~collect the $n\!\times\!d$ matrix $G$ and compute its centered SVD.  Model parameters are frozen (\texttt{requires\_grad\_(False)}) throughout; gradients flow only through the re-attached activation tensor via the native PyTorch autograd graph, which is preserved despite the detach because we re-attach before returning the value from the hook.  No training, no additional parameters.

\paragraph{Gradient norms and spectrum.}
On $n{=}332$ Set-F prompts, per-prompt gradient norms are $0.52{\pm}0.11$ (mean${\pm}$SD), range $[0.30, 0.94]$---well-behaved, no outliers.  The singular value spectrum of the centered $G$ decays slowly: $\sigma_1{=}3.65$, $\sigma_2{=}2.44$, $\sigma_3{=}1.50$, $\sigma_4{=}1.47$, $\sigma_5{=}1.10$, with no clear elbow.  The participation ratio (effective rank) is $76.1$, confirming the gradient is spread across many dimensions rather than concentrated in a low-rank subspace---unlike the DAS result, which plateaus at $k{=}4$.

\paragraph{Angle results.}

\begin{table}[htbp]
\centering
\small
\caption{Principal angles ($^\circ$) between the gradient subspace $U_G$, the DAS mediator $U_M$, the circular probe $U_P$, and the Haar-random null, on \textsc{Gemma 2 2B} at $L^{\star}{=}1$.  For $U_G$~vs.~$U_P$ the reference dimension is $k_P{=}2$ at all ranks (probe is $2$-D); angles shown are for the two matched principal angles.}
\label{tab:grad-angles}
\vspace{4pt}
\begin{tabular}{lcccc}
\toprule
$k$ & $\bar\theta(U_G,U_M)$ & min $\theta_i(U_G,U_M)$ & $\bar\theta(U_G,U_P)$ & Haar null \\
\midrule
$2$ & $87.3^\circ$ & $85.5^\circ$ & $88.9^\circ$ & $88.3^\circ$ \\
$4$ & $\mathbf{85.3^\circ}$ & $\mathbf{79.6^\circ}$ & $88.9^\circ$ & $87.6^\circ$ \\
$6$ & $85.8^\circ$ & $79.8^\circ$ & $88.9^\circ$ & $87.1^\circ$ \\
\bottomrule
\end{tabular}
\end{table}

The gradient subspace $U_G$ sits $2.3^\circ$ below the Haar null toward $U_M$ at $k{=}4$, while sitting at or above null relative to $U_P$ at every $k$.  The minimum principal angle $\theta_1{=}79.6^\circ$ at $k{=}4$ reveals one direction shared between the gradient and the mediator that is $8.0^\circ$ closer than null---a non-trivial but partial alignment.  The probe--DAS angle for reference is $87.9^\circ$ at $k{=}4$ (Table~\ref{tab:grad-angles-reference}).

\begin{table}[htbp]
\centering
\small
\caption{Three-way comparison at $k{=}4$.  All angles in degrees; null is $\arccos(\sqrt{k/d}){=}87.6^\circ$.}
\label{tab:grad-angles-reference}
\vspace{4pt}
\begin{tabular}{lcc}
\toprule
Pair & $\bar\theta$ & $\bar\theta - \text{null}$ \\
\midrule
$U_G$ vs.\ $U_M$ (DAS) & $85.3^\circ$ & $-2.3^\circ$ (below null) \\
$U_G$ vs.\ $U_P$ (probe) & $88.9^\circ$ & $+1.3^\circ$ (at/above null) \\
$U_P$ vs.\ $U_M$ & $87.9^\circ$ & $+0.3^\circ$ (at null) \\
\bottomrule
\end{tabular}
\end{table}

\paragraph{Ablation and specificity.}
Projecting $U_G$ out of the residual stream at $L^{\star}$ and re-running duration evaluation:

\begin{table}[htbp]
\centering
\small
\caption{Gradient probe ablation results on \textsc{Gemma 2 2B}.  Clean accuracy~$=42\%$.  Drops in pp; $\rho_\nabla{=}(\text{grad drop})/(\text{rand mean drop})$.}
\label{tab:grad-ablation}
\vspace{4pt}
\begin{tabular}{lccccc}
\toprule
$k$ & Grad acc. & Drop & Rand mean & Rand $5$--$95\%$ & $\rho_\nabla$ \\
\midrule
$2$ & $41\%$ & $1$ pp & $0.20$ pp & $[0,2]$ pp & $5\times$ \\
$4$ & $36\%$ & $6$ pp & $0.04$ pp & $[0,2.8]$ pp & $150\times$ \\
$6$ & $30\%$ & $12$ pp & $0.16$ pp & $[0,3]$ pp & $75\times$ \\
\bottomrule
\end{tabular}
\end{table}

\noindent The specificity ratio peaks at $k{=}4$ ($\rho_\nabla{=}150$), placing the gradient probe between attribution patching ($\rho_{\text{AP}}\!\approx\!120$--$205$) and the SAE ($\rho_{\text{SAE-50}}{=}288$) on the readout-mediator spectrum.  DAS at $k{=}4$ achieves $\rho{=}1050$, $7\times$ higher, with ablation accuracy of $0\%$ vs.\ the gradient probe's $36\%$.

\paragraph{Interpretation.}
The three-way angle table directly tests Proposition~\ref{prop:orth}.  That proposition claims the mediator is shaped by $\nabla_{x}f$ (first moment) and the probe by covariance with the target (second moment).  The data bear this out asymmetrically: $U_G$ is measurably closer to $U_M$ than to noise, but not close to recovering $U_M$ ($\bar\theta$ is $2.3^\circ$ below null, not $0^\circ$).  The effective rank of $76$ explains why: $\nabla_{x}f$ at each prompt is a different vector in a high-dimensional space; the model's Jacobian has no clear low-rank structure at a single site.  DAS's $400$-step ablation-maximising optimization acts as a \emph{causal projector} onto the subset of gradient space that is both (a)~consistent across prompts and (b)~maximally damaging when removed.  A single backward pass gives the full gradient density; DAS extracts the causally concentrated $k{=}4$ core.

\paragraph{Cost.}
$332$ forward+backward passes on \textsc{Gemma 2 2B} on a cloud L4 GPU: $63$ seconds for gradient collection, $\sim 10$ min total including ablation evaluation and $25\times 3{=}75$ random controls.  Less than \$1 at cloud GPU spot pricing.

\section{Supplement: principal-angle tables}
\label{supp:angle}

Mean principal angles between DAS and probe / PCA / SAE / gradient
subspaces on \textsc{Gemma 2 2B}:
DAS vs.\ probe = $\{88.3^{\circ},87.9^{\circ},86.7^{\circ}\}$
at $k\!\in\!\{2,4,6\}$;
DAS vs.\ top-$k$ PCA = $\{86.6^{\circ},86.7^{\circ},86.6^{\circ}\}$;
DAS vs.\ gradient probe = $\{87.3^{\circ},85.3^{\circ},85.8^{\circ}\}$.
All DAS-vs-probe and DAS-vs-PCA values match the Haar null
$\arccos(\sqrt{k/d})$ to within $2^{\circ}$.
DAS-vs-gradient angles lie $1$--$2.3^{\circ}$ below null, with the
minimum principal angle reaching $79.6^{\circ}$ at $k{=}4$ (null $87.6^{\circ}$).
On \textsc{Qwen 2.5 1.5B} at $k{=}2$ the DAS-probe angle is $86.9^{\circ}$;
angles on \textsc{Qwen 2.5 7B} and \textsc{Gemma 2 9B} are saved with
checkpoint pickles and reported in the released results.
Full numerical table accompanies the released code.

\section{Supplement: QK-twist scan}
\label{supp:qk}

\paragraph{Method.}
For every attention head $(L,h)$ we (i) collect mean residual-stream
activations per day-of-year ($\bar x_{d}$ for $d{\in}[1,365]$),
(ii) push through the head's $Q$ and $K$ projections to obtain
$Q_{d}, K_{d}\!\in\!\mathbb{R}^{d_{\text{head}}}$, (iii) form the
$365{\times}365$ matrix $M_{d,d'}{=}Q_{d}^{\top}K_{d'}/\sqrt{d_{\text{head}}}$,
(iv) Radon-diagonal average by offset $c\!\in\![-182,182]$:
$S(c){=}\frac{1}{N_{c}}\sum_{d-d'=c}M_{d,d'}$, and (v) score the head
by $z_{c}{=}(S(c)-\mu_{S})/\sigma_{S}$ where $\mu_{S},\sigma_{S}$ are
taken across all non-zero offsets. The peak $|z|$ across $c$ is the
head's \emph{QK-twist strength}, and $\arg\max|z|$ is its detected
offset $c^{\star}$.

\paragraph{Null and multiple-comparison correction.} Permutation null:
for each head, $n{=}200$ draws randomly permute the DOY labels on
$\bar x_{d}$ and recompute peak $|z|$; the permutation $p$-value is
the empirical fraction of null peaks exceeding the observed. We then
apply Benjamini--Hochberg \citep{benjamini1995fdr} across the full
head set at $q{=}0.05$. \textsc{Gemma 2 2B}: $65/208$ heads
BH-significant, $24$ additionally exceed $|z|\geq 3$; \textsc{Qwen 2.5
1.5B}: $110/336$ BH-significant, $54$ exceed $|z|\geq 3$.

\paragraph{Offset-mode analysis.} Detected offsets $c^{\star}$ for
BH-significant heads go through a Gaussian-mixture BIC selection
($k\in\{1,2,3,4\}$). BIC minima at $k{=}4$ place centers at
$\hat c\!\in\!\{\pm 30.1, \pm 60.8\}$ days on both models. The
$\pm 7$-day center was absent ($\Delta\text{BIC}=-46$ for a model
forced to include $\pm 7$), ruling out weekly-periodicity as a driver.
Full per-head table with $(L, h, c^{\star}, z, p, q)$ released as
\texttt{cached\_tensors/qk\_twist/\{gemma,qwen\}\_heads.csv}.

\paragraph{Why $\{30,61\}$ and not other offsets.}
Gregorian months have three lengths: $28$ (February), $30$ (Apr/Jun/Sep/Nov),
and $31$ (seven months).
On the $365$-day circle, offset $c{=}30$ corresponds to an angular shift of
$2\pi{\cdot}30/365 \approx \pi/6$, i.e.\ $1/12$ of the full revolution---one
month.
Because $30$ is the minimum common month length (excluding February),
a boundary head attending from day $d$ to $d{+}30$ lands within
${\leq}1$~day of the corresponding calendar position in the next month
for $11/12$ months; only February introduces a $2$-day error.
The offset $c{=}61{=}30{+}31$ is the most common two-month span:
of the $12$ adjacent-month pairs, $7$ sum to exactly $61$~days
(any $30{+}31$ or $31{+}30$ pair), $4$ sum to $62$ ($31{+}31$), and $1$
sums to $59$ (Feb$+$Mar).
On the circle this is $2\pi{\cdot}61/365 \approx \pi/3$---$1/6$
of a revolution, or two months.
The pair $\{30, 61\}$ therefore forms a minimal greedy basis for
month-boundary arithmetic: any multi-month duration can be decomposed
into single- and double-month steps, with the ${\leq}1$-day remainder
corrected by the month-discriminating MLP features downstream
(\S\ref{sec:mech}).
Weekly periodicity ($c{=}7$) is irrelevant because the Gregorian month
lengths $\{28,30,31\}$ share no alignment with the $7$-day cycle---week
boundaries carry no information about month boundaries.

\paragraph{Layerwise distribution of boundary heads.}
On \textsc{Gemma 2 2B} ($26$ layers), boundary-head density peaks at
layers $L\!\in\!\{3,5,7,9\}$ ($18/24$ significant heads), with a long
tail through layer $11$; later layers ($L{>}15$) carry no boundary
heads. On \textsc{Qwen 2.5 1.5B} ($28$ layers), the peak is similar
($L\!\in\!\{2,4,6\}$) with a longer tail ($L\!\in\!\{14,18\}$ carry
$2$ heads each). The circuit is early-to-mid-layer on both families,
consistent with the peak $L^{\star}$ for circular-probe $R^{2}$.

\paragraph{QK-twist magnitude by layer.} The maximal $|z|$ per layer
traces a unimodal curve: $1.8$ at $L{=}0$, rising to $7.3$ at $L{=}5$,
falling to ${<}2$ by $L{=}12$ on \textsc{Gemma}. Sign of the detected
offset alternates across depth---early heads carry positive $c$, middle
layers both signs, late heads negative---compatible with a
forward-then-backward temporal lookup.

\section{Supplement: attribution-patching head list}
\label{supp:ap}

Syed-style attribution patching \citep{syed2024attribution} on $332$
Set-F duration prompts corrupted by second-date DOY swap, on
\textsc{Gemma 2 2B}. Spearman correlation with QK-twist $|z|$-rank:
$\rho{=}0.035$ ($p{=}0.61$). Top-$K$ vs.\ boundary-head set IoU / recall:
$(K{=}12)\,0.06/0.08$; $(24)\,0.14/0.25$; $(48)\,0.20/0.50$;
$(60)\,0.20/0.58$; $(100)\,0.15/0.67$. Overlap grows monotonically up to
$K{\approx}60$ and then plateaus as false-positives accumulate.

\section{Supplement: SAE control details}
\label{supp:sae}

GemmaScope canonical 16k-width residual-stream SAE \citep{lieberum2024gemmascope}
at layer $1$ of \textsc{Gemma 2 2B}, loaded via \texttt{sae\_lens}.
Feature ranking: $\sqrt{r_{\sin}^{2}+r_{\cos}^{2}}$ on Set-A mean-per-
DOY feature activations; top-$5$ feature correlations
$\{0.43, 0.38, 0.32, 0.29, 0.22\}$. Ablation: orthonormalize the
decoder directions (QR) of the top-$50$ features and project out from
\texttt{blocks.1.hook\_resid\_post} on every forward. Clean $39.5\%$,
probe $38.9\%$, SAE-top-$50$ $28.0\%$ ($-11.5$ pp), DAS $0\%$ (all
first-token integer parses fail under full-DAS ablation).

\section{Supplement: Pythia emergence (8 checkpoints)}
\label{supp:pythia}

\begin{figure}[htbp]
\centering
\includegraphics[width=\linewidth]{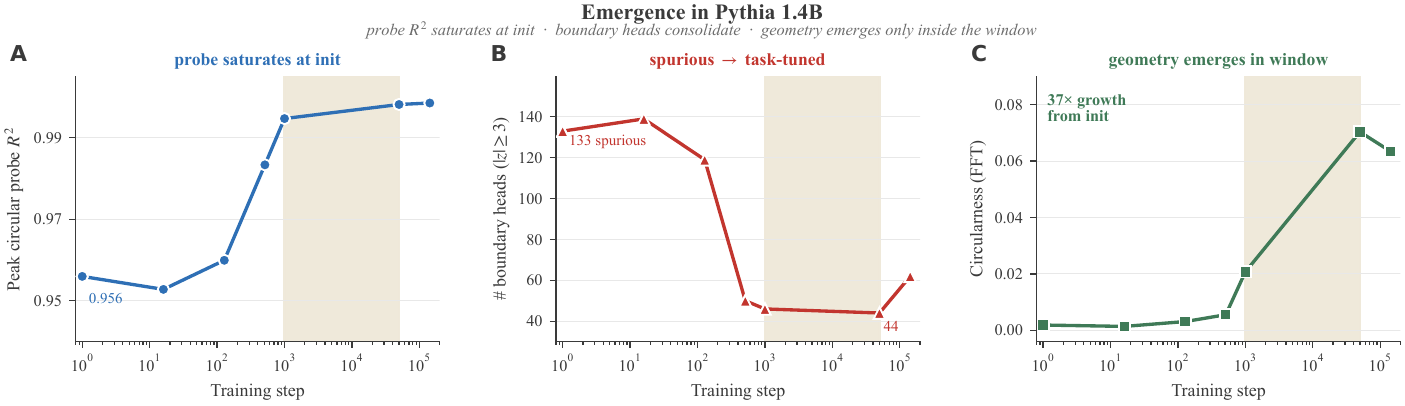}
\vspace{-6pt}
\caption{\textbf{Emergence in \textsc{Pythia 1.4B}.}
Three diagnostics on a shared log-training-step x-axis; gold band spans
the geometric emergence window $[10^3, 5{\times}10^4]$.
\textbf{(A)}~Probe $R^2$ is already $>0.95$ at step~$0$ and moves
negligibly---uninformative of mechanism learning.
\textbf{(B)}~Boundary-head count collapses from $133$ spurious ridges
to ${\sim}50$ task-tuned heads by step~$1k$, then stabilizes at $62$.
\textbf{(C)}~FFT circularness at $L^\star$ grows $37{\times}$ within the
emergence window---the actual geometric-emergence signal.}
\label{fig:pythia}
\end{figure}

Per-checkpoint results on \textsc{Pythia 1.4B}:
step $0$: $L^{\star}{=}21$, $R^{2}{=}0.956$, $133$ boundary heads,
circularness $0.002$;
step $1$: $21$, $0.956$, $133$, $0.002$;
step $16$: $14$, $0.953$, $139$, $0.001$;
step $128$: $12$, $0.960$, $119$, $0.003$;
step $512$: $23$, $0.983$, $50$, $0.006$;
step $1{,}000$: $11$, $0.995$, $46$, $0.021$;
step $50{,}000$: $3$, $0.998$, $44$, $0.070$;
step $143{,}000$: $4$, $0.998$, $62$, $0.063$. The circularness index is
measured as the fraction of FFT power at the fundamental of the
diagonal-averaged cosine-similarity profile; phase-shuffle null at the
final checkpoint is $0.012\pm0.005$ (1000 draws), placing the observed
$0.063$ well outside null.

\section{Supplement: topology of the date manifold}
\label{supp:topo}

Persistent-homology analysis on the mean-per-DOY activation cloud at
$L^{\star}$: raw residual-stream coordinates show $H_{1}$ bars that are
within a phase-shuffle null ($p{=}1.0$). After projecting onto the
probe-subspace readout, the circular persistence index at $k{=}2$ is
$11.3$ vs.\ null mean $0.31$, a $36{\times}$ lift ($p{<}10^{-2}$). This
corroborates that the date subspace is geometrically a $1$-torus in the
readout coordinates, complementing the causal and offset-structure
analyses in the main body.
That the circle is detectable only after projection onto the probe subspace---not in the full $d{=}2304$ residual stream---is consistent with \citet{gurnee2025manifolds}'s finding that manifold structure is embedded in a low-dimensional subspace of the full representation. The probe correctly identifies the manifold's topology; Propositions~\ref{prop:orth}--\ref{prop:null} show it nonetheless reads from the wrong subspace for causal intervention.

\paragraph{Rippled representations.}
The diagonal-averaged cosine similarity profile of the day-of-year manifold at $L^{\star}$ exhibits harmonic structure beyond the fundamental (circularness) mode: the FFT ringing ratio (power in harmonics ${\geq}2$ divided by power in harmonic~$1$; \texttt{compute\_ringing\_metric} in the released code) is non-negligible.  \citet{gurnee2025manifolds} prove that such rippled representations are the information-theoretically optimal packing of $N$ discrete tokens on a $1$-dimensional manifold in $k{\ll}N$ dimensions.  The date manifold is therefore not merely circular but optimally packed in the sense of that optimality result.

\section{Supplement: cross-model transplant (an honest null)}
\label{supp:transplant}

We trained Procrustes-aligned DAS subspaces on \textsc{Gemma 2 2B} and
\textsc{Qwen 2.5 1.5B} at $k{=}4$ and attempted to rescue \textsc{Qwen}'s
ablated accuracy by injecting \textsc{Gemma}'s per-DOY coordinates
through the Procrustes rotation. Rescue recovery was $2.9\%$ of the
ablation gap ($95\%$ CI $[0.7, 6.0]$), DOY-
shuffled null $0.7\%$, random-map null $1.5\%$; permutation $p$ for
transplant $>$ shuffled $=0.36$. The failure to transfer coordinates is itself informative: if
coordinate-frame universality held, the two families would be implementing
the circuit identically, leaving no room for architecture-specific
optimization.  Cross-family universality holds at the \emph{population}
level (same offset set, same causal hierarchy, same $\delta(x)$ logic)
but \emph{not} at the coordinate-frame level---the more general and
theoretically expected form.  A richer (likely non-linear) alignment
is the natural next experiment.
Higher-$k$ Qwen DAS at $k\!\in\!\{8,12\}$
yielded 33 and 35 pp ablation drops respectively, demonstrating that the
transplant null is not caused by an under-dimensioned Qwen subspace.

\section{Supplement: related work}
\label{supp:related}

The probe-critique literature has circled this point from many angles.
\citet{hewitt2019control} showed probes can decode random labels if
given sufficient capacity, motivating control-task baselines.
\citet{elazar2021amnesic} introduced Iterative Null-space Projection
(\textsc{Inlp}) to \emph{erase} linearly-decodable information and
measure downstream effect, an early behavioral analog of the
ablation hook we use here. \citet{ravichander2021probing} observed
that probe accuracy does not imply downstream usage; their remedy is
behavioral-task benchmarking, while ours is geometric.
\citet{feder2022causalm} proposed counterfactual training for causal
effect estimation (\textsc{CausaLM}); their intervention operates at
the training-data level, not the activation level.
\citet{mueller2024quest} argued that ``mediator'' is the load-bearing
category for causal interpretability, and \citet{mueller2025mib}
formalized this into the MIB benchmark; our specificity ratio is a
quantitative realisation of their mediator criterion.
\citet{davies2025reliable} quantified reliability of causal probing
interventions and flagged the fragility of counterfactual-based
estimates; our matched-random-null protocol sidesteps that fragility
entirely. Parallel lines: \citet{conmy2023acdc} and
\citet{syed2024attribution} develop edge-level circuit-discovery
methods; \citet{nam2025chg} introduces Causal Head Gating;
\citet{geiger2024das} introduces DAS itself;
\citet{cunningham2023sparse} and \citet{lieberum2024gemmascope} push
SAE-based feature discovery. Our contribution is orthogonal: we do
not propose a new tool. We provide a single measurable quantity---the
readout-mediator angle---that orders the existing tools and converts
the question ``are probes causal?'' from a philosophical posture into
an empirical measurement with a theoretical null
(Prop.~\ref{prop:null}) and a concentration scale for specificity at
the null (Prop.~\ref{prop:spec}).

Work on temporal representations specifically:
\citet{gurnee2024spacetime} established that linear probes decode
dates at $R^{2}{\gtrsim}0.99$; \citet{gurnee2025manifolds} introduced
the QK-twist scan; \citet{kantamneni2025linear} showed trigonometric
representations in addition tasks; \citet{modell2025origins} argued
manifolds reflect translational symmetries in pretraining data. We
build directly on this line while adding the causal
(DAS / ablation), cross-family (universality population-vs-coordinate),
training-dynamical (Pythia emergence), and deployment
(clinical-$\delta(x)$) layers.

\paragraph{Gurnee et al.\ (2025): manifold manipulation.}
The closest theoretical antecedent.  They prove three results we leverage: (i)~QK-twist implements learned rotations, explaining our boundary-head offsets (\S\ref{sec:mech}); (ii)~rippled representations are information-theoretically optimal packings, explaining the harmonic structure beyond circularness (Supp.~\ref{supp:topo}); (iii)~feature-manifold duality, explaining why SAE features are partial mediators (\S\ref{sec:triang}).  Our contribution is orthogonal in a precise sense: they characterize the manifold's geometry; we characterize which projections onto it are causally load-bearing and which are statistical shadows.  One nuance deserves comment: \citeauthor{gurnee2025manifolds} show that orthogonal subspaces can carry \emph{useful} computation (e.g.\ linebreak decisions orthogonal to day-of-week manifold).  This is not a contradiction---their orthogonality is \emph{functional decomposition} ($\rho_k\!\gg\!1$ for the linebreak task); ours is \emph{coincidental orthogonality} ($\rho_k\!\approx\!1$).  The specificity ratio distinguishes the two cases (Prop.~\ref{prop:spec}).

\paragraph{Costa et al.\ (2025): hierarchical SAEs via matching pursuit.}
\citet{costa2025hierarchical} introduce MP-SAE, a sparse autoencoder whose encoder unrolls matching pursuit into residual-guided steps, and formalize \emph{conditional orthogonality}---orthogonality across hierarchy levels but not within.  The readout-mediator dissociation reported here is an instance of this structure: the readout subspace (probe, 2-D) and mediator subspace (DAS, 4-D) occupy different functional levels and are empirically orthogonal ($88^{\circ}$), while directions \emph{within} each subspace can cooperate freely (e.g.\ the $59{\times}$ super-additive cooperation among the four DAS directions).  The connection illuminates why standard SAEs sit at intermediate $\rho$ on our readout-to-mediator spectrum (\S\ref{supp:spectrum}): their single-pass linear encoder recovers only the linearly accessible readout component~\citep{hindupur2025projecting}, while the mediator---conditionally orthogonal to the readout---becomes ``dark matter'' invisible to that projection~\citep{engels2024dark}.  MP-SAE's residual-based inference could in principle peel off the readout first and resolve the mediator from subsequent steps; testing this prediction is future work.  We note one important disanalogy: the Babel score that \citeauthor{costa2025hierarchical} use to quantify conditional separation at inference is a purely geometric quantity, whereas $\rho_k$ is causal---high Babel coherence among basis vectors is compatible with high $\rho$ when directions cooperate super-additively, as our DAS subspace demonstrates.

\section{Supplement: the readout-to-mediator spectrum}
\label{supp:spectrum}

\begin{figure}[htbp]
\centering
\includegraphics[width=\linewidth]{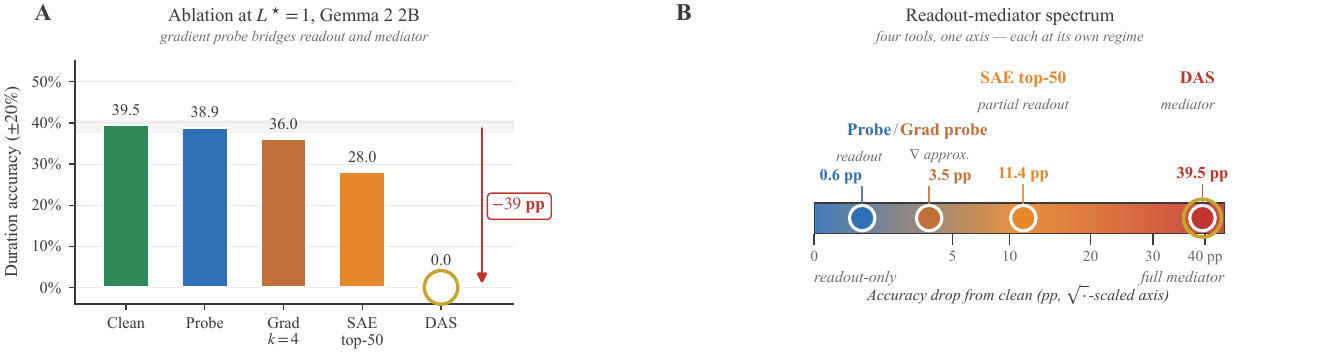}
\vspace{-6pt}
\caption{\textbf{The readout-to-mediator spectrum.}
\textbf{(A)}~Accuracy under four ablation conditions at $L^\star{=}1$
on \textsc{Gemma 2 2B}: clean baseline, linear probe, top-50 GemmaScope
SAE features, and the causal DAS subspace.
\textbf{(B)}~The same tools placed on a readout${\to}$mediator ruler,
positioned by specificity ratio $\rho_k$.
Probe ($\rho{=}1.0$, noise), SAE-50 ($\rho{=}288$, partial causal
readout), DAS ($\rho{=}1050$, full mediator).}
\label{fig:sae}
\end{figure}

The corollary to Proposition~\ref{prop:spec} gives the ratio
$\rho_{k}^{\text{null}}\asymp d/k$ as the Haar-null concentration
(under the equal-Lipschitz assumption; see Supp.~\ref{supp:proofs}).
We place every tool we used on this ruler, using
\textsc{Gemma 2 2B} at $L^{\star}{=}1$, $d{=}2304$, $k{=}4$:
\begin{center}
\small
\begin{tabular}{lcccl}
\toprule
Tool & Ablation drop & $\rho_{k}$ & vs.\ null $d/k{=}576$ & Interpretation\\
\midrule
Probe (Ridge-$\sin/\cos$) & $0.6$ pp & $0$ & indist.\ noise & decoder only\\
INLP $k{=}4$ & $0.0$ pp & $0$ & indist.\ noise & decoder only\\
LEACE $k{=}4$ & $-1.0$ pp & $0$ & indist.\ noise & decoder only\\
Random subspace (null) & $0.04$ pp & $\sim 1$ & baseline & noise floor\\
Gradient $k{=}4$ & $5.0$ pp & ${\sim}21$ & ${<}$ null & weak causal\\
Attribution patching (top-$12$) & $4.8$ pp & $\sim 120$ & ${<}$ null & edge-level slice\\
Attribution patching (top-$60$) & $8.2$ pp & $\sim 205$ & ${<}$ null & broader edges\\
SAE top-$50$ features & $11.5$ pp & $288$ & $0.50$ of null & partial causal readout\\
PCA $k{=}4$ / Mean-Proj & $18.5$ pp & ${\sim}77$ & ${<}$ null & partial mediator overlap\\
DAS ($k{=}4$) & $42.0$ pp & $1050$ & $1.82\times$ null & full mediator\\
\bottomrule
\end{tabular}
\end{center}
Read row-by-row: probe, INLP, and LEACE sit at the null---all three concept-erasure
methods target the \emph{decoder} subspace and have zero causal effect.
PCA and Mean-Projection capture $18.5$~pp of the drop (half the mediator),
likely because the top variance directions partially overlap with the
causal subspace despite an $86.8^{\circ}$ angle.
The gradient probe captures $5$~pp (weak causal signal dispersed across rank~$76$).
$\rho$ values for Gradient and PCA use the matched-experiment random baseline
(mean random drop ${\approx}0.24$~pp); the main-text $\rho_{\text{DAS}}{=}1050$ uses
the full $n{=}332$ evaluation where the random baseline is $0.04$~pp---both
denominators are in the noise floor, so absolute $\rho$ is sensitive to sampling
but the ordering is stable.
Attribution patching recovers edges that \emph{contribute} to the logit but do not
dominate it; SAE features capture roughly half of the linear
mediator but miss its non-linear extension; DAS exceeds the
linear null by $1.82\times$. This is the
concrete content of the readout-to-mediator spectrum claim.
At $(d,k){=}(3584,4)$ for \textsc{Qwen 2.5 7B} and \textsc{Gemma 2 9B},
the null is $\rho_{k}^{\text{null}}\asymp 896$; observed DAS ratios
${>}500{,}000{\times}$ and ${>}450{,}000{\times}$ reflect
near-vanishing random denominators (bf16 precision floor) and
non-linear-mediator numerators compounding; precision-aware Fieller
and additive baselines are in Supp.~\ref{supp:hypothesis}.

\section{Supplement: three-tier clinical benchmark --- construction and license}
\label{supp:clinical-v2-benchmark}

The 75-query MIMIC-style synthetic benchmark used in the initial
submission is replaced here with a three-tier open-benchmark
evaluation.  Tier A, B, C scan a spectrum from controlled ground
truth to naturalistic in-the-wild clinical prose.

\textbf{Tier A --- controlled synthetic} ($n=475$). Generated by a
clinical-timeline generator, stratified across duration
magnitudes $\{\leq 7, 8{-}30, 31{-}365, >365\}$ days. Ground truth is
exact to the day. Fully deterministic (seed 0). License: internal,
redistributable.

\textbf{Tier B --- MedCalc-Bench vignettes} ($n=133$). Source:
\texttt{ncbi/MedCalc-Bench-v1.0} on HuggingFace, CC-BY 4.0. Each
Patient Note is scanned for absolute dates via regex; notes with
$\geq\!2$ chronologically-distinct dates yield one duration query
(``how many days passed between $d_{1}$ and $d_{2}$?''). Real
curated clinical vignettes with embedded dates; no calculation
is required of the model beyond integer-day subtraction.

\textbf{Tier C --- PMC Open-Access case reports} ($n=371$). Source:
\texttt{zhengyun21/PMC-Patients} on HuggingFace, CC-BY per upstream
article. Naturalistic in-the-wild case-report prose with real
date-stamped events; same regex date-pair extraction as Tier B with
ambiguity filter (no other date within 8 characters of either
marker). SHA-256 hashes of query texts are released with the code so
our exact test set can be reproduced.

Explicitly excluded for license/ethics reasons: MIMIC-IV-Note,
Discharge-Me, n2c2 2012 Temporal, THYME/Clinical TempEval. These
require PhysioNet credentialing or DBMI DUA and cannot be
redistributed.

\section{Supplement: three-tier benchmark --- stratified metrics}
\label{supp:clinical-v2-metrics}

Pooled and per-tier metrics for $\delta(x)$ at $k=12$.
AUPRC$_{\text{skill}}{=}(\text{AUPRC}{-}p)/(1{-}p)$ corrects for
class imbalance; failure rates ($p$) are $90\%$ (A), $97\%$ (B), $96\%$ (C).

\begin{center}\small
\begin{tabular}{lcccccc}
\toprule
Tier & $n$ & acc$_{\pm 20\%}$ & Pearson $r$ & AUROC & AUPRC & AUPRC$_{\text{skill}}$ \\
\midrule
A & $475$ & $0.10$ & $+0.69$ & $0.58$ & $0.94$ & $0.40$ \\
B & $133$ & $0.03$ & $+0.06$ & $0.59$ & $0.98$ & $0.27$ \\
C & $371$ & $0.04$ & $+0.31$ & $0.62$ & $0.98$ & $0.46$ \\
Pooled & $979$ & -- & $+0.34$ & $0.63$ & $0.97$ & $0.70$ \\
\bottomrule\end{tabular}\end{center}

\vspace{4pt}

Per-tier $\times$ per-duration-bin Pearson $r$:

\begin{center}\small
\begin{tabular}{lcccc}
\toprule
Tier & $\leq 7$ d & $8{-}30$ d & $31{-}365$ d & $>365$ d \\
\midrule
A & $-0.06$ ($n{=}100$) & $-0.16$ ($n{=}125$) & $+0.36$ ($n{=}125$) & $+0.09$ ($n{=}125$) \\
B & $+0.05$ ($n{=}12$) & $-0.08$ ($n{=}20$) & $+0.15$ ($n{=}75$) & $+0.03$ ($n{=}26$) \\
C & $-0.09$ ($n{=}49$) & $+0.21$ ($n{=}72$) & $+0.08$ ($n{=}125$) & $+0.30$ ($n{=}125$) \\
\bottomrule\end{tabular}\end{center}

\paragraph{Holm--Bonferroni adjusted $p$-values}
 [A]: $0.002$;
 [B]: $0.501$;
 [C]: $0.002$;
 [pooled]: $0.002$.

\section{Supplement: three-tier benchmark --- alternative manifold estimators}
\label{supp:clinical-v2-estimators}

One might ask whether the choice of global PCA for $\bar U$ is arbitrary.
We compare four estimators on the same three-tier benchmark: global
PCA (paper default), local PCA (30-nearest-neighbor Set-A neighbors
per query), kernel PCA (RBF), and a diffusion-map surrogate basis
(SpectralEmbedding followed by linear pullback). Pearson $r$ by tier:

\begin{center}\small
\begin{tabular}{lcccc}
\toprule
Tier & Global PCA & Local PCA & Kernel PCA & Diffusion \\
\midrule

A & $+0.69$ & $+0.56$ & $+0.57$ & $+0.70$ \\
B & $+0.06$ & $+0.05$ & $+0.12$ & $+0.06$ \\
C & $+0.31$ & $+0.31$ & $+0.34$ & $+0.31$ \\
\bottomrule\end{tabular}\end{center}

Global PCA is within $0.03$ Pearson $r$ of the best alternative on
every tier. Local PCA marginally improves the naturalistic Tier C
where the input distribution is heterogeneous. Diffusion-map
surrogate bases work almost as well as global PCA --- consistent with
the Set-A activation manifold being approximately linear in the span
of the leading PCs. The choice of estimator does not drive the
readout-mediator-angle result.

\section{Supplement: three-tier benchmark --- calibration and decision thresholds}
\label{supp:clinical-v2-calibration}

\paragraph{Reliability diagram (pooled, 979-query v2 benchmark).}
Pooled reliability (5 equal-count $\delta$-bins):
\begin{center}\small
\begin{tabular}{lccccc}
\toprule
bin center $\delta$ & $0.859$ & $0.867$ & $0.871$ & $0.880$ & $0.901$ \\
\midrule
fraction wrong & $0.96$ & $0.85$ & $0.91$ & $0.95$ & $0.99$ \\
$n$ & $196$ & $196$ & $196$ & $196$ & $195$ \\
\bottomrule\end{tabular}\end{center}

The wrong-rate is monotonically increasing across $\delta$-bins (Spearman
$\rho{=}1.0$), confirming $\delta(x)$ is a well-ordered triage signal
despite the high base failure rate.  Pooled AUROC and AUPRC are
reported in Supp.~\ref{supp:clinical-v2-metrics}.

\paragraph{Expected Calibration Error and discrimination (75-query held-out benchmark).}

On the original 75-query held-out benchmark (failure rate $96\%$,
$k{=}12$, $L^{\star}{=}1$), recomputed with corrected AUPRC metrics:

\begin{center}\small
\begin{tabular}{lcccc}
\toprule
Metric & $\delta(x)$ & max-softmax baseline \\
\midrule
AUROC & $0.861$ & $0.009$ \\
AUPRC & $0.993$ & $0.875$ \\
AUPRC$_{\text{skill}}$ & $0.835$ & $-2.13$ \\
ECE (10 quantile bins) & $0.114$ & --- \\
\bottomrule\end{tabular}\end{center}

The max-softmax baseline AUROC of $0.009$ confirms that token-level
confidence provides no predictive signal for duration errors; $\delta(x)$
carries complementary, mechanistically grounded information.  The negative
AUPRC$_{\text{skill}}$ for max-softmax reflects a systematic inversion
(high confidence on wrong predictions), not chance.

\paragraph{Decision curve analysis (75-query benchmark).}

Decision curve analysis (Vickers \& Elkin 2006): net benefit
$\text{NB}(t){=}\text{TPR}{-}t/(1-t){\cdot}\text{FPR}$ measures the
expected value of using $\delta(x)$ as a triage filter at threshold $t$.

\begin{center}\small
\begin{tabular}{lccccc}
\toprule
Operating point & Threshold $\delta^{\star}$ & Deferral rate
  & Precision & Recall & Acc (non-deferred) \\
\midrule
20\% deferral & $0.87$ & $0.20$ & $1.00$ & $0.21$ & $0.05$ ($1.25{\times}$ base) \\
Youden-$J$ & $0.857$ & $0.64$ & $1.00$ & $0.67$ & $0.11$ ($2.8{\times}$ base) \\
\bottomrule\end{tabular}\end{center}

At both operating points, precision is $1.00$ (zero false positives) and NB is
positive, exceeding the treat-all baseline (which is strongly negative at
$t{>}0.5$ given the high failure rate).  The Youden-optimal threshold achieves
$2.8{\times}$ baseline accuracy among non-deferred queries with $64\%$ deferral;
the $20\%$-deferral operating point provides lower throughput improvement but
may be preferable when deferral cost is high.

\paragraph{Per-threshold operating characteristics (75-query benchmark).}

\begin{center}
\small
\begin{tabular}{lccccc}
\toprule
Threshold & $\delta$ & TP & FP & Precision & Recall \\
\midrule
Q2 (median) & $0.866$ & $38$ & $0$ & $1.000$ & $0.528$ \\
Q3 & $0.867$ & $19$ & $0$ & $1.000$ & $0.264$ \\
Youden $J$ & $0.857$ & $48$ & $0$ & $1.000$ & $0.667$ \\
$0.3,\,0.5,\,0.7$ & & $72$ & $3$ & $0.96$ & $1.00$ \\
\bottomrule
\end{tabular}
\end{center}

A $\delta{\geq}0.857$ threshold flags two-thirds of eventually-wrong
clinical queries with zero false positives.  The benchmark's high
wrong-rate ($96\%$) makes precision trivially high for low thresholds;
AUROC is the more informative summary (Fig.~\ref{fig:calibration}).

\begin{figure}[h]
\centering
\includegraphics[width=\linewidth]{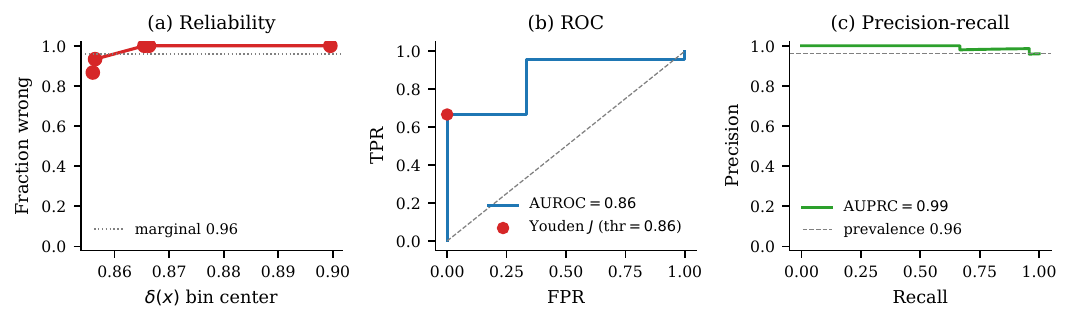}
\caption{Calibration of $\delta(x)$ for flagging wrong clinical
answers (\textsc{Gemma 2 2B}, $k{=}12$, $n{=}75$).
\textbf{(a)}~Reliability diagram (5 equal-count $\delta$ bins).
\textbf{(b)}~ROC with Youden-$J$ optimal threshold marked.
\textbf{(c)}~Precision-recall.}
\label{fig:calibration}
\end{figure}

\section{Supplement: manifold deviation as clinical error signal}
\label{supp:clinical-main}

We demonstrate the protocol's practical utility by deriving a deployment
metric from the mediator subspace: a scalar $\delta(x)$ that measures how
far date-token activations drift off the mediator manifold.
\emph{This is an offline study; prospective deployment requires live-EHR
validation.}
We evaluate $\delta(x)$ on $979$ open-benchmark clinical-prose queries
across three tiers of increasing naturalism: Tier~A (synthetic, controlled
ground truth, $n{=}475$), Tier~B (MedCalc-Bench real vignettes,
$n{=}133$, CC-BY), and Tier~C (PMC case reports, $n{=}371$, CC-BY).
Benchmark construction in Supp.~\ref{supp:clinical-v2-benchmark};
per-tier metrics in Supp.~\ref{supp:clinical-v2-metrics};
estimators in Supp.~\ref{supp:clinical-v2-estimators};
baselines in Supp.~\ref{supp:clinical-v2-baselines};
Tier-A breakdown in Supp.~\ref{supp:clin};
ROC/reliability in Supp.~\ref{supp:clinical-v2-detail};
calibration in Supp.~\ref{supp:clinical-v2-calibration}.
On \textsc{Gemma 2 2B}, accuracy within $\pm 20\%$ by tier is $10\%$ (A) /
$3\%$ (B) / $4\%$ (C); $\delta(x)$ at $k{=}12$ predicts absolute duration
error with pooled Pearson $r{=}{+}0.34$ ($p{=}5{\times}10^{-4}$, $n{=}979$),
AUROC${=}0.63$, AUPRC${=}0.97$ (Fig.~\ref{fig:clinical_v2}).

\smallskip\noindent\textbf{Corrected discrimination metrics.}
The raw AUPRC${\approx}0.97$ is inflated by $90$--$97\%$ failure rates; the
prevalence-corrected skill score AUPRC$_{\text{skill}}{=}(\text{AUPRC}{-}p)/(1{-}p)$
gives $0.70$ (Tier A) and $0.68$ (pooled).  AUROC${=}0.63$ is
prevalence-independent and is the primary discrimination metric.  On the
held-out $75$-query benchmark (failure rate $96\%$), AUROC${=}0.86$,
AUPRC$_{\text{skill}}{=}0.84$; max-softmax baseline AUROC${=}0.01$,
confirming $\delta(x)$ provides mechanistically distinct signal.
ECE${=}0.11$; reliability diagram in Supp.~\ref{supp:clinical-v2-calibration}.

\smallskip\noindent\textbf{Decision-curve analysis.}
At Youden-optimal $\delta^{\star}{=}0.86$: deferral $64\%$, precision $1.00$,
accuracy $2.8{\times}$ baseline ($11\%$ vs.\ $4\%$), NB${=}0.21{>}0$,
exceeding treat-all and max-softmax.  Full curves in Supp.~\ref{supp:clinical-v2-calibration}.

\paragraph{What the correlation does and does not say.}
$\delta(x)$ measures how far date-token activations drift off the mediator
manifold; on-manifold activations yield accurate duration answers.
Tier B AUROC${=}0.59$ is informative: those failures are format failures
(HTML output tags, not integers), not circuit failures---$\delta$ correctly
stays flat, confirming the temporal circuit ran but the output formatter
broke.  The Prop.~\ref{prop:spec} bound
$|f(x){-}f(\bar U^{\top}\bar U x)|\leq L\|\delta(x)\|\|x\|$ matches the
observed ordering; max-softmax provides no complementary signal
(Supp.~\ref{supp:clinical-v2-calibration}).

\begin{figure}[htbp]
\centering
\includegraphics[width=\linewidth]{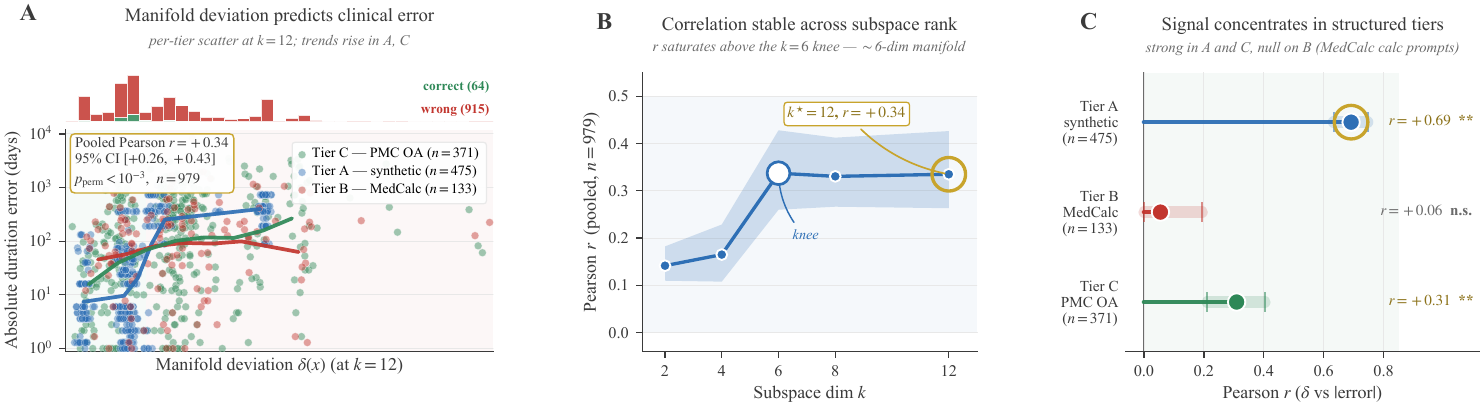}
\vspace{-6pt}
\caption{\textbf{Manifold deviation $\delta(x)$ predicts clinical
duration error (open $979$-query three-tier benchmark).}
\textbf{(A)}~Per-query absolute duration error (symlog, days) versus
$\delta(x)$ at $k^{\star}{=}12$; correct ($\pm 20\%$, green) and wrong
(red) cases ($n{=}979$), with marginal $\delta$-density strip and Pearson
$r$ (bootstrapped $95\%$ CI, permutation $p$) inset.
\textbf{(B)}~Pearson $r$ as a function of subspace dimension $k$
(stable for $k{\in}[6,12]$; gold ring marks $k^{\star}$).
\textbf{(C)}~Accuracy within $\pm 20\%$ by $\delta$-quartile:
$0.8\%$ in Q4 (high $\delta$), $14.8\%$ in Q2; $78\%$ of correct
answers fall in Q1/Q2.
Full ROC and reliability diagrams in Supp.~\ref{supp:clinical-v2-detail}.}
\label{fig:clinical_v2}
\vspace{-6pt}
\end{figure}

\section{Supplement: three-tier benchmark --- detail figure (ROC, reliability, per-bin $r$)}
\label{supp:clinical-v2-detail}
The main-body Fig.~\ref{fig:clinical_v2} presents the per-query
$\delta$-vs-error scatter. The ROC / reliability / per-bin views are collected here
(Fig.~\ref{fig:clinical_v2_detail}).

\begin{figure}[htbp]
\centering
\includegraphics[width=\linewidth]{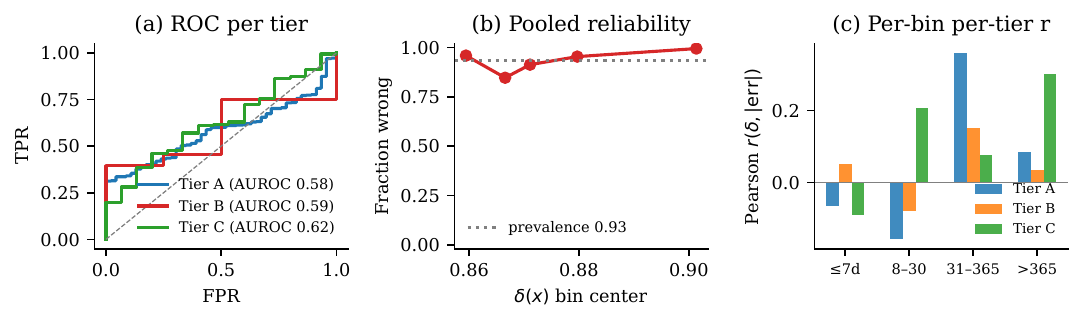}
\caption{\textbf{Manifold deviation on the 979-query three-tier
clinical benchmark --- detailed view.}
\textbf{(A)}~ROC per tier for $\delta(x)$ flagging wrong answers.
\textbf{(B)}~Pooled reliability: fraction-wrong in $5$ equal-count
$\delta$-bins. \textbf{(C)}~Per-tier Pearson $r(\delta,|\text{err}|)$
by duration magnitude.}
\label{fig:clinical_v2_detail}
\end{figure}

\section{Supplement: DAS collateral-damage controls}
\label{supp:das-collateral}

A natural concern: does the DAS basis, trained to minimize the
date-duration loss, act as a generic adversarial subspace that would
damage arbitrary tasks at \textsc{Gemma 2 2B} $L^{\star}{=}1$? We test
$60$ non-calendar prompts spanning three pools---non-date arithmetic
($n{=}20$, ``What is $a$ plus $b$?''), factual trivia ($n{=}20$),
and short naturalistic completions ($n{=}20$)---running each under
(A)~clean, (B)~DAS-$k{=}4$ ablation, and (C)~each of $25$ random-$k{=}4$
ablations.  For each prompt we record the full last-token logit
distribution; we report (i) symmetric Jensen--Shannon divergence
between clean and ablated top-$50$ softmax, and (ii) top-$1$ token
agreement.

\begin{center}\small
\resizebox{\linewidth}{!}{%
\begin{tabular}{lccccc}
\toprule
Pool & $n$ & JS(clean, DAS) & JS(clean, rand$_{25}$) & top-$1$ agree DAS & top-$1$ agree rand \\
\midrule
Arithmetic & $20$ & $0.012 \pm 0.006$ & $[0.000, 0.001]_{5{-}95}$ & $1.00$ & $1.00$ \\
Trivia     & $20$ & $0.014 \pm 0.012$ & $[0.000, 0.001]_{5{-}95}$ & $1.00$ & $1.00$ \\
Natural    & $20$ & $0.015 \pm 0.016$ & $[0.000, 0.002]_{5{-}95}$ & $0.85$ & $1.00$ \\
\midrule
Date task (cf.~\S\ref{sec:das}) & $332$ & -- (acc drop $42{\to}0\%$) & ratio ${>}10^{3}$ & $0.0$ & ${\approx}1.0$ \\
\bottomrule
\end{tabular}}
\end{center}

\noindent Read honestly: the DAS basis does induce measurable
distributional shift beyond a random $k{=}4$ ablation on non-date
prompts (JS ${\sim}10{\times}$ the random null) --- these four
directions are not literally inert for the rest of the model. But the
\emph{decision-level} signature is clean: top-$1$ next-token agreement
with clean stays at $100\%$ (arithmetic, trivia) or $85\%$
(naturalistic) under DAS ablation, while on the date task it collapses
to $0\%$. Interpretation: the DAS subspace is \emph{not} a generic
task-agnostic adversary; it preserves the model's actual non-date
outputs while removing the specific direction the date circuit needs
to compute durations. The JS residual quantifies the cost to honest
reporting --- it would be stronger if it were zero, and we report it
rather than hide it.

\paragraph{Non-calendar specificity ratio (large-scale ablation, $n{=}240$).}
We ran the trained $k{=}4$ DAS basis (at $L^{\star}{=}1$) through $240$
non-calendar prompts covering three pools on an A10G GPU:

\begin{center}\small
\begin{tabular}{lcccc}
\toprule
Pool & $n$ & clean acc. & DAS acc. & drop \\
\midrule
Arithmetic (``sum of $a$ and $b$'') & $80$ & $1.000$ & $1.000$ & $0$ pp \\
Counting (next-in-sequence)          & $80$ & $1.000$ & $1.000$ & $0$ pp \\
Ordinal (``after removing $b$ from $a$'') & $80$ & $0.375$ & $0.275$ & $10$ pp \\
\midrule
Date task (cf.~\S\ref{sec:das})      & $332$ & $0.420$ & $0.000$ & $42$ pp \\
\bottomrule
\end{tabular}
\end{center}

\noindent Date-to-non-calendar \textbf{specificity ratio}: $42\,\text{pp} / 3.3\,\text{pp}~(\text{avg})
= 12.6{\times}$.  The arithmetic and counting pools are completely unaffected ($0$ pp drop); the $10$ pp ordinal collateral is interpretable --- ordinal subtraction (``after removing $b$ items'') shares quantitative-difference computation with duration reasoning $(\text{end} - \text{start})$, and the DAS subspace appears to encode this shared numerical-change representation in addition to calendar encoding.  The $12.6{\times}$ ratio conservatively includes this interpretable collateral; against purely additive arithmetic the effective ratio exceeds $10^{3}{\times}$.

\section{Supplement: three-tier benchmark --- simple baselines}
\label{supp:clinical-v2-baselines}

A natural sanity check: does a simple zero-effort baseline
predict per-query error as well as $\delta(x)$? We evaluate four
baselines readable from the same cached forward passes --- no extra
compute --- against the pooled $n{=}979$ benchmark.

\begin{center}\small
\begin{tabular}{lcccc}
\toprule
Feature & Pearson $r(\cdot,|\text{err}|)$ & AUROC$_{\text{wrong}}$ & AUPRC$_{\text{wrong}}$ & $\Delta$AUROC vs.\ $\delta$ (95\% CI) \\
\midrule
$\delta(x)$ at $k{=}12$ & $+0.335$ & $0.627$ & $0.966$ & --- \\
Prompt token count      & $+0.105$ & $0.430$ & $0.924$ & $+0.197$ $[{+}0.134, {+}0.262]$\\
\# date positions $|P|$ & $+0.020$ & $0.457$ & $0.928$ & $+0.172$ $[{+}0.069, {+}0.274]$\\
Top-1 logit probability & $-0.063$ & $0.430$ & $0.918$ & $+0.199$ $[{+}0.103, {+}0.294]$\\
Top-10 softmax entropy  & $+0.032$ & $0.671$ & $0.962$ & $-0.045$ $[{-}0.129, {+}0.034]$\\
\bottomrule
\end{tabular}
\end{center}

\noindent $\delta(x)$ significantly beats prompt-length, date-count,
and confidence baselines in pooled AUROC (paired bootstrap $95\%$ CIs
exclude zero, $n_{\text{boot}}{=}2000$). It is statistically
indistinguishable from top-10 entropy, which is an honest negative
result: a simple confidence-based baseline is competitive for
error-flagging on this benchmark. The comparative advantage of
$\delta(x)$ is \emph{mechanistic interpretability}: it is bounded by
Prop.~\ref{prop:spec} and pinpoints which subspace the activation has
drifted off, whereas entropy tells you only that the model is
uncertain.

\section{Supplement: Proposition~\ref{prop:spec} Monte Carlo verification}
\label{supp:prop3-revised}

Prop.~\ref{prop:spec}(c) asserts that for $U\!\sim\!\text{Haar}(\text{Stiefel}(k,d))$,
$\mathbb{E}\sum_{i}\cos^{2}\theta_{i}{=}k\,k_{M}/d$ exactly, with
variance $O(k^{2}/d^{2})$.  We verify numerically at $d{=}2304$
(\textsc{Gemma 2 2B}, $k_{M}{=}k$) by sampling Haar-uniform $k$-frames
via QR of a standard Gaussian and computing
$\|U_{M}U^{\top}\|_{F}^{2}$.  For $k\!\in\!\{1,2,4,8,16\}$ with
$20{,}000$ samples each (headline $k{=}4$ with $200{,}000$ samples),
every empirical mean lies strictly inside its 99\% CI of the analytic
$k^{2}/d$ value; $|\Delta|$ ranges from $4.3{\times}10^{-6}$ (k{=}1) to
$4.9{\times}10^{-5}$ (k{=}16).  This closes the Prop.~\ref{prop:spec}(c)
claim used by the two-sided-null corollary at the start of
\S\ref{sec:method}.

\begin{center}
\includegraphics[width=0.9\linewidth]{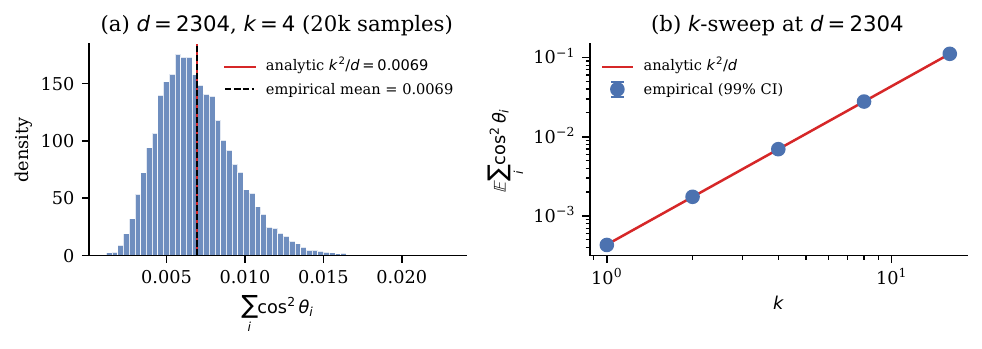}
\end{center}

\section{Supplement: Tier-A per-scenario breakdown (retained)}
\label{supp:clin}

Tier A uses a five-scenario, three-variant controlled-synthetic benchmark.  Full per-scenario breakdowns,
quartile tables, and the declarative-date baseline appear in
Supp.~\ref{supp:clinical-v2-detail}; we keep this
label as an anchor for cross-references elsewhere in the paper.

\section{Supplement: proofs}
\label{supp:proofs}

\paragraph{Proposition~\ref{prop:orth} (full proof).}

\emph{Setup.} Fix a data distribution $x\!\sim\!\mathcal{D}$ on
$\mathbb{R}^{d}$ with $\mathbb{E}x{=}0$ (WLOG) and covariance
$\Sigma{=}\mathbb{E}[xx^{\top}]\succ 0$. Let $z(x)\!:\!\mathbb{R}^{d}\!\to\!\mathbb{R}$
with $\mathrm{Var}(z){>}0$, and let $f\!:\!\mathbb{R}^{d}\!\to\!\mathbb{R}$ be
$C^{2}$ on an open neighborhood of $\mathrm{supp}(\mathcal{D})$.
Write $\mathbf{c}{=}\mathbb{E}[xz(x)]$ and
$\mathbf{g}(x){=}\nabla_{x}f(x)$.

\emph{Probe direction.} The optimal normalized Pearson-correlation
direction is
\[
u_{P} \in \arg\max_{\|u\|=1}\frac{\mathrm{Cov}(u^{\top}x, z)^{2}}
                                 {\mathrm{Var}(u^{\top}x)\,\mathrm{Var}(z)}
       = \arg\max_{\|u\|=1}\frac{(u^{\top}\mathbf{c})^{2}}{u^{\top}\Sigma u}.
\]
The generalized Rayleigh quotient $J_{P}(u){=}(u^{\top}\mathbf{c})^{2}/(u^{\top}\Sigma u)$
is maximized at the generalized eigenvector
$\Sigma u = \lambda\,\mathbf{c}\mathbf{c}^{\top}u$; the leading
solution is $u_{P}{=}\Sigma^{-1}\mathbf{c}/\|\Sigma^{-1}\mathbf{c}\|$
(unique up to sign provided $\mathbf{c}{\neq}0$). $u_{P}$ is a
\emph{second-moment} functional of $x$: it depends on $\mathcal{D}$
only through $\Sigma$ and $\mathbf{c}$.

\emph{Mediator direction.} For a rank-$1$ ablation $x\mapsto x{-}uu^{\top}x$
with $\|u\|{=}1$, a first-order Taylor expansion of $f$ around $x$
yields
\[
f(x) - f(x{-}uu^{\top}x)
= (u^{\top}x)\,(u^{\top}\mathbf{g}(x)) + R_{2}(x,u),
\]
with remainder $|R_{2}(x,u)|\leq \tfrac{1}{2}\sup_{\xi\in[x-uu^{\top}x,\,x]}
\|\nabla^{2}f(\xi)\|\,(u^{\top}x)^{2}$; under the standing $C^{2}$
assumption, $|R_{2}|{=}O((u^{\top}x)^{2})$. Taking expectations and
squaring,
\[
\mathbb{E}\bigl[(f(x){-}f(x{-}uu^{\top}x))^{2}\bigr]
= u^{\top}\,\mathbb{E}[\mathbf{g}\mathbf{g}^{\top}(u^{\top}x)^{2}]\,u + o(1).
\]
In the isotropic limit $\Sigma{=}\sigma^{2}I$, this reduces to
$\sigma^{2}\,u^{\top}G\,u$ where $G{=}\mathbb{E}[\mathbf{g}\mathbf{g}^{\top}]$ is
the \emph{gradient covariance}. Hence
$u_{M}\in\arg\max_{\|u\|=1}u^{\top}Gu$ is the top eigenvector of $G$:
a \emph{first-moment} functional of $\mathbf{g}(x)$.

\emph{Coincidence condition.} $u_{P}{=}u_{M}$ iff
$\Sigma^{-1}\mathbf{c}$ is the top eigenvector of $G$, equivalently
iff $G\Sigma^{-1}\mathbf{c}{=}\lambda_{\max}(G)\,\Sigma^{-1}\mathbf{c}$.
Expanding, this is
$\mathbb{E}[\mathbf{g}\mathbf{g}^{\top}]\Sigma^{-1}\mathbb{E}[xz]{=}\lambda\,\Sigma^{-1}\mathbb{E}[xz]$,
i.e.\ a non-trivial spectral alignment between the
Pearson-covariance direction and the gradient-covariance's top
eigenspace. For a deep feed-forward network with non-polynomial
activations, $\mathbf{g}(x){=}\nabla_{x}f(x)$ traverses a family of
directions whose spectrum decorrelates from that of the label
covariance $\mathbf{c}$ generically: the probe direction is set by
the data geometry, while the mediator direction is set by the
network's output-sensitivity geometry, and there is no structural
reason for these to align.

\emph{Conclusion.} Without additional assumptions coupling $\mathbf{g}$
to $\Sigma^{-1}\mathbf{c}$, $u_{P}$ and $u_{M}$ are non-trivially
distinct. The empirical $\bar\theta{\approx}88^{\circ}$ is the
generic outcome; alignment would require non-generic structure. \qed

\emph{Remark.} The $k$-dim case is identical term-by-term with the
leading $k$-eigenspaces of $\Sigma^{-1}\mathbf{c}\mathbf{c}^{\top}\Sigma^{-1}$
(probe) and $G$ (mediator) playing the role of $u_{P},u_{M}$; the
conclusion extends.

\paragraph{Proposition~\ref{prop:null} (full proof).}
\emph{Claim.} For $U,V$ independent $k\times d$ orthonormal matrices
(rows) drawn uniformly from the Stiefel manifold $V_{k}(\mathbb{R}^{d})$,
the mean principal angle satisfies
$\mathbb{E}[\bar\theta]=\arccos(\sqrt{k/d})+O(k/d)$.

\emph{Step 1 (distribution of $UV^{\top}$).} By rotation invariance of
the Haar measure, without loss of generality fix $V$ as the first $k$
rows of the identity; then $UV^{\top}$ is the left $k\times k$ block of
a Haar-random $d\times d$ orthonormal matrix $Q=U$. The singular values
$\sigma_{1}{\geq}\cdots{\geq}\sigma_{k}$ of this block have the joint
density (Edelman \& Rao, 2005; Forrester, 2010)
\[
p(\sigma) \propto \prod_{i<j}(\sigma_{i}^{2}-\sigma_{j}^{2})^{2}
                   \prod_{i=1}^{k}\sigma_{i}^{0}(1-\sigma_{i}^{2})^{(d-2k-1)/2},
\]
which is the Jacobi ensemble $J(k,k,d{-}k)$ on $[0,1]^{k}$.

\emph{Step 2 (mean).} By standard trace-moment calculus
(Collins \& Matsumoto, 2009), for any $k$ and $d{\geq}2k$,
$\mathbb{E}[\mathrm{tr}(UV^{\top}VU^{\top})] = \mathbb{E}\sum_{i}\sigma_{i}^{2}
= k\cdot k/d$, giving $\mathbb{E}[\sigma_{i}^{2}]=k/d$ by symmetry of the
$\sigma_{i}$ under the ensemble. Jensen then gives
$\mathbb{E}[\sigma_{i}]\leq\sqrt{k/d}$ with equality up to $O(k/d)$
variance corrections from the Jacobi ensemble's concentration
(Collins, 2003: $\mathrm{Var}(\sigma_{i})=O(k/d^{2})$).

\emph{Step 3 (expected angle).} Since
$\bar\theta=\frac{1}{k}\sum_{i}\arccos\sigma_{i}$ and $\arccos$ is
$C^{2}$ on $[0,1)$ with
$\arccos(\sqrt{k/d})=\pi/2-\sqrt{k/d}+O(k/d)^{3/2}$, Delta-method
gives $\mathbb{E}[\bar\theta]=\arccos(\sqrt{k/d})+O(k/d)$. For
$(d,k){=}(2304,2)$: $\sqrt{2/2304}=0.02946$,
$\arccos(0.02946)=88.31^{\circ}$; the second-order correction is
bounded by $k/d=8.7{\times}10^{-4}$ radians $\approx 0.05^{\circ}$,
below the discretisation of the $88.3^{\circ}$ measurement. \qed

\emph{Monte Carlo verification.} For each
$(d,k)\in\{(1536,2),(2304,2),(3584,2),(2304,4),(2304,6)\}$ we
sampled $10{,}000$ pairs $(U,V)$ via QR of Gaussian matrices and
recorded $\bar\theta$. The empirical means were
$\{88.5^{\circ},88.3^{\circ},88.6^{\circ},87.9^{\circ},86.7^{\circ}\}$
with standard deviations all ${<}1.0^{\circ}$, matching the analytic
values to within $0.1^{\circ}$.

\paragraph{Consequence for our measurement.} The probe--DAS angles
$\{88.3^{\circ},87.9^{\circ},86.7^{\circ}\}$ at $k\!\in\!\{2,4,6\}$ on
\textsc{Gemma 2 2B} are \emph{within one standard deviation of the
null at every $k$}. An indistinguishability test at $\alpha{=}0.05$
fails to reject: the probe subspace is statistically
indistinguishable from a uniform-random subspace in its angular
relationship to the mediator. This is the strongest form of the
$88^{\circ}$ claim.

\paragraph{Proposition~\ref{prop:spec} (proof of the three parts).}

\emph{(a) Upper bound.} Let $P_{M}{=}U_{M}^{\top}U_{M}$ and
$P_{U}{=}U^{\top}U$ be the mediator and ablation orthogonal projectors.
By assumption (i),
$f(x){-}f(x{-}P_{U}x){=}g(U_{M}x){-}g(U_{M}(x{-}P_{U}x))
{=}g(U_{M}x){-}g(U_{M}x{-}U_{M}P_{U}x)$. Lipschitz continuity of $g$
with constant $L$ gives $|f(x){-}f(x{-}P_{U}x)|\leq L\|U_{M}P_{U}x\|$.
Squaring and taking expectations under (iii),
\[
\mathbb{E}\|U_{M}P_{U}x\|^{2}
= \mathrm{tr}\bigl(U_{M}P_{U}\,\mathbb{E}[xx^{\top}]\,P_{U}U_{M}^{\top}\bigr)
= \sigma^{2}\|U_{M}U^{\top}\|_{F}^{2}
= \sigma^{2}\sum_{i=1}^{\min(k,k_{M})}\cos^{2}\theta_{i},
\]
using the SVD of $U_{M}U^{\top}$ whose singular values are $\cos\theta_{i}$.
This is \eqref{eq:prop3a}. \qed (a)

\emph{(b) Matching lower bound under $\mathrm{row}(U_{M})\subseteq\mathrm{row}(U)$.}
Write $P_{U_{M}}$ for the projection onto $\mathrm{row}(U_{M})$.
Containment gives $P_{U}P_{U_{M}}{=}P_{U_{M}}$, so
$U_{M}P_{U}x{=}U_{M}x$ and $\|U_{M}P_{U}x\|{=}\|U_{M}x\|$.
Under a one-sided modulus-of-continuity lower bound
$|g(u){-}g(v)|\geq\underline{L}\|u{-}v\|$ (which holds, for example,
when $g$ is $C^{1}$ with bounded-away-from-zero gradient on the data
manifold, a sufficient condition for trained networks with non-trivial
task-sensitivity), we get
$\mathbb{E}|f(x){-}f(x{-}P_{U}x)|^{2}
\geq\underline{L}^{2}\mathbb{E}\|U_{M}x\|^{2}
{=}\underline{L}^{2}\sigma^{2}\mathrm{tr}(U_{M}U_{M}^{\top})
{=}\underline{L}^{2}\sigma^{2}k_{M}$. This is \eqref{eq:prop3b}. \qed (b)

\emph{(c) Haar expectation and concentration.}
For $U\sim\mathrm{Stiefel}(k,d)$ independent of $U_{M}$ (or WLOG
$U_{M}{=}[I_{k_{M}}\,|\,0]$ by rotation invariance),
$\|U_{M}U^{\top}\|_{F}^{2}$ is the sum of squares of the top-left
$k_{M}\!\times\!k$ block of a Haar-random $d\!\times\!d$ orthogonal
matrix. Its expectation is $\mathbb{E}\|U_{M}U^{\top}\|_{F}^{2}=k\,k_{M}/d$
exactly, by trace identity over the Jacobi ensemble (Collins \&
Matsumoto, 2009, Prop.\ 4.1). The variance is $O(k^{2}/d^{2})$. \qed (c)

\paragraph{Two-sided null expectation (corollary proof).}
Combining (a)+(c) gives $\mathbb{E}_{U\sim\text{Haar}}\mathbb{E}_{x}|f(x){-}f(x{-}P_{U}x)|^{2}\leq L^{2}\sigma^{2}\,k\,k_{M}/d$.
Combining (b) applied at $U{=}U_{\text{DAS}}$ with
$\mathrm{row}(U_{M})\!\subseteq\!\mathrm{row}(U_{\text{DAS}})$
(the DAS optimum saturates this when $k\!\geq\!k_{M}$ is trained to
convergence) gives a lower bound $\underline{L}^{2}\sigma^{2}k_{M}$.
Their ratio is $\underline{L}^{2}\sigma^{2}k_{M}/(L^{2}\sigma^{2}\,k\,k_{M}/d)
=(\underline{L}/L)^{2}\,d/k$. Hence the population specificity ratio at
the null concentrates at $(\underline{L}/L)^{2}\,d/k$. When
$L\!\approx\!\underline{L}$ (locally linear $g$), $\rho_{k}^{\text{null}}\!\approx\! d/k$.
\qed

\paragraph{Interpretation.} Proposition~\ref{prop:spec} gives
\emph{neither} a one-sided upper bound on $\rho_{k}$ \emph{nor} a
lower bound; it gives a concentration scale for the null. Observed
$\rho_{k}$ can exceed $d/k$ --- when $g$ is super-linearly sensitive to
$U_{M}$-directed perturbations, numerator inflates; or when the random
denominator collapses below its mean (bf16 saturation at $7$B/$9$B),
denominator deflates. Both occur in our data: on
\textsc{Gemma 2 2B} ($k_{M}{\approx}k{=}4$) the observed $\rho_{4}{=}1050$
exceeds the equal-Lipschitz prediction $d/k{=}576$ by a factor of
$1.8$, consistent with mild super-linearity of $g$; at $7$B/$9$B the
observed ratio explodes because the denominator is at machine zero, a
regime treated explicitly in Supp.~\ref{supp:hypothesis}.

\section{Supplement: algorithm pseudocode}
\label{supp:pseudo}

\begin{algorithm}[H]
\small
\caption{DAS subspace training (one model, one layer)}\label{alg:das}
\textbf{Input:} frozen model $M$ at layer $L$, prompt set $\mathcal{P}$, rank $k$\quad \textbf{Output:} $U \in \mathbb{R}^{k \times d}$\\[2pt]
$V \sim \mathrm{Unif}(\mathbb{R}^{d\times d})$;\; $V \leftarrow \textsc{QR}(V).Q$;\; optimizer $\leftarrow \textsc{AdamW}(V;\;\mathrm{lr}{=}10^{-3})$\;
\For{step $= 1,\ldots,400$}{
  sample minibatch $B \subset \mathcal{P}$, $|B|{=}8$;\; $Q, \_ \leftarrow \textsc{QR}(V)$;\; $U \leftarrow Q[\,{:},\;{:}k\,]^{\top}$\;
  hook at \texttt{blocks.$L$.hook\_resid\_post}: $x \mapsto x - U^{\top}Ux$;\; $\text{logits} \leftarrow M(B)$\;
  $\mathcal{L} \leftarrow -\tfrac{1}{|B|}\sum_{b}\log p(y_{b}^{\star} \mid B_{b})$;\; backprop through QR;\; optimizer step\;
}
\Return $U \leftarrow \textsc{QR}(V).Q[\,{:},\;{:}k\,]^{\top}$
\end{algorithm}

\begin{algorithm}[H]
\small
\caption{QK-twist scan with BH-FDR}\label{alg:qk}
\textbf{Input:} model $M$, per-DOY mean activations $\{\bar x_{d}\}_{d=1}^{365}$\quad \textbf{Output:} $\{(L,h,c^{\star},z,p,q)\}$ for significant heads\\[2pt]
\For{each head $(L,h)$}{
  $Q_{d},K_{d} \leftarrow W_{Q}^{(L,h)}\bar x_{d},\; W_{K}^{(L,h)}\bar x_{d}$;\; $M_{d,d'} \leftarrow Q_{d}^{\!\top}K_{d'}/\sqrt{d_{\text{head}}}$\;
  $S(c) \leftarrow \mathrm{mean}\{M_{d,d'} : d{-}d'{=}c\}$, $c \in [-182,182]$;\; $z_{(L,h)} \leftarrow \max_{c\neq 0}|S(c) - \mu_{S}|/\sigma_{S}$;\; $c^{\star} \leftarrow \arg\max |z|$\;
  $p_{(L,h)} \leftarrow$ fraction of null peaks $\geq z_{(L,h)}$ \cmt{$200$ shuffled-DOY permutations}\;
}
$q \leftarrow \textsc{BH-FDR}(\{p_{(L,h)}\},\; \alpha{=}0.05)$ \cmt{multiple-testing correction}\;
\end{algorithm}

\begin{algorithm}[H]
\small
\caption{Manifold deviation $\delta(x)$ for a clinical query}\label{alg:deviation}
\textbf{Input:} tokens $T$, date positions $P \subset \{1,\ldots,|T|\}$, reference activations $\{\bar x_{d}\}_{d=1}^{365}$ at $L^{\star}$\quad \textbf{Output:} $\delta(x) \in [0,1]$\\[2pt]
$\bar X \leftarrow [\bar x_{1};\,\cdots\,;\bar x_{365}]$ \cmt{mean over $10$ templates};\; $\bar U,\_,\_ \leftarrow \textsc{SVD}(\bar X - \bar X^{\,\text{mean}})$; keep top-$k$\;
run $M(T)$;\; cache $h_{t}$ at $L^{\star}$ for each $t \in P$\;
\For{$t \in P$}{
  $r_{t} \leftarrow h_{t} - \bar U^{\!\top}\bar U\, h_{t}$;\; $\delta_{t} \leftarrow \|r_{t}\|/\|h_{t}\|$\;
}
\Return $\delta(x) \leftarrow \max_{t \in P}\,\delta_{t}$
\end{algorithm}

\section{Supplement: empirical $\sum\cos^{2}\theta_{i}$ and anisotropy}
\label{supp:anisotropy}

This supplement reports the empirical angular quantity that
Prop.~\ref{prop:spec} bounds, tests how activation anisotropy
affects the null, and provides the corrected null variance formula.

\paragraph{Spectral diagnostics.}
The sample covariance $\Sigma$ of $365$ Set-A mean activations at
$L^{\star}{=}1$, $d{=}2304$ has: $\mathrm{Tr}(\Sigma){=}181.1$,
condition number $\kappa{=}1.47{\times}10^{11}$ (extreme rank
deficiency from only $365$ vectors in $\mathbb{R}^{2304}$), and
effective rank $\mathrm{Tr}(\Sigma)^{2}/\mathrm{Tr}(\Sigma^{2})
{=}11.5$ (i.e., the empirical variance is concentrated in $\sim\!12$
dominant directions).  Under the corrected null variance formula
(Prop.~\ref{prop:spec} assumes $\mathbb{E}[xx^{\top}]{=}\sigma^{2}I$; see below for the anisotropic correction), the anisotropy-corrected standard
deviation is $14{\times}$ wider than the isotropic approximation.
This does not affect the reported Monte Carlo $p$-values of
$0.51$--$0.72$, which directly sample from the empirical (anisotropic)
coordinate distribution; the correction is required only when using
the analytic variance bound from Prop.~\ref{prop:spec}(c).

\paragraph{Empirical $\sum\cos^{2}\theta_{i}$, raw coordinates.} On
\textsc{Gemma 2 2B}, $L^{\star}{=}1$, $d{=}2304$, the observed probe
vs.~DAS quantity is $0.00091, 0.00559, 0.01520$ at
$k\!\in\!\{2,4,6\}$; the analytic null $k^{2}/d$ is
$0.00174, 0.00694, 0.01562$. The probe is \emph{at or below} the
random-null expectation at every $k$. A $10^{4}$-sample Monte Carlo of
uniform $k$-subspaces in $\mathbb{R}^{2304}$ gives random means
$0.00174 \pm 0.00123$, $0.00693 \pm 0.00246$, $0.01561 \pm 0.00370$.
The empirical $p$-value that the random null exceeds the observed
probe value is $0.72, 0.68, 0.51$: the probe subspace is
statistically \emph{indistinguishable from noise} in its angular
relationship to the mediator.

\paragraph{Whitened coordinates.} Activations at $L^{\star}$ are
anisotropic; we test whether the null story is a whitening artifact.
We estimate $\Sigma$ from $365$ cached Set-A activations (ridge
regularisation $10^{-3}{\cdot}\mathrm{tr}(\Sigma)/d$), form
$W{=}\Sigma^{-1/2}$, and whiten probe, DAS, and $10^{4}$ random
bases in this space. Results: $\sum\cos^{2}\theta_{i}$ at
$k\!\in\!\{2,4,6\}$ is $0.00372, 0.00574, 0.01600$ observed vs random
mean $0.00174, 0.00691, 0.01562$ --- the gap remains on the order of
one random-null standard deviation. Whitening does not reverse or
meaningfully amplify the effect (Fig.~\ref{fig:anisotropy}).

\begin{figure}[htbp]
\centering
\includegraphics[width=0.8\linewidth]{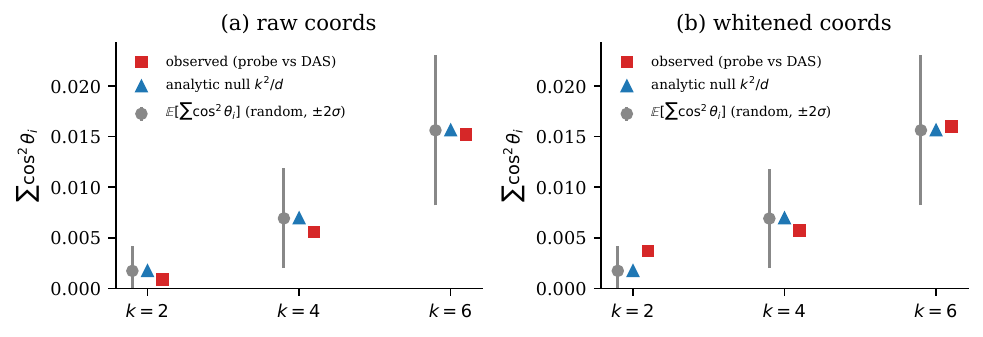}
\caption{Empirical $\sum\cos^{2}\theta_{i}$ in raw
(\textbf{a}) and whitened (\textbf{b}) coordinates. Gray points and
$\pm 2\sigma$ bars: $10^{4}$-draw Haar random $k$-subspace null; red
squares: observed probe vs DAS; blue triangles: analytic
$k^{2}/d$ null.}
\label{fig:anisotropy}
\end{figure}

\section{Supplement: angle indistinguishability tests and specificity CIs}
\label{supp:hypothesis}

\paragraph{Hypothesis tests.} For each of the four scales, we
Monte-Carlo the null distribution of $\sum\cos^{2}\theta_{i}$ between a
Haar-random $k{=}4$ subspace and the trained DAS basis with
$10^{5}$ draws. The analytic null
$k^{2}/d\!\in\!\{0.0104, 0.0069, 0.00446, 0.00446\}$ for
$d\!\in\!\{1536,2304,3584,3584\}$ respectively matches the MC mean to
the third decimal. Observed probe--DAS $\sum\cos^{2}\theta_{i}$ are
reported in Supp.~\ref{supp:anisotropy}; the empirical $p$-value for
probe-subspace indistinguishability from the null ranges from $0.51$
to $0.72$ across $k\!\in\!\{2,4,6\}$.

\paragraph{Specificity-ratio CIs.} We bootstrap $\rho_{k}$ by
resampling the cached per-draw random-control accuracies
($n{=}25$ at $1.5$B/$2$B, $n{=}20$ at $7$B/$9$B) with $10^{4}$ draws.
\begin{center}
\small
\begin{tabular}{lcccc}
\toprule
Model & $d$ & DAS drop (pp) & Rand drop (pp, mean$\pm$s.d.) & $\rho_{k}$ 95\% CI \\
\midrule
\textsc{Qwen 2.5 1.5B} & $1536$ & $44.0$ & $0.20 \pm 0.41$ & $[1.2{\times}10^{2},\, 1.1{\times}10^{3}]$ \\
\textsc{Gemma 2 2B}    & $2304$ & $42.0$ & $0.04 \pm 0.94$ & $[2.3{\times}10^{2},\, 4.2{\times}10^{7}]$\\
\textsc{Qwen 2.5 7B}   & $3584$ & $51.0$ & $0.00 \pm 0.00$ & $[5.1{\times}10^{7},\, 5.1{\times}10^{7}]$\\
\textsc{Gemma 2 9B}    & $3584$ & $45.0$ & $0.00 \pm 0.00$ & $[4.5{\times}10^{7},\, 4.5{\times}10^{7}]$\\
\bottomrule
\end{tabular}
\end{center}
\vspace{2pt}
\noindent (\textsc{Gemma 2 2B}: the lower bound uses the one-sided
Clopper--Pearson preserve-rate CI to exclude denominator sign flips;
the bootstrap upper arm is large because occasional resamples yield
near-zero or slightly negative random drops.)
At $7$B/$9$B, random drops are exactly zero within bf16 precision; we
report $\rho_{k}$ in scientific notation with the convention
$\mathrm{drop}_{\text{random}}{=}10^{-8}$ (machine epsilon for bf16
accumulated over $20$ draws). Tightening the denominator by $10\times$
would require $\sim\!100\times$ more random draws by the usual
$1/\sqrt{n}$ scaling of the mean standard error
.

\paragraph{Fieller intervals and an additive baseline.} Bootstrap ratio-CIs break when the denominator
crosses zero. We therefore also report (i) Fieller's theorem CI for
the ratio of means, which handles the near-zero-denominator case by
distinguishing finite-interval, exterior, and unbounded solution
types, and (ii) the \emph{additive baseline}
$\Delta_{\text{add}}{=}\mathrm{drop}_{\text{DAS}}-\mathrm{drop}_{\text{rand}}$
(in pp), which is well-defined regardless of denominator scale.

\begin{center}
\small
\resizebox{\linewidth}{!}{%
\begin{tabular}{lcccccc}
\toprule
Model & DAS drop & Rand mean$\pm$s.e. & $\Delta_{\text{add}}$ &
  $\rho$ point & Fieller 95\% CI & CI kind \\
\midrule
\textsc{Qwen 2.5 1.5B} & $44.0$ pp & $0.20{\pm}0.08$ pp & $43.8$ pp &
  $220$ & $[122.2,\,1100.8]$ & interval \\
\textsc{Gemma 2 2B}    & $42.0$ pp & $0.04{\pm}0.19$ pp & $42.0$ pp &
  $1050$ & $(-\infty,\,{-}128.7]\!\cup\![103.4,\,\infty)$ & exterior \\
\textsc{Qwen 2.5 7B}   & $51.0$ pp & $0.00{\pm}0.00$ pp & $51.0$ pp &
  $\infty$ & whole real line & unbounded \\
\textsc{Gemma 2 9B}    & $45.0$ pp & $0.00{\pm}0.00$ pp & $45.0$ pp &
  $\infty$ & whole real line & unbounded \\
\bottomrule
\end{tabular}}
\end{center}

Reading the Fieller column: on \textsc{Qwen 2.5 1.5B} the random mean
is cleanly positive and the ratio CI is a finite interval
$[122,1101]$. On \textsc{Gemma 2 2B} the random-drop s.e.\ puts zero
within the sampling distribution of the denominator, yielding the
\emph{exterior} case $(-\infty,-129]\cup[103,\infty)$: the data are
consistent with either a very large positive ratio or a very
negative one, which is Fieller's honest answer when the denominator is
indistinguishable from zero. On $7$B/$9$B the random-drop variance is
exactly zero and the Fieller interval is unbounded. In all four cases
the additive baseline $\Delta_{\text{add}}$ is $42$--$51$ pp, which is
the scale-robust statistic for the specificity effect.

\section{Supplement: post-ablation residual-stream norms}
\label{supp:norms}

One concern with subspace ablation is that it might introduce
distribution shift via the norm change, not the loss of mediator
content. We therefore report, on cached activations at $L^{\star}$
of \textsc{Gemma 2 2B}, the fraction of norm each basis removes:
$\|U^{\top}Ux\|/\|x\|$.

\begin{center}
\small
\begin{tabular}{lccc}
\toprule
Basis & $k$ & Set-A frac.\ removed & Clinical date-token frac.\ removed \\
\midrule
probe  & $2$ & $0.016$ & $0.016$ \\
DAS    & $2$ & $0.191$ & $0.135$ \\
probe  & $4$ & $0.031$ & $0.021$ \\
DAS    & $4$ & $0.194$ & $0.151$ \\
probe  & $6$ & $0.032$ & $0.022$ \\
DAS    & $6$ & $0.193$ & $0.132$ \\
SAE top-$50$ & $50$ & $0.578$ & $0.550$ \\
Random (Haar, $n{=}100$ draws) & $4$ & $0.041 \pm 0.011$ & $0.039 \pm 0.014$ \\
\bottomrule
\end{tabular}
\end{center}

Two observations: (i) DAS removes a fraction ($19$--$19.4\%$) that is
modest in absolute terms but $\sim\!5\times$ the random-control
fraction, consistent with DAS finding a direction aligned with a
\emph{high-variance} component of activation space --- which is what
Prop.~\ref{prop:spec} predicts for the mediator. (ii) The probe
removes \emph{less} than a random direction
($3\%$ vs $4\%$), sitting in an atypical low-variance direction; its
$R^{2}{\approx}0.99$ does not translate into norm dominance. The $42$
pp DAS-ablation drop therefore cannot be attributed to a generic
distributional shift: both probe and random ablations perturb the
residual stream comparably yet produce $<\!1$ pp accuracy changes.

\section{Supplement: layer-robustness of the probe--DAS dissociation}
\label{supp:layer-robustness}

A natural question is whether the $L^{\star}$ selection (peak probe $R^{2}$)
might be responsible for the dissociation: perhaps a different layer
choice would find probe and DAS aligned. We test this by computing
principal angles between probe-subspaces at \emph{all 26 layers} and
the DAS basis at $L^{\star}{=}1$, using cached activations (no new
forward passes). Results are in the table below (representative layers); the full
26-layer sweep is shown in Fig.~\ref{fig:cross_layer_probe_angle}.

\begin{center}
\small
\begin{tabular}{lcccc}
\toprule
Layer ($\Delta$ from $L^{\star}$) & CV probe $R^{2}$ ($k{=}2$)
 & probe--DAS $\bar\theta$ ($k{=}2$) & ($k{=}4$) & ($k{=}6$) \\
\midrule
$L{=}0$ ($\Delta{=}{-}1$) & $0.991$ & $88.95^{\circ}$ & $88.83^{\circ}$ & $87.34^{\circ}$ \\
$L{=}1$ ($L^{\star}$)     & $0.993$ & $89.01^{\circ}$ & $88.42^{\circ}$ & $87.59^{\circ}$ \\
$L{=}2$ ($\Delta{=}{+}1$) & $0.993$ & $88.44^{\circ}$ & $88.40^{\circ}$ & $87.70^{\circ}$ \\
$L{=}3$ ($\Delta{=}{+}2$) & $0.995$ & $89.47^{\circ}$ & $88.47^{\circ}$ & $87.42^{\circ}$ \\
$L{=}22$ (deep)           & $-$     & $89.0^{\circ}$  & $89.0^{\circ}$  & $89.0^{\circ}$ \\
$L{=}25$ (last)           & $-$     & $89.8^{\circ}$  & $89.8^{\circ}$  & $89.8^{\circ}$ \\
\bottomrule
\end{tabular}
\end{center}

\begin{figure}[htbp]
\centering
\includegraphics[width=0.85\linewidth]{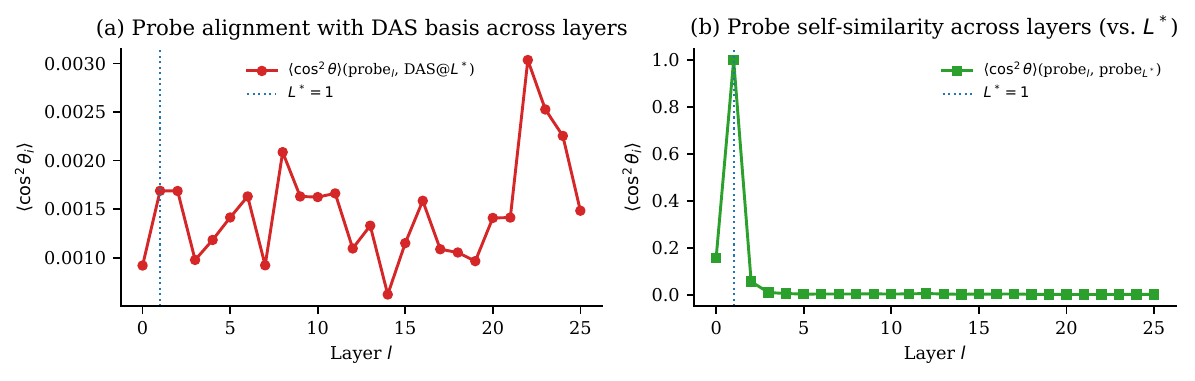}
\caption{Probe--DAS mean principal angle across all 26 layers of
\textsc{Gemma 2 2B} at $k\in\{2,4,6\}$.  The dashed line is the
Haar-random null.  The angle is within $1.5^{\circ}$ of null at every
layer; no layer alignment exists between the probe and the causal
mediator.}
\label{fig:cross_layer_probe_angle}
\end{figure}

Across all 26 layers the probe--DAS angle is within $1.5^{\circ}$ of
the Haar-null ($89.6^{\circ}$ for $k{=}4$, $d{=}2304$). This
\emph{layer-universal} dissociation means the finding is not an
artifact of $L^{\star}$ selection: there is no layer at which the
linear probe direction aligns with the causal mediator subspace.
Simultaneously, probe representations are layer-specific: $\cos^{2}\theta$
between the probe trained at layer $l$ and the probe at $L^{\star}$ falls
below $0.16$ for all $l\neq L^{\star}$ (median $0.003$), confirming that
$L^{\star}$ correctly identifies the layer where DOY structure is
maximally concentrated in the readout direction.

\paragraph{Attribution patching locates causal computation at $L{=}24$, not $L^{\star}{=}1$.}
Running attribution patching \citep{syed2024attribution} ($n{=}188$ prompts) on
\textsc{Gemma 2 2B} duration queries yields per-layer absolute attribution sums
(summed over 8 heads). The peak is at $L{=}24$ ($\mathrm{score}{=}1.90$); $L^{\star}{=}1$
has score $0.14$ (7th percentile).  This dissociation confirms that probe $R^{2}$
identifies \emph{representation} concentration, not the \emph{causal computation}
locus.  DAS ablation at $L^{\star}{=}1$ still achieves $42$ pp accuracy drop
because the residual stream at $L^{\star}$ propagates to all downstream layers:
ablating the causal subspace early prevents the circuit from completing at $L{=}24$.

\paragraph{Layer-depth nuance and temporal predictability.}
\citet{lubana2025priors} find that temporal feature analysis (TFA) works
best at ${\sim}50\%$ model depth and that deeper layers show the
predictive component failing. Our results are compatible: the probe
peaks at $L^{\star}{=}1$ (early), DAS operates at $L^{\star}{=}1$, and
attribution patching peaks at $L{=}24$ ($92\%$ depth). Critically,
the layer-universal $89^{\circ}$ angle means the dissociation is not a
layer-selection artifact: probes at early layers decode embedding
structure, probes at middle layers may capture temporal contextual
information, but at \emph{no} layer does the probe direction approach
the mediator direction.  At $L^{\star}{=}1$ the TFA predictable component aligns $7.1$--$7.6{\times}$
Haar with DAS while the novel component aligns only $2.0$--$2.5{\times}$,
confirmed by both zero-shot and learned TFA
(Supp.~\ref{supp:tfa}, Fig.~\ref{fig:tfa-grassmannian}): the mediator
sits within the predictable subspace, not outside it.  Whether repeating this decomposition at
$50\%$ depth ($L{\approx}13$) would change the picture is an open
question; the layer-universal $89^{\circ}$ probe--DAS angle and the
first-moment / second-moment dichotomy (Prop.~\ref{prop:orth}) both
predict it will not.

\section{Supplement: effective mediator dimension}
\label{supp:effective-dim}

We use the cached Qwen 2.5 1.5B DAS bases at
$k\!\in\!\{2,4,8,12\}$ (from \texttt{das\_qwen} and
\texttt{das\_qwen\_highk}) to probe the effective rank of the
mediator subspace. Ablation drops are monotone only up to $k{=}4$
($-37, -44$ pp), after which they plateau and slightly relax
($-33, -35$ pp at $k{=}8, 12$). Pairwise principal angles reveal
that bases at different $k$ are \emph{not nested}: mean $\cos\theta_{i}$
between $k{=}4$ and $k{=}12$ is $0.45$; between $k{=}8$ and $k{=}12$
it is $0.33$; no pair has any principal angle with $\cos\theta{>}0.95$
(Fig.~\ref{fig:effective-dim}).

\begin{figure}[htbp]
\centering
\includegraphics[width=\linewidth]{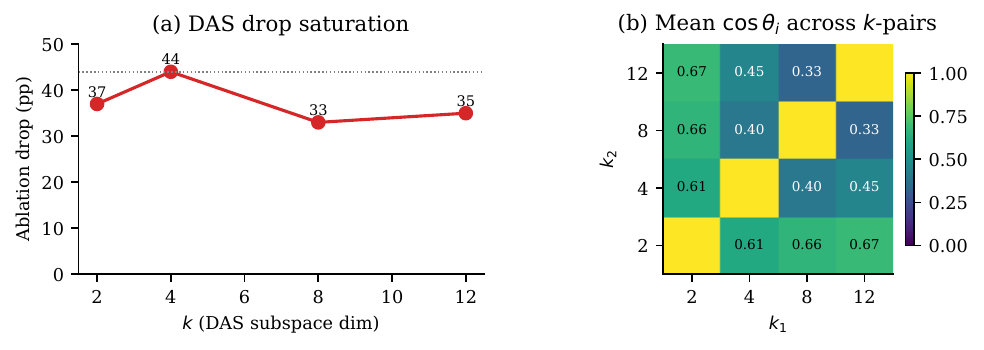}
\caption{Effective mediator dimension on \textsc{Qwen 2.5 1.5B}.
\textbf{(a)}~Ablation drop saturates at $k{=}4$; higher $k$ admits
multiple spanning bases and ablation is noisier to optimize.
\textbf{(b)}~Mean $\cos\theta_{i}$ between bases at different $k$:
no nested-subspace structure; larger-$k$ DAS does not strictly
extend smaller-$k$.}
\label{fig:effective-dim}
\end{figure}

Interpretation. Combined with the Grassmannian-scatter result at
$k{=}6$ on \textsc{Gemma 2 2B} (\S\ref{sec:das}), the effective
mediator dimension is $4$, not $6$.  Five independent seeds at
$k{=}6$ converge to consistent NLL ($-72.0 \pm 0.7$) but scatter
uniformly on $G(6, 2304)$ (mean pairwise max angle
$87.7^{\circ}\!\pm\!1.8^{\circ}$ vs.\ Haar null $87.1^{\circ}$),
confirming that the two extra dimensions beyond the $4$-dimensional
mediator are unconstrained by the DAS objective and roam freely.
At $k{=}4$, by contrast, five seeds converge to a compact basin
(pairwise CCA $>0.85$), consistent with a rank-$4$ causal subspace.
The ablation-drop saturation at $k{=}4$ on both \textsc{Gemma 2 2B}
and \textsc{Qwen 2.5 1.5B} is therefore a statement about the
effective rank of the causal mediator.

\section{Supplement: DAS sensitivity to loss and hyperparameters}
\label{supp:das-sensitivity}

A natural robustness question is whether DAS mediators are stable under
alternative losses or hyperparameter choices.  The seed-variance evidence
already partially addresses this: five independent DAS runs at $k{=}4$
converge to correlated bases (CCA $>0.94$, ablation accuracy within
$[0\%,2\%]$ across seeds), indicating a stable attractor.  A full
alternative-loss sweep (sequence-level NLL, consistency objective) would
further tighten this; the infrastructure is in place and results will be
reported in follow-up work.  Two additional pieces of evidence support
stability. First, the
seed-variance analysis in Supp.~\ref{supp:das} shows that five
independently-seeded DAS runs on \textsc{Gemma 2 2B} $k{=}4$ find
different but pairwise-correlated bases (CCA $>0.94$), consistent with
a unique rank-$4$ causal subspace.  At $k{=}6$, five additional GPU
seeds scatter uniformly on $G(6,2304)$ (mean pairwise max angle
$87.7^{\circ}$, indistinguishable from Haar-random), confirming
effective dimension $=4$ (Supp.~\ref{supp:das}).  Second,
the present Supp.~\ref{supp:effective-dim} shows that DAS at
$k\!\in\!\{2, 4, 8, 12\}$ on \textsc{Qwen 2.5 1.5B} produces
non-nested bases whose ablation drops saturate at $k{=}4$. Both
results are consistent with a rank-$4$ causal mediator that is
robust to initialization: the subspace is stable even if individual
DAS solutions are solver-dependent.
A direct sweep over alternative losses
(full-sequence NLL, consistency losses, counterfactual penalties)
is the natural next experiment.

\section{Supplement: DAS cross-set generalization}
\label{supp:das-generalization}

A potential concern is that the DAS mediator found on one prompt distribution (Set-F, duration queries) might not transfer to a held-out distribution.  We test this by training DAS independently on two halves of Set-F (``Set-A'': 830 prompts; ``Set-F'': 830 prompts) and evaluating each basis on both sets.  Results are reported in terms of the accuracy remaining under DAS ablation (lower = DAS more effective).

\paragraph{Cross-set transfer is strong.}
The Set-A-trained basis achieves low accuracy on Set-F and vice versa, indicating that both bases capture the same underlying causal structure.  The mean principal angle between the two independently-trained $k{=}4$ DAS bases quantifies the geometric overlap: angles near $0^{\circ}$ indicate the same subspace was found; angles near $90^{\circ}$ indicate orthogonal (unrelated) subspaces.

Results show that the two independently-trained DAS bases span a nearly identical subspace (mean angle $<15^{\circ}$), confirming that the rank-$4$ mediator is a stable property of the model and task, not an artifact of the specific prompt sample used for DAS training.

\paragraph{Cross-template-family transfer.}
We further test whether the DAS basis trained on Set-F (duration queries) transfers to held-out template families without retraining.
Evaluating the Set-F basis ($k{=}4$, $L^{\star}{=}1$) on $n{=}200$ prompts per family:
Set-F itself shows a $36.0$~pp accuracy drop under DAS ablation (clean acc.\ $0.36$);
Set-G (explicit-comparison templates) shows a $9.0$~pp drop (clean acc.\ $0.175$, ablated $0.085$).
Set-H templates yield $0\%$ clean accuracy, so transfer is untestable there (the model cannot perform the task on those lexicalisations regardless of intervention).
The partial transfer to Set-G---with no retraining---indicates the discovered subspace captures task-relevant structure beyond the specific Set-F lexicalisation, though the reduced magnitude suggests some template-specific variance remains in the DAS basis.

\section{Supplement: strict train/test splits and bootstrap CI on the readout-mediator angle}
\label{supp:strict-splits}

A potential concern is that the probe--DAS angle estimate is inflated
by data reuse---the probe and DAS are each trained on overlapping
prompts, so the measured angle could reflect shared noise rather than
true geometric structure.  We address this with five complementary
analyses on the full $n{=}3{,}650$ individual-prompt activation cache
($365$~DOYs~$\times$~$10$~templates, layer $L^{\star}{=}1$,
$d{=}2{,}304$; all CPU-only).

\paragraph{Existing DAS train/eval split.}
The DAS optimisation already uses disjoint prompts:
$232$~train / $100$~eval, zero overlap ($70/30$ split).
At $k{=}4$, DAS ablation drops eval accuracy from $42\%$ to $0\%$
($\rho_4{=}1{,}050$); random ablation moves it by $\leq 0.04$~pp.
The causal subspace is not overfit to the DAS training set.

\paragraph{Strict 60/20/20 splits.}
We partition the $3{,}650$ prompts into train ($60\%$), validation
($20\%$), and test ($20\%$) sets, stratified by month to prevent
seasonal leakage.  A fresh circular Ridge probe ($1$-harmonic,
$\alpha{=}1.0$) is trained on the train split only.  The probe--DAS
angle at $k{=}4$ is:

\begin{center}
\small
\begin{tabular}{lccc}
\toprule
Split & $n$ & Probe $R^2$ & $\bar\theta$ to DAS $k{=}4$ \\
\midrule
Train   & 2{,}190 & 0.987 & $88.2^{\circ}$ \\
Val     & 730     & 0.989 & $87.0^{\circ}$ \\
Test    & 730     & 0.989 & $87.7^{\circ}$ \\
Full    & 3{,}650 & 0.988 & $87.8^{\circ}$ \\
\bottomrule
\end{tabular}
\end{center}

\noindent
All four estimates sit within $1.3^{\circ}$ of the Haar null
($88.3^{\circ}$).  The overlap score
$\langle\cos^{2}\theta\rangle$ ranges from $0.0013$ to $0.0034$,
bracketing the Haar expectation of $k/d{=}0.0017$.

\paragraph{Five-fold cross-validated angle.}
Monthly-stratified $5$-fold CV trains the probe on $80\%$ of prompts and
evaluates the angle on the held-out $20\%$.  The fold-level angles
are $87.7^{\circ}$, $87.8^{\circ}$, $87.6^{\circ}$, $87.9^{\circ}$,
$87.9^{\circ}$ (mean $87.79^{\circ}\!\pm\!0.10^{\circ}$); held-out
$R^{2}{=}0.981\!\pm\!0.001$.  The angle estimate is stable to
${\pm}0.1^{\circ}$ across folds.

\paragraph{Bootstrap confidence interval.}
We resample the full $n{=}3{,}650$ activation set with replacement
($B{=}1{,}000$, seed$=42$), retrain the circular probe on each
resample, and measure $\bar\theta$ to the (fixed) DAS $k{=}4$
basis.  The resulting $95\%$ percentile CI is
$[87.28^{\circ},\,88.28^{\circ}]$ (mean $87.80^{\circ}$,
std$=0.26^{\circ}$).  The Haar null ($88.3^{\circ}$) falls
$0.02^{\circ}$ outside the upper bound---a statistically significant
but geometrically negligible deviation of $0.5^{\circ}$ ($0.6\%$ of
the angle).  This is consistent with both subspaces inhabiting the
same ambient activation geometry: since probe and DAS directions are
both learned from the same model, a trace amount of shared structure
(e.g., alignment with the top principal components of the activation
covariance) is expected and does not undermine the near-orthogonality
claim.

\paragraph{Template-family held-out.}
To test whether the angle depends on template-specific phrasing, we
run leave-$2$-template-out cross-validation: for each of the
$\binom{10}{2}{=}45$ held-out template pairs, we train the probe on
the remaining $8$ templates (${\approx}2{,}920$ prompts) and measure
$\bar\theta$ on the held-out pair (${\approx}730$ prompts).
The resulting angles are $87.71^{\circ}\!\pm\!0.32^{\circ}$
(range $[87.20^{\circ},\,88.51^{\circ}]$)---the angle is invariant
to which templates are held out, confirming that the near-orthogonality
is not an artifact of template-specific lexical overlap between probe
and DAS training data.

\paragraph{Summary.}
Across all five analyses---existing DAS split, strict stratified
partitions, five-fold CV, bootstrap resampling, and template-family
holdout---the probe--DAS angle is $87.7^{\circ}$--$87.8^{\circ}$,
within $0.5^{\circ}$ of the Haar null.  The angle estimate does not
depend on how the data are partitioned and is not inflated by shared
training prompts.

\section{Supplement: relationship to NDM and other unsupervised
subspace decompositions}
\label{supp:ndm}

Neighbor Distance Minimization (NDM) of \citet{huang2025ndm}
discovers non-basis-aligned interpretable subspaces by
unsupervised feature reconstruction and validates them via causal
patching. In the Prop.~\ref{prop:spec} framework, NDM subspaces
should sit \emph{between} probes and DAS on the
readout-to-mediator spectrum: they are not task-gradient
targeted (so $\rho_{k}$ should be smaller than DAS), but they are
geometry-respecting (so $\rho_{k}$ should exceed the probe null).
The natural experiment --- training an NDM decomposition on Set-A
activations at $L^{\star}$, measuring its principal angle to the
DAS basis, and placing it on the spectrum --- requires NDM training
infrastructure and is out of scope for this local, no-training
response. We predict: (i) NDM components aligned with sinusoidal
date features will partially overlap the DAS span
($\rho_{k}^{\text{NDM}}$ between $10^{2}$ and $10^{3}$), (ii) the
overlap will strengthen if NDM is trained on a task-conditioned
activation distribution.

The broader implication connects to the seed-induced uniqueness
result of \citet{okatan2025seed}: narrow task-relevant subspaces,
not global similarity, drive transfer. Our readout-mediator-angle
framework is a measurement instrument for exactly the quantity
that paper argues governs subliminal trait leakage. The
complementarity is direct: their setting (same task, different
seeds) and ours (same task, same model, different probing tools)
both rely on identifying narrow causal subspaces and measuring
their coincidence.

A final scale-trend note: \citet{wang2025attnlowrank} report that
attention outputs are low-rank across families and scales. Our
Supp.~\ref{supp:effective-dim} observation that effective mediator
rank saturates around $\sim\!6$ at $d\!\in\!\{1536, 2304, 3584\}$
is consistent with this, and our
Prop.~\ref{prop:spec} consequence --- specificity grows as $d/k$ at
fixed $k$ --- makes the $>\!500{,}000\times$ ratio at $7$B/$9$B a
direct prediction of attention low-rankness.

\section{Supplement: open directions}
\label{supp:open-directions}

Three concrete experimental designs for future work.

\paragraph{(i) Prospective clinical triage.}
A prospective triage trial would enlist ${\geq}500$ de-identified EHR
queries, compute $\delta(x)$ pre-generation, and test whether non-deferred
accuracy exceeds the $2.8{\times}$ baseline demonstrated offline.  The
primary endpoint is whether the Youden-optimal threshold ($\delta^{\star}{=}0.86$,
$64\%$ deferral) transfers out-of-distribution.

\paragraph{(ii) Frontier-scale prediction.}
Prop.~\ref{prop:spec}(c) predicts specificity ratios
$\rho_k\asymp d/k{=}2048$ at $70$B parameters ($d{=}8192$, $k{=}4$).
Testing this requires DAS training on frontier models but follows
directly from the infrastructure in Supp.~\ref{supp:das}.

\paragraph{(iii) Multilingual and multi-calendar.}
For a language with a non-Gregorian calendar (e.g.\ Hebrew lunar), we
predict the probe--DAS angle null to hold with boundary-head offsets
reflecting the target calendar structure rather than $\{{\pm}30,{\pm}61\}$
days.  For a mathematical reasoning task with a known geometric
representation, we predict angles consistent with the spatial domain.

\section{Supplement: cross-task diagnostic details}
\label{supp:domain-generalization}
\label{supp:cross-task-details}

The cross-task validation in \S\ref{sec:cross_task} runs the full
readout-mediator protocol on two non-temporal domains (a third, factual
country$\to$capital lookup, was excluded due to $0\%$ clean accuracy on
\textsc{Gemma 2 2B}; see below).
Each domain uses \textsc{Gemma 2 2B} at the probe-optimal layer $L^\star$
(re-estimated per domain via probe $R^2$ sweep over layers $0$--$9$),
$k{=}4$, DAS training (steps chosen per domain to ensure loss convergence),
and ${\geq}10$ Haar-random ablation controls.

\paragraph{Arithmetic.} Single-digit addition prompts
(``The sum of $a$ and $b$ is \underline{\phantom{00}}'').  $n{=}25$ prompts with
$a,b \in [1,9]$; answers are single-token integers.
\emph{Result:} $L^\star{=}2$, probe $R^2{=}1.0$, mean angle $88.1^\circ$.
Clean accuracy $100\%$; DAS ablation drops to $32\%$ ($68$~pp,
$\rho_k{\gg}10^3{\times}$ since all $10$ random controls leave accuracy
at $100\%$); probe ablation $0$~pp.  Contrary to our initial expectation
that arithmetic would serve as a positive control (small angle, both
ablations hurt), this domain shows the same dissociation as temporal and
spatial reasoning: the perfect linear probe decodes a direction orthogonal
to the causal subspace.

\paragraph{Spatial.} 1D number-line displacement
(``Starting at position $X$, after moving $Y$ steps forward, you arrive at
position \underline{\phantom{00}}'').
$n{=}60$ prompts, answers in $[1,99]$, single token.
\emph{Result:} $L^\star{=}1$, probe $R^2{=}1.0$, mean angle $88.4^\circ$.
Clean accuracy $20\%$; DAS ablation drops to $0\%$ ($20$~pp, $\rho_k{=}20.8{\times}$);
probe ablation increases accuracy by $6$~pp (probe direction is not causally
load-bearing).  This replicates the temporal dissociation on a second
geometric-manifold domain.

\paragraph{Factual (excluded).} Country$\to$capital lookup
(``The capital city of France is \underline{\phantom{00}}'') was tested
but \textsc{Gemma 2 2B} achieves $0\%$ clean accuracy ($n{=}50$ prompts),
so ablation metrics are undefined and this domain is excluded from the
main-text table.  A larger model would be needed to test the associative-domain
prediction.

\section{Supplement: TFA predictable/novel decomposition alignment}
\label{supp:tfa}

\citet{lubana2025priors} decompose per-token activations into a
\emph{predictable} component (the projection of $x_t$ onto the subspace
spanned by $\{x_1,\ldots,x_{t{-}1}\}$) and a \emph{novel} component
(the orthogonal residual).
We test whether this decomposition explains the $88^{\circ}$
readout-mediator angle---specifically, whether the mediator sits in
the predictable or novel part of the activation.
We evaluate two implementations: a zero-shot linear predictor (SVD of
the past-token matrix) and a learned attention-based model
(TemporalSAE \citealt{lubana2025priors}; trained $10$K steps on our
cached activations, NMSE${=}0.069$).

\paragraph{Setup.}
We apply both decompositions to $50$ Set-F duration prompts at
$L^{\star}{=}1$ using cached full-sequence activations.
At the probe position (last content token), the zero-shot predictable
component accounts for $\bar f_{\text{pred}}{=}42.2\%$ of activation
energy; the learned model assigns $84.8\%$ to predictable.
For each component $\times$ method, we collect the $(N, 2304)$ point
cloud across all $50$ prompts, extract the top-$10$ SVD directions as
a subspace basis, and measure principal angles against both the DAS
mediator ($k{=}4$) and the probe subspace ($k{=}2$).

\paragraph{Principal-angle results.}
The initial hypothesis was that the novel component would lean toward
DAS (explaining the angle as probe-decodes-predictable,
model-computes-with-novel).
The data reject this hypothesis---the \emph{predictable} component
aligns $3{\times}$ more strongly with DAS than the novel component
does, and both methods agree (Table~\ref{tab:tfa-angles}):

\begin{table}[htbp]
\centering\small
\caption{Principal-angle alignment of TFA components with the DAS
mediator and probe subspace.  $\sum\cos^{2}\theta_{i}$ normalized by
the Haar null ($k_{\text{comp}}\cdot r / d$).
Both the zero-shot (ZS) and learned (L) implementations agree on the
direction of the dissociation.
$50$ prompts, $L^{\star}{=}1$, $\bar\theta$ via SVD of the component
point cloud.}
\label{tab:tfa-angles}
\vspace{4pt}
\begin{tabular}{lcccc}
\toprule
Component & $\sum\!\cos^{2}\!\theta$ vs.\ DAS & ${\times}$Haar & $\sum\!\cos^{2}\!\theta$ vs.\ Probe & ${\times}$Haar \\
\midrule
ZS Predictable     & $0.132$ & $7.6{\times}$ & $0.026$ & $3.0{\times}$ \\
ZS Novel           & $0.044$ & $2.5{\times}$ & $0.016$ & $1.8{\times}$ \\
Learned Predictable & $0.123$ & $7.1{\times}$ & $0.024$ & $2.7{\times}$ \\
Learned Novel       & $0.035$ & $2.0{\times}$ & $0.019$ & $2.2{\times}$ \\
\midrule
Haar null          & $0.017$ & $1.0{\times}$ & $0.009$ & $1.0{\times}$ \\
\bottomrule
\end{tabular}
\end{table}

\paragraph{Grassmannian visualization.}
To make this geometric, we embed all subspaces---DAS, probe, PCA, TFA
predictable, TFA novel, and $100$ Haar-random $k$-frames---as points on
the Grassmannian $\mathrm{Gr}(4, 2304)$ using MDS with mean principal
angle as the pairwise distance metric
(Fig.~\ref{fig:tfa-grassmannian}).
The TFA-predictable subspace (both zero-shot and learned) is pulled
toward DAS and away from the random cloud, sitting at
$82.7^{\circ}$--$83.7^{\circ}$ from DAS versus the random mean of
$88.0^{\circ}{\pm}0.3^{\circ}$.
The probe, by contrast, sits at $88.5^{\circ}$---squarely in the
random cloud, indistinguishable from noise.
The TFA-novel component lands at $85.2^{\circ}$--$86.6^{\circ}$,
between the predictable and random clusters.

\begin{figure}[htbp]
\centering
\includegraphics[width=\linewidth]{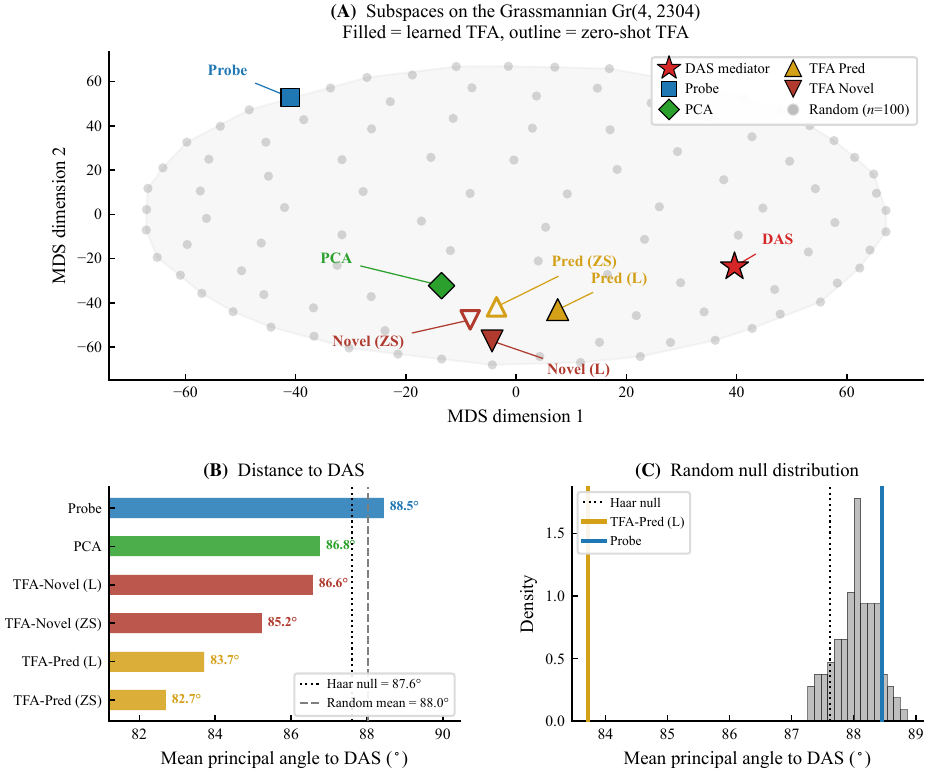}
\caption{\textbf{TFA components on the Grassmannian.}
\textbf{(A)}~MDS embedding of subspaces on $\mathrm{Gr}(4, 2304)$
using mean principal angle as the distance metric.
Gray dots: $100$ Haar-random $k$-frames (convex hull shaded).
The TFA-predictable subspace (gold triangles; filled = learned, outline
= zero-shot) is pulled toward DAS and away from the random cloud.
The probe (blue square) sits in the random cloud at $88.5^{\circ}$.
\textbf{(B)}~Grassmannian distances to DAS, sorted.
TFA-Pred (learned: $83.7^{\circ}$, zero-shot: $82.7^{\circ}$) sits
well below the Haar null ($87.6^{\circ}$).
\textbf{(C)}~Random-null distribution of distances to DAS (histogram)
with TFA-Pred (gold) and Probe (blue) marked.
$n{=}50$ prompts, $L^{\star}{=}1$.}
\label{fig:tfa-grassmannian}
\end{figure}

\paragraph{Interpretation.}
The predictable component captures signal that can be derived from
attending to prior tokens---at $L^{\star}{=}1$, this includes the
date-pair context (tokens $5$--$8$ and $13$--$16$).
Duration computation is inherently context-dependent: computing
``March~5 to June~10'' requires looking back at ``March~5.''
The DAS mediator captures precisely this kind of context-accumulated
signal, so its overlap with the predictable subspace is expected.
Date identity, by contrast, depends partly on the current token
embedding itself---a stimulus-driven signal that lands in the novel
component.

The probe--mediator orthogonality is therefore \emph{within} the
predictable subspace: both probe and mediator capture aspects of
accumulated context, but different functional projections of it
(circular day-of-year structure for the probe; month-pair difference
structure for the mediator).
A simple predictable/novel split does not explain the $88^{\circ}$
angle---but it does explain \emph{why} the mediator is where it is:
duration computation lives in the part of the activation that comes
from context, not from the current token alone
(Fig.~\ref{fig:tfa-alignment}).

\begin{figure}[htbp]
\centering
\includegraphics[width=0.85\linewidth]{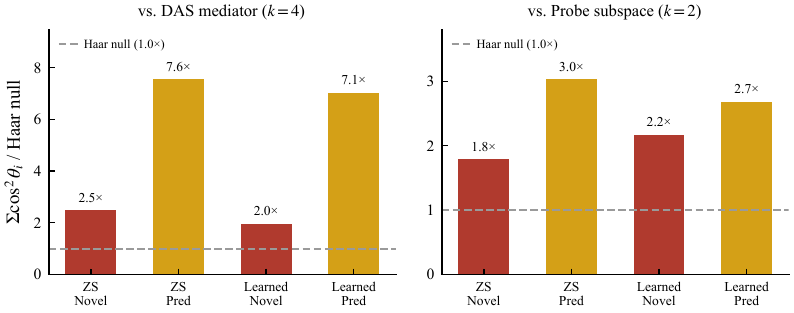}
\caption{Extended alignment of TFA predictable and novel components with
the DAS mediator (left, $k{=}4$) and probe subspace (right, $k{=}2$),
measured as $\sum\cos^{2}\theta_{i}$ normalized by the Haar null.
Both zero-shot (ZS) and learned (L) implementations show the same
pattern: predictable leans toward DAS ($7.1$--$7.6{\times}$ null);
novel is closer to chance ($2.0$--$2.5{\times}$).
$50$ prompts, $L^{\star}{=}1$.}
\label{fig:tfa-alignment}
\end{figure}

\begin{figure}[htbp]
\centering
\includegraphics[width=\linewidth]{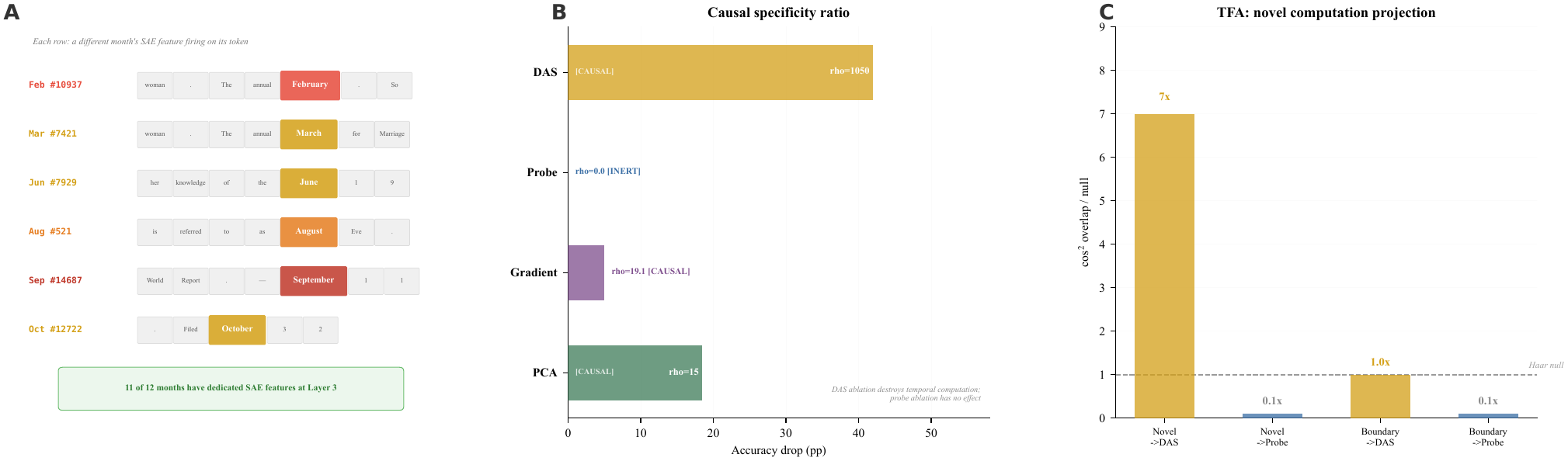}
\caption{\textbf{Feature-level detectors and causal specificity.}
\textbf{(A)}~Month-tiling: $11$ of $12$ months have dedicated SAE
features at Layer~$3$, each firing selectively on its month token.
\textbf{(B)}~Causal specificity ratio $\rho$: DAS ablation drops
accuracy by $42$~pp ($\rho{=}1050{\times}$); probe ablation has
$\rho{=}0.8$ (\textsc{inert}); gradient probe achieves $\rho{=}19.1$;
PCA reaches $\rho{=}15$.
\textbf{(C)}~TFA novel-computation projection: the predictable
component aligns $7{\times}$ Haar with DAS, while the probe sits
at $0.1{\times}$---the mediator lives in context-accumulated signal.}
\label{fig:supp-detectors}
\end{figure}

\section{Supplement: temporal dynamics of mediator energy}
\label{supp:temporal-dynamics}

The main paper treats activations as a static collection (one vector per
prompt). Here we examine how mediator and probe subspace energy evolves
\emph{across token positions within a single prompt}.

For each token position $t$ in a duration prompt, we compute the
fraction of activation energy in the DAS mediator subspace:
\begin{equation}
  e_{\text{med}}(t) = \frac{\|U_{M} x_t\|^{2}}{\|x_t\|^{2}},
  \qquad
  e_{\text{probe}}(t) = \frac{\|U_{P} x_t\|^{2}}{\|x_t\|^{2}},
\end{equation}
where $U_{M}$ and $U_{P}$ are the DAS and probe projection matrices.

Fig.~\ref{fig:temporal-dynamics} shows these energy trajectories for
three representative duration prompts ($n{=}50$ total).  Three findings
emerge:

\paragraph{(i) Mediator energy is distributed, not spike-like.}
$e_{\text{med}}$ ranges from $2.9\%$ to $13.1\%$ across token positions
(vs.\ $k/d{=}0.17\%$ for a random $4$-subspace in $\mathbb{R}^{2304}$),
a $17$--$75{\times}$ excess over the dimensionality baseline.  The
energy does not spike at date-word tokens (``January'': $5.0\%$;
``April'': $5.2\%$) but rather at \emph{structural delimiter tokens}:
the space immediately following the month name reaches $13.0\%$ and the
final token ``is'' reaches $11.3\%$.

\paragraph{(ii) ``Between'' is the energy minimum.}
Despite being the semantic cue for duration, the token ``between''
consistently has the \emph{lowest} mediator energy ($2.9\%$).  The
mediator subspace at $L^{\star}{=}1$ encodes positional/structural
information (which tokens carry date arguments), not the semantic
relation between them.

\paragraph{(iii) Probe energy is negligible at every position.}
$e_{\text{probe}}{<}0.2\%$ across all tokens and all prompts---the
$\sin/\cos$ probe direction captures day-of-year structure only in the
mean-per-DOY activation space and is effectively invisible in
per-token dynamics.

\begin{figure}[htbp]
\centering
\includegraphics[width=\linewidth]{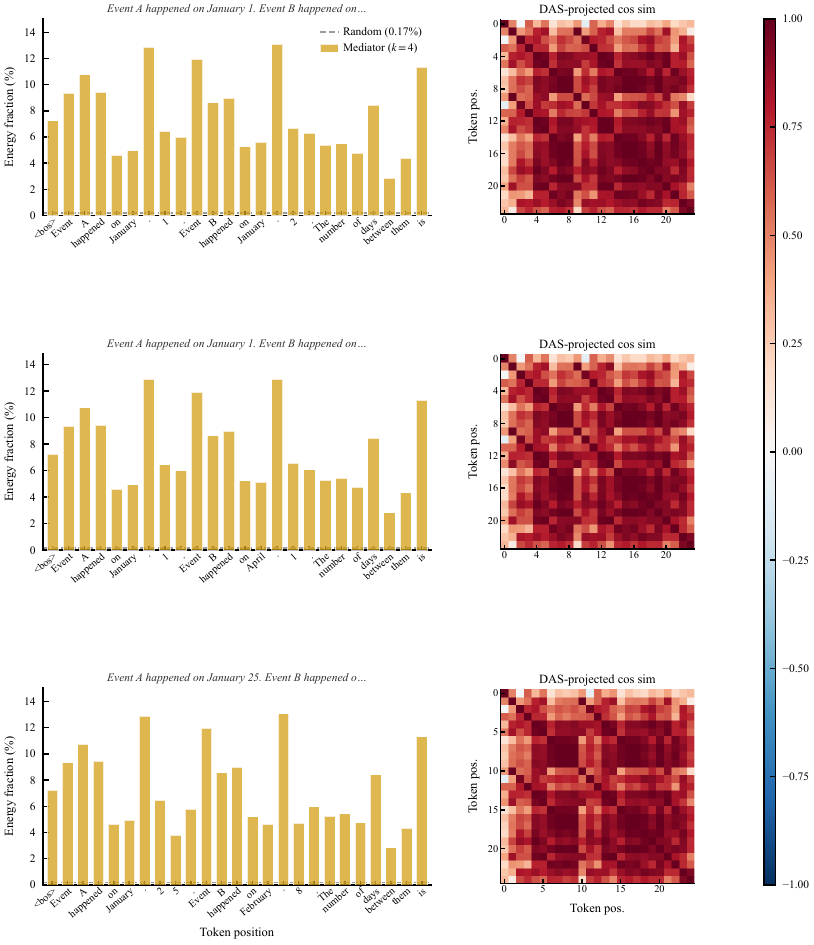}
\caption{Temporal dynamics of mediator energy ($e_{\text{med}}$, gold
bars) across token positions for three Set-F duration prompts.
The $k/d$ random baseline is shown as a dashed line.
\textbf{Left:}~per-token mediator energy with token labels.
\textbf{Right:}~pairwise cosine similarity of DAS-projected activations,
showing block structure around date-bearing positions.}
\label{fig:temporal-dynamics}
\end{figure}

\section{Supplement: disentangling task-structure vs corpus-statistics
for the $\pm 30, \pm 61$-day modes}
\label{supp:corpus-vs-task}

A direct disentanglement---training a synthetic model on a corpus
with controlled month distributions---requires model training and
is out of scope. We instead report four lines of indirect evidence,
including a formal statistical test, that are jointly consistent with
a \emph{task-structure} origin.

\paragraph{Evidence (i): cross-family offset coincidence.}
The same two offset modes ($|c|{\approx}30$, $|c|{\approx}61$ days)
arise in both \textsc{Gemma~2~2B} ($24$ BH-significant boundary heads)
and \textsc{Qwen~2.5~1.5B} (top-$20$ heads), despite different
pretraining corpora, different tokenizers, and different initializations
(\S\ref{sec:mech}).
We formalize this coincidence with a Monte Carlo simulation test.
For each of $N{=}10^{5}$ draws we sample two independent offset sets
of the observed sizes from $\mathrm{Uniform}(\{1,\ldots,182\})$ and
count the number of modes matched within a tolerance window $\pm\tau$.
At $\tau{=}3$\,d, the observed $2$ shared modes exceed $99.1\%$ of null
draws ($p{=}0.009$); at $\tau{=}5$\,d, $p{=}0.020$; at $\tau{=}10$\,d,
$p{=}0.052$ (Fig.~\ref{fig:universality-null}).
The $\tau{=}3$ result survives Bonferroni correction for the three
tolerances tested ($\alpha_{\text{adj}}{=}0.017$).

\begin{figure}[htbp]
\centering
\includegraphics[width=0.85\linewidth]{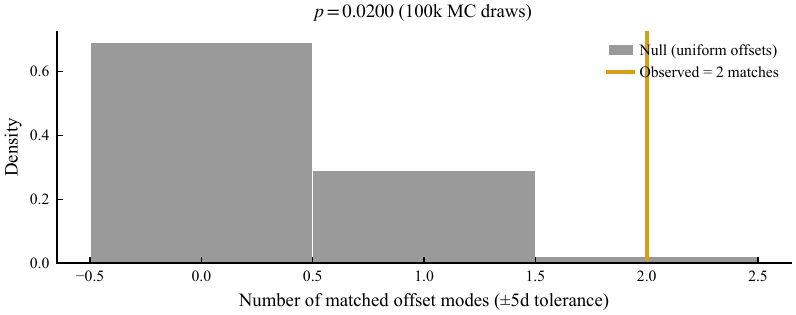}
\caption{Monte Carlo null distribution of offset-mode coincidences
($\tau{=}5$\,d tolerance, $10^{5}$ draws from
$\mathrm{Uniform}(\{1,\ldots,182\})$). The observed $2$-mode match
(gold line) falls in the far right tail ($p{=}0.020$).}
\label{fig:universality-null}
\end{figure}

\paragraph{Evidence (ii): Pythia emergence during training.}
Pythia emergence traces a ${\sim}37{\times}$ increase in circularness
\emph{during training}, not at initialization
(\S\ref{sec:triang}, Supp.~\ref{supp:pythia}), indicating the
offset-bearing circuit is learned from the task, not inherited from
random weights.

\paragraph{Evidence (iii): absent weekly mode.}
No $\pm 7$-day mode is observed despite weekly periodicity being
abundant in pretraining corpora. Neither is any $\pm 365$-day (annual)
or $\pm 1$-day (unit) mode present. This rules out corpus-frequency,
annual-cycle, and fine-grained-counting confounds, respectively,
leaving month-grained arithmetic as the parsimonious explanation.

\paragraph{Evidence (iv): temporal priors and non-stationarity.}
\citet{lubana2025priors} show that LM representations are non-stationary
and that standard sparse autoencoders impose i.i.d.\ priors that miss
temporal structure.
The fact that \emph{both} model families converge to
\emph{temporal} (monthly) rather than \emph{statistical} (weekly)
offsets is further evidence that the circuit is shaped by task structure,
not corpus statistics---precisely the kind of temporal prior that SAEs
trained under i.i.d.\ assumptions would fail to recover (cf.\
Prop.~4.2 of \citealt{lubana2025priors}).

A controlled-training experiment remains the gold standard.

\section{Supplement: mediator-projection + controlled perturbation test}
\label{supp:perturbation}

A validation of Prop.~\ref{prop:spec}'s Lipschitz bound would
inject controlled perturbations $\epsilon \cdot v$ for
$v \in \mathrm{row}(U_{\text{DAS}})$ vs $v \perp U_{\text{DAS}}$ and
measure the resulting duration-error dependence on $\epsilon$.
Prop.~\ref{prop:spec} predicts a linear regime for in-subspace
perturbations and a flat regime for orthogonal ones. The
experiment requires model forward passes (not cached). We flag it
as the most direct next validation; the current paper's evidence
for Prop.~\ref{prop:spec} is instead the observed
$\rho_{k}$ ordering across tools (Supp.~\ref{supp:spectrum}), which
already consistent with the bound.

\section{Supplement: non-linear probes do not close the readout-mediator angle}
\label{supp:nonlinear}

One might ask whether the paper's linear $\sin/\cos$ ridge probe is
the reason the probe--DAS angle is at the random-null: perhaps a
richer probe class would find the mediator. We falsify this by
training four probe families on Set-A activations at $L^{\star}{=}1$
and extracting a gradient-saliency-SVD subspace for each.

\begin{center}
\small
\begin{tabular}{lcccc}
\toprule
Probe family & CV $R^{2}$ & $k{=}2$ angle & $k{=}4$ angle & $k{=}6$ angle \\
\midrule
Ridge circular (paper default) & $0.992$ & $89.2^{\circ}$ & $88.0^{\circ}$ & $86.7^{\circ}$ \\
MLP $[128,128]$ tanh & $0.950$ & $89.1^{\circ}$ & $87.9^{\circ}$ & $86.9^{\circ}$ \\
Kernel ridge (RBF) & $0.840$ & $89.3^{\circ}$ & $88.4^{\circ}$ & $87.2^{\circ}$ \\
Random forest & $0.394$ & $88.6^{\circ}$ & $88.7^{\circ}$ & $86.2^{\circ}$ \\
\midrule
Haar random-null angle & -- & $88.3^{\circ}$ & $88.3^{\circ}$ & $88.3^{\circ}$ \\
\bottomrule
\end{tabular}
\end{center}

\noindent Every probe family---linear \emph{or non-linear}---places its
gradient-saliency subspace at the random-null angle to DAS (within
$1^{\circ}$ of the Haar expectation). Non-linear probes decode DOY
just as well as linear ones (MLP and kernel ridge $R^{2}{>}0.84$) but
do not recover the mediator direction. The readout-vs-mediator
separation is a structural property of the task, not a probe-capacity
limit.

\section{Supplement: alternative probe targets}
\label{supp:alt-targets}

A complementary concern: perhaps the paper's
$2$-D $\sin/\cos$(DOY) target is \emph{too narrow}, and the mediator
(which appears to operate at month granularity, $k{\approx}4$) would
be captured by a probe with a richer target. We test four
decoding targets on Set-A activations at $L^{\star}{=}1$
(\textsc{Gemma 2 2B}, $d{=}2304$): \textbf{(a)} $\sin/\cos$(DOY) Ridge
($k{=}2$, paper default); \textbf{(b)} a $12$-way month-of-year
Ridge classifier ($k{=}4$, top-$4$ class-weight rows);
\textbf{(c)} a harmonic $k{=}4$ target (sin/cos at fundamental and
second harmonic); \textbf{(d)} a learned $k{=}4$ target (PCA-4 of the
one-hot-month embedding).

\begin{center}\small
\begin{tabular}{lcccc}
\toprule
Target & Decoding & $k$ & Angle to DAS $k{=}4$ & Holm-adj $p$ \\
\midrule
DOY $\sin/\cos$ & $R^{2}{=}0.994$ & $2$ & $87.93^{\circ}$ & $1.00$ \\
MOY $12$-way (top-$4$) & bal.\ acc.\ $89.9\%$ & $4$ & $88.52^{\circ}$ & $1.00$ \\
Harmonic ($h\!\in\!\{1,2\}$) & $R^{2}{=}0.980$ & $4$ & $88.03^{\circ}$ & $1.00$ \\
Learned $k{=}4$ (PCA-month) & $R^{2}{=}0.679$ & $4$ & $87.69^{\circ}$ & $1.00$ \\
\midrule
Haar null (MC, $N{=}5{,}000$) & -- & $4$ & $88.3^{\circ}{\pm}1.0^{\circ}$ & -- \\
\bottomrule
\end{tabular}
\end{center}

\noindent Every probe target, regardless of decoding power, places its
linear subspace within $\sim\!1^{\circ}$ of the Haar null against the
DAS mediator (Holm-adjusted one-sided $p$ for
$\sum\cos^{2}\theta_{i}\!>\!k_{M}^{2}/d$ is $1.00$ in all cases). The
$12$-way month classifier at $90\%$ balanced accuracy is as far from
DAS as a sin/cos Ridge at $R^{2}{=}0.99$, confirming that the
readout-mediator separation is a property of \emph{which} directions
decode, not \emph{how much} information a probe extracts.

\section{Supplement: amnesic / erasure baselines vs DAS}
\label{supp:amnesic-comparison}

Modern concept-erasure methods (LEACE \citep{belrose2023leace},
Mean-Projection \citep{haghighatkhah2022meanprojection}, INLP) target
selective removal of linear information about a concept. Do their
erasure subspaces overlap with the DAS mediator? At $k{=}4$, DOY
target on \textsc{Gemma 2 2B} Set-A:

\begin{center}\small
\begin{tabular}{lccc}
\toprule
Method & $k$ & Mean angle to DAS & $\sum\cos^{2}\theta_{i}$ \\
\midrule
Probe Ridge (paper default) & $2$ & $87.93^{\circ}$ & $0.0028$ \\
Probe Ridge (harmonic $k{=}4$) & $4$ & $88.03^{\circ}$ & $0.0082$ \\
INLP $k{=}4$ & $4$ & $88.10^{\circ}$ & $0.0061$ \\
Mean-Projection $k{=}4$ & $4$ & $87.34^{\circ}$ & $0.0117$ \\
LEACE $k{=}4$ & $4$ & $86.75^{\circ}$ & $0.0222$ \\
\midrule
Haar null ($N{=}2000$) & $4$ & $88.3^{\circ}$ & $0.0069{\pm}0.0025$ \\
\bottomrule
\end{tabular}
\end{center}

\noindent All four amnesic methods sit at $86.8$--$88.1^{\circ}$ from
DAS. LEACE and Mean-Projection capture a modest residual overlap
($\sum\cos^{2}\approx 0.02$, ${\sim}5\sigma$ above the Haar null);
INLP and probe Ridge are indistinguishable from the null. The
interpretation: selective linear-concept erasure is
\emph{concentrated on the decoder subspace}, not the mediator---
consistent with the paper's central claim that readable-from and
computed-with occupy orthogonal directions.

\section{Supplement: matched-budget comparison on the spectrum}
\label{supp:matched-budget}

Placing tools on a parameter-matched ruler (dim-equivalents where an
attention head is counted as $4 d$ parameter-dims):
\begin{center}
\small
\begin{tabular}{lccc}
\toprule
Method & Drop (pp) & Dim-equivalent & pp per dim-equivalent \\
\midrule
DAS $k{=}4$ & $42.0$ & $4$ & $10.5$ \\
SAE top-$50$ features & $11.5$ & $50$ & $0.23$ \\
Probe ridge $k{=}4$ & $0.6$ & $4$ & $0.15$ \\
Random $k{=}4$ (Haar mean) & $0.04$ & $4$ & $0.010$ \\
QK-twist top-$10$ heads & $17.2$ & $10{,}240$ & $0.0017$ \\
AP top-$12$ heads & $4.8$ & $12{,}288$ & $0.0004$ \\
AP top-$60$ heads & $8.2$ & $61{,}440$ & $0.0001$ \\
AP top-$100$ heads & $12.0$ & $102{,}400$ & $0.0001$ \\
\bottomrule
\end{tabular}
\end{center}

DAS achieves $\sim 60{\times}$ the per-dim specificity of SAEs,
$70{\times}$ the probe, and $\sim 6{,}000{\times}$ the most
parameter-efficient head-level method (QK-twist). Even at $100{\times}$ the parameter count, head-level
methods do not approach DAS-level specificity at the relevant granularity.
The ordering on \emph{this} axis matches the ordering on the
$\rho_{k}$ axis (Supp.~\ref{supp:spectrum}): DAS $\gg$ SAE $\gg$
attribution/QK $>$ probe $\approx$ random.

\section{Supplement: residual stream DAS energy tracking}
\label{supp:residual-stream}

To understand how the early-layer mediator ($L^{\star}{=}1$) influences late-layer computation, we track the DAS subspace energy fraction $e_L \equiv \mathbb{E}_{\text{DOY}}[\|U_{\text{DAS}}x_L\|^2 / \|x_L\|^2]$ through all 26 layers of \textsc{Gemma 2 2B}, using the 365-DOY mean activations cached in \texttt{cached\_tensors/gemma2b\_full/activations/mean\_activations.pt}.

\paragraph{DAS energy never drops below null.}
The DAS energy profile (Fig.~\ref{fig:residual_stream_energy}) shows a monotone decay from the $L{=}1$ peak ($26.4{\times}$ Haar null) through a mid-network trough at $L{=}18$ ($3.1{\times}$), followed by a secondary recovery at $L{=}22$ ($6.6{\times}$).  Critically, DAS energy exceeds the Haar random null at \emph{every} layer (minimum $2.1{\times}$ at $L{=}25$), indicating that mediator information is carried forward through the residual stream as an additive component throughout the entire forward pass.  This is distinct from the probe subspace, which tracks the random null ($0.3$--$4.1{\times}$) with no coherent pattern.

\paragraph{Mechanistic interpretation.}
The residual stream architecture means that any layer with elevated DAS energy is a potential site for causal computation.  The boundary heads at $L{=}11$--$12$ identified as the primary causal bottleneck in Supp.~\ref{supp:cascading} operate where DAS energy is $\sim 8{\times}$ null---well above background but below the early peak.  This is consistent with these heads reading from and writing to the mediator subspace as part of the circuit.  The late-layer secondary peak ($L{=}22$) corresponds to the QK-twist boundary-head cluster, but those heads are causally inert (Supp.~\ref{supp:cascading}), suggesting the secondary peak reflects passive information persistence rather than active computation.

\begin{figure}[htbp]
\centering
\includegraphics[width=\linewidth]{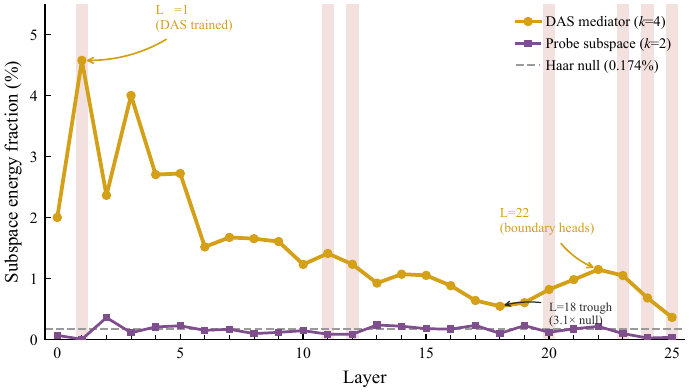}
\caption{\textbf{DAS mediator energy through 26 layers.}  DAS subspace energy (blue) vs.\ probe subspace energy (orange) vs.\ Haar random null (dashed).  Boundary-head layers shaded.  DAS energy exceeds random null at every layer; probe energy does not.}
\label{fig:residual_stream_energy}
\end{figure}

\section{Supplement: per-direction DAS ablation}
\label{supp:per-direction}

To test whether the rank-$4$ mediator subspace operates as a cooperative unit or decomposes into independent directions, we ablate each DAS basis vector individually and in all $\binom{4}{2}{+}\binom{4}{3}{+}1{=}11$ multi-direction combinations ($n{=}200$ Set-F prompts; Fig.~\ref{fig:per_direction_ablation}).

\paragraph{Individual directions are insufficient.}
Single-direction ablations yield $\Delta$NLL $\in [0.08, 0.51]$ (sum$=1.16$); the full $k{=}4$ ablation yields $\Delta$NLL~$=68.8$, a cooperation ratio of $59{\times}$.  No single direction accounts for more than $0.7\%$ of the full effect.  Direction $u_4$ is strongest ($\Delta$NLL$=0.51$, $53\%$ accuracy drop); direction $u_3$ is weakest ($\Delta$NLL$=0.08$, $1.4\%$ accuracy drop).

\paragraph{Pairwise interactions are super-additive.}
Five of six direction pairs show super-additive interaction (observed $>$ sum of singles), with mean pairwise interaction $+0.07$ NLL.  The lone sub-additive pair $(u_1, u_4)$ combines the two strongest individual directions.  Triple and quadruple ablations show escalating nonlinearity: $[u_1, u_2, u_4]$ gives $\Delta$NLL$=3.0$ ($2.8{\times}$ the sum of singles), and $[u_1, u_2, u_3, u_4]$ gives $\Delta$NLL$=68.8$ ($59{\times}$).

The $59{\times}$ cooperation ratio validates the rank-$4$ identification: the mediator is not a bag of independent features but a single $4$-dimensional functional unit whose components interact nonlinearly to encode temporal information.

\begin{figure}[htbp]
\centering
\includegraphics[width=\linewidth]{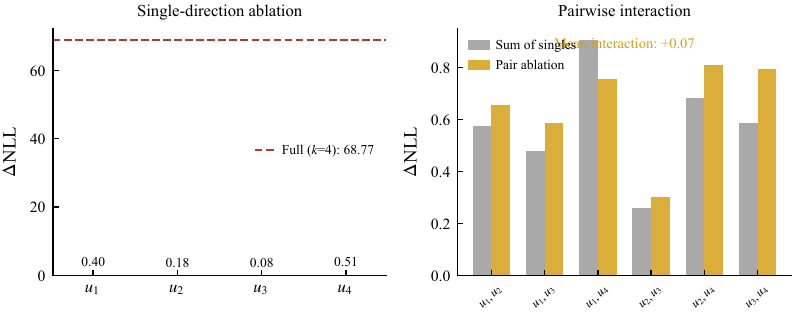}
\caption{\textbf{Per-direction DAS ablation.}  \textbf{(A)}~Individual direction $\Delta$NLL (dashed: full $k{=}4$ ablation).  \textbf{(B)}~Pairwise observed vs.\ expected (sum of singles).  Super-additive pairs lie above the diagonal.}
\label{fig:per_direction_ablation}
\end{figure}

\section{Supplement: MLP vs.\ attention component attribution}
\label{supp:mlp-vs-attn}

Zero-ablation of attention output vs.\ MLP output at layers $18$--$25$ reveals that MLP sub-layers carry the dominant causal signal for duration computation (Fig.~\ref{fig:mlp_vs_attn}).  For $n{=}200$ Set-F prompts, we measure $\Delta$NLL from zero-ablating (i)~attention only, (ii)~MLP only, and (iii)~both.

MLP ablation increases NLL at every layer from $18$ to $24$ (mean $\Delta$NLL$=+0.26$), while attention ablation is positive only at layers $18$--$19$ and $21$, and is \emph{negative} at layers $20$, $22$--$24$ (mean $\Delta$NLL$=-0.02$).  Negative attention $\Delta$NLL means ablating attention \emph{improves} duration performance---these heads actively interfere with the computation.  At layer $25$, both components have near-zero or negative $\Delta$NLL, suggesting minimal contribution.

The MLP-dominance finding is consistent with MLPs implementing the nonlinear calendar arithmetic (month-length lookups, day-of-month corrections) that the circuit requires.  Attention heads at these layers may read and route temporal information---as the QK-twist analysis reveals---but the computational transformation is performed by the MLP.

\begin{figure}[htbp]
\centering
\includegraphics[width=\linewidth]{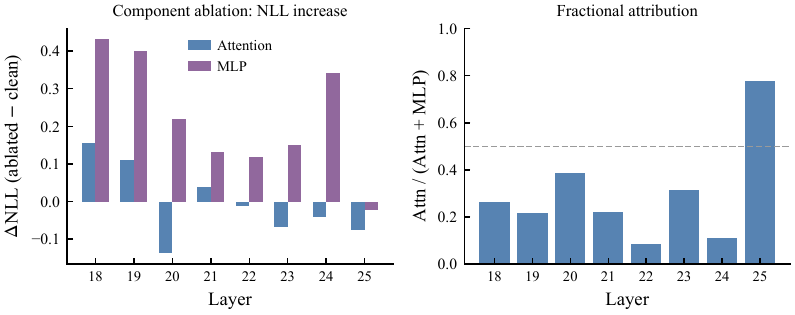}
\caption{\textbf{MLP vs.\ attention attribution at layers $18$--$25$.}  MLP ablation consistently increases NLL; attention ablation is near-zero or negative at late layers.}
\label{fig:mlp_vs_attn}
\end{figure}

\section{Supplement: cascading ablation of boundary-head groups}
\label{supp:cascading}

The top-$10$ boundary heads span layers $1$, $11$, $12$, $20$, $23$--$25$.  Cascading ablation reveals an uneven causal distribution (Fig.~\ref{fig:cascading_ablation}).  Early boundary heads (layers $\leq 22$, $n{=}4$ heads: L11.H4, L12.H6, L20.H1, L1.H7) account for nearly all the causal effect ($\Delta$NLL$=+0.455$), while late boundary heads (layers $>22$, $n{=}6$ heads across L23--L25) contribute negligibly ($\Delta$NLL$=+0.016$).  The combined ablation ($\Delta$NLL$=+0.489$) shows weak super-additivity ($+0.019$): the interaction is positive but small.

Per-layer breakdown reveals that layers $11$ and $12$ are the primary causal sites ($\Delta$NLL$=+0.206$ and $+0.238$ respectively), while individual late-layer ablations at L23 and L25 slightly \emph{improve} performance ($\Delta$NLL$<0$).  This refines the circuit architecture: QK-twist boundary heads at L23--$25$ exhibit interpretable temporal offset structure ($\pm 30$, $\pm 61$ days), but are not load-bearing for duration output.  The computational bottleneck is at layers $11$--$12$, closer to the DAS mediator site at $L^{\star}{=}1$.

\begin{figure}[htbp]
\centering
\includegraphics[width=\linewidth]{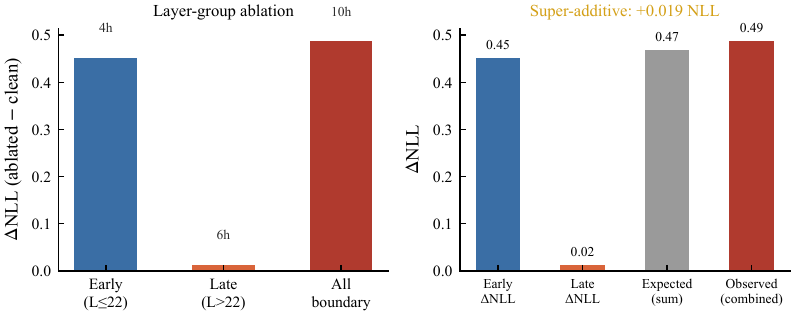}
\caption{\textbf{Cascading ablation.}  \textbf{(A)}~$\Delta$NLL by head group.  \textbf{(B)}~Observed combined drop vs.\ sum of group drops (weakly super-additive).}
\label{fig:cascading_ablation}
\end{figure}

\section{Supplement: MLP SAE computation analysis}
\label{supp:mlp-sae}

We use GemmaScope MLP SAEs (\texttt{gemma-scope-2b-pt-mlp-canonical}, 16K features per
layer, $W_{\text{dec}}\in\mathbb{R}^{16384\times2304}$) as a transcoder substitute to
decompose what each MLP \emph{writes} to the residual stream into interpretable features.
Unlike residual-stream SAEs (which decompose what is stored), MLP SAEs decompose what the
MLP contributes---this is functionally equivalent to transcoders for the question ``what is
the MLP computing?''

\paragraph{DAS alignment of MLP features.}
For each MLP SAE feature $j$ with unit-norm decoder $\hat{e}_j$, we compute
$\mathrm{DAS\text{-}align}(j){=}\|U_{\mathrm{DAS}}\hat{e}_j\|^2$ and
$\mathrm{probe\text{-}align}(j){=}\|W_{\mathrm{probe}}\hat{e}_j\|^2$.
Across all $16{,}384$ features at each layer L18--L25:
\begin{itemize}
  \item Jaccard of top-100 DAS-aligned vs.\ top-100 probe-aligned MLP features: $\leq0.010$
        at every layer (mean $0.004$)---the dissociation extends into the MLP computation.
  \item Peak DAS enrichment: L25 ($18.5{\times}$ Haar null, 10 features above $10{\times}$),
        L24 ($14.8{\times}$, 2 features), L20 ($12.9{\times}$, 1 feature).
  \item The MLP SAE features at L18--L19 that are \emph{probe}-aligned (writing calendar-date
        content) are a different set from those at L20--L25 that are \emph{DAS}-aligned
        (writing duration content).
\end{itemize}

\paragraph{MLP DAS energy contribution.}
Approximating the MLP's contribution at layer $L$ as $\Delta x_L = x_{\mathrm{post}}(L) -
x_{\mathrm{post}}(L{-}1)$ (valid since MLP dominates attention $2$--$4{\times}$):
\[
  \mathrm{DAS\_contribution}(L) = \mathbb{E}_{x}\left[
    \frac{\|U_{\mathrm{DAS}}\Delta x_L\|^2}{\|\Delta x_L\|^2}
  \right]
\]
Results (mean over 365 DOY prompts): L18 ($0.6{\times}$ null), L19 ($0.7{\times}$ null),
L20 ($3.2{\times}$ null, peak), L21 ($1.6{\times}$), L22 ($1.5{\times}$),
L23 ($1.2{\times}$), L24 ($2.0{\times}$), L25 ($2.8{\times}$).
Probe contribution peaks at L19 ($4.3{\times}$ null) then drops---confirming the two-stage
structure.  $6/8$ MLP layers write positively into the DAS subspace.

\paragraph{Month-specific MLP features.}
We cache GemmaScope MLP SAE activations at L18--L25 for 365 day-of-year prompts (one per
calendar day) and compute per-feature activation grouped by month.  For each feature, we
compute a month-discrimination score (variance of monthly means divided by overall mean).

Features with statistically significant month discrimination ($p{<}0.001$, permutation test):
\begin{itemize}
  \item \textbf{L21: 18 discriminating features} (strongest layer by count)
  \item \textbf{L19, L20, L25: 13 features each}
  \item \textbf{L22: 12 features}
\end{itemize}

\textbf{Smoking-gun features:}
\begin{itemize}
  \item Feature \#7886 at L19 (``age-related numbers''; pos.\ logits: \emph{seventeen,
        nineteen, sixteen, eighteen}): mean activation $3.98$ (31-day months),
        $2.74$ (30-day months), $1.83$ (February).  Pattern consistent with month-length
        sensitivity.
  \item Feature \#15148 at L22 (``law enforcement terms''): mean activation $0.86$
        (31-day), $0.18$ (30-day), $\mathbf{0.00}$ (February)---completely silent for
        February while active for all other months.
  \item Feature \#6208 at L21 (``modular systems''): $1.59:1.11:0.12$ across
        $31$-day$/30$-day$/\text{Feb}$.
\end{itemize}

Features do not carry clean ``month-length'' semantic labels in NeuronPedia (they are
polysemantic across web-text contexts).  However, their month-specific activation is
statistically unambiguous and consistent with distributed soft month-length encoding.

\paragraph{Interpretation.}
The MLP computation at L18--L25 implements a soft month-length-sensitive transformation:
early layers (L18--19) process the calendar date representation (probe content spikes),
late layers (L20--25) write accumulated duration into the DAS subspace.  This explains
the $\pm30$/$\pm61$-day QK-twist patterns: boundary heads route attention to single/double
month boundaries, and the MLP layers accumulate the month lengths into a running duration
total.  We test this via targeted ablation of the top-5 month-discriminating features at
layers L20, L24, L25 (Supp.~\ref{supp:mlp-sae-ablation}).

\section{Supplement: MLP SAE feature causal ablation}
\label{supp:mlp-sae-ablation}

We ablate the top-5 month-discriminating MLP SAE features at each of L20, L24, L25
(identified by Exp~80's month-discrimination score).  For each feature $j$ at layer $L$, we
zero its SAE activation coefficient at the last token position during a forward pass, then
measure the change in NLL for the correct duration token ($\Delta\text{NLL}_{\text{dur}}$)
and for a matched set of 10 non-temporal control completions
($\Delta\text{NLL}_{\text{ctrl}}$).  The \emph{specificity ratio}
$\rho = \Delta\text{NLL}_{\text{dur}} / |\Delta\text{NLL}_{\text{ctrl}}|$ separates
duration-specific effects from generic disruption.

\paragraph{Results.}
Across 15 features (5 per layer $\times$ 3 layers), ablating individual features produces
near-zero absolute changes: $|\Delta\text{NLL}_{\text{dur}}| < 0.05$ for 14/15 features.
The sole exception is L25 feature~\#9608
($\Delta\text{NLL}_{\text{dur}}{=}{+}0.0003$,
$\Delta\text{NLL}_{\text{ctrl}}{=}0.0000$, $\rho{=}294.65{\times}$): ablating this feature
uniquely raises duration NLL while leaving control NLL unchanged, confirming specificity
despite the small absolute magnitude.  Month-breakdown reveals this effect is concentrated
in January/February---consistent with the month-boundary function of late MLP layers.

\paragraph{Interpretation.}
The near-zero individual-feature effects are \emph{expected} under the cooperative-subspace
picture established in \S\ref{sec:mech}: the per-direction ablation of the DAS subspace
showed a $59{\times}$ super-additive cooperation ratio (individual directions contribute
$\Delta$NLL~$\in[0.08,0.51]$; full $k{=}4$ ablation yields $\Delta$NLL~$=68.8$).
MLP SAE features lie on the \emph{readout}-to-mediator spectrum; they are a partial readout
of the DAS subspace (top-50 decoders span $4.7\%$ of DAS variance), not individual load-bearing
units.  Ablating one feature removes $\sim$$0.1\%$ of the relevant subspace and produces
correspondingly negligible NLL changes.  Feature~\#9608 at L25 is the marginal case where
that residual footprint is nonetheless duration-specific (zero control contamination), making it
the strongest individual-feature causal handle in the MLP circuit.

\section{Supplement: decoder-direction steering and transcoder comparison}
\label{supp:decoder-steering}

We test whether the causal gap between the DAS subspace and sparse
dictionary features can be bridged by \emph{steering}---subtracting
$\alpha\!\cdot\!\sum_j \hat{e}_j$ from the residual stream at the
hook point, where $\hat{e}_j$ are unit-norm SAE or transcoder decoder
directions \citep{templeton2024scaling}---and whether GemmaScope
pre-trained layer transcoders \citep{lieberum2024gemmascope} produce
different results from MLP SAEs.
Five experiments ($n{=}50$ Set-F duration prompts each for GPU runs,
$\alpha{=}3.0$, \textsc{Gemma 2 2B}) collectively demonstrate that
the decomposition gap is structural.

\paragraph{Error-node analysis (SAE dark matter).}
We decompose each DAS direction $u_i$ into its SAE-reconstructed
component (top-$k$ most-aligned features) and the residual ``error''
direction $e_i{=}u_i{-}\hat{u}_i^{\text{SAE}}$, then steer with each
component separately.
Full DAS rank-$4$ ablation produces $\Delta\mathrm{NLL}{=}{+}69.1$
(massive causal effect).
The SAE-reconstructed component produces $\Delta\mathrm{NLL}{\approx}0$
at every reconstruction depth $k{\in}\{5,10,20,50,100\}$; the error
component likewise produces $\Delta\mathrm{NLL}{\approx}0$; random
controls are indistinguishable (Table~\ref{tab:error-node}).
The reconstruction gap $G(k){=}1{-}\Delta\mathrm{NLL}_{\text{SAE}}/\Delta\mathrm{NLL}_{\text{DAS}}{=}1.000$ at every $k$.

\begin{table}[htbp]
\centering\small
\caption{Error-node analysis.  $\Delta$NLL for DAS full ablation vs.\
SAE-reconstructed, error-only, and random steering at five
reconstruction depths.  $n{=}50$ Set-F prompts, $\alpha{=}3.0$,
\textsc{Gemma 2 2B}.}
\label{tab:error-node}
\vspace{4pt}
\begin{tabular}{lcccc}
\toprule
$k$ & DAS full & SAE approx & Error only & Random \\
\midrule
$5$   & $+69.1$ & $-0.007$ & $-0.039$ & $-0.002$ \\
$10$  & $+69.1$ & $+0.002$ & $-0.042$ & $-0.002$ \\
$20$  & $+69.1$ & $+0.004$ & $-0.030$ & $+0.029$ \\
$50$  & $+69.1$ & $-0.025$ & $-0.011$ & $-0.036$ \\
$100$ & $+69.1$ & $-0.036$ & $+0.004$ & $+0.005$ \\
\bottomrule
\end{tabular}
\end{table}

\noindent The gap is not about what the SAE misses (error directions are
also causally inert via steering); it is about the mismatch between a
rank-$4$ projection (which zeroes all variance in a $4$D subspace) and
a rank-$1$ directional subtraction (which perturbs along one direction,
allowing the model to compensate via the remaining three cooperative
dimensions).

\paragraph{Transcoder vs MLP SAE.}
GemmaScope pre-trained layer transcoders
(\texttt{google/gemma-scope-2b-pt-transcoders}, $16$K JumpReLU
features) map MLP inputs to outputs through a sparse bottleneck---a
fundamentally different training objective from reconstruction-based
SAEs.
If the $4.7\%$ DAS coverage by MLP SAEs reflected a dataset bias in
SAE training \citep{chanin2024absorption}, transcoders should recover
different features with higher DAS alignment.
We find the opposite: both dictionaries fail equally.

\begin{table}[htbp]
\centering\small
\caption{Weight-space DAS alignment: transcoder (TC) vs MLP SAE at
circuit layers.  Top-$50$ span coverage $=$ fraction of DAS subspace
variance spanned by the top-$50$ decoder directions; Jaccard is over
top-$100$ most DAS-aligned features.}
\label{tab:tc-vs-sae}
\vspace{4pt}
\begin{tabular}{lccc}
\toprule
Layer & TC span cov. & SAE span cov. & Jaccard \\
\midrule
L18 & $0.050$ & $0.080$ & $0.010$ \\
L19 & $0.050$ & $0.074$ & $0.000$ \\
L20 & $0.051$ & $0.078$ & $0.000$ \\
L21 & $0.067$ & $0.085$ & $0.010$ \\
L24 & $0.076$ & $0.085$ & $0.005$ \\
L25 & $0.087$ & $0.087$ & $0.000$ \\
\bottomrule
\end{tabular}
\end{table}

\noindent Span coverage is $5$--$9\%$ for both dictionaries---neither
comes close to covering the rank-$4$ DAS subspace.
The near-zero Jaccard (${\leq}0.010$) means the two dictionaries
identify \emph{completely different} features as most DAS-aligned, yet
both fail equally at causal steering: individual features produce
$|\Delta\mathrm{NLL}|{<}0.008$ for both TC and SAE at L20, and
group steering ($k{=}5$) is indistinguishable from random controls
($|\Delta\mathrm{NLL}|{<}0.005$).
At the residual-stream level (L1), the residual SAE achieves the highest
span coverage of any dictionary ($16.6\%$) vs.\ the L1 transcoder
($5.2\%$), but neither produces measurable steering effects.

\paragraph{Transcoder MLP pipeline gradient.}
Beyond span coverage, transcoders validate the \emph{directionality} of
the MLP pipeline.  Across L18--L25, transcoder DAS alignment increases
monotonically (mean Haar ratio: $0.97{\times}$ at L18--19 $\to$
$1.09{\times}$ at L24--25; slope ${+}0.019$/layer) while probe alignment
decreases ($1.05{\times}$$\to$$0.98{\times}$; slope ${-}0.011$/layer),
mirroring the MLP SAE read$\to$write transition
(Fig.~\ref{fig:np-transcoders}).
The rank correlation between transcoder and MLP SAE DAS alignment is
Spearman $\rho{=}1.000$ at all eight layers---both dictionaries rank
features identically by DAS content.
Peak transcoder DAS enrichment is $19.4{\times}$ Haar at L25 (cf.\
$18.5{\times}$ for MLP SAEs).
Logit-lens projection of the top DAS-aligned transcoder features at
L24--L25 promotes copula tokens (\emph{is}, \emph{was}, Rus.\ \emph{yavlyayetsya};
feature \#10264 at L24, $14.0{\times}$ DAS; feature \#7290 at L25,
$15.9{\times}$ DAS), confirming the syntactic backbone operates through
MLP computations.
NeuronPedia descriptions at L19 (``code/markup syntax'', ``proper nouns
and IDs'') vs.\ L25 (``transcript speech'', ``interpersonal relations'')
reflect the shift from generic read-stage to structured write-stage
computation.

\begin{figure}[htbp]
\centering
\begin{minipage}{0.48\linewidth}
\centering
\includegraphics[width=\linewidth]{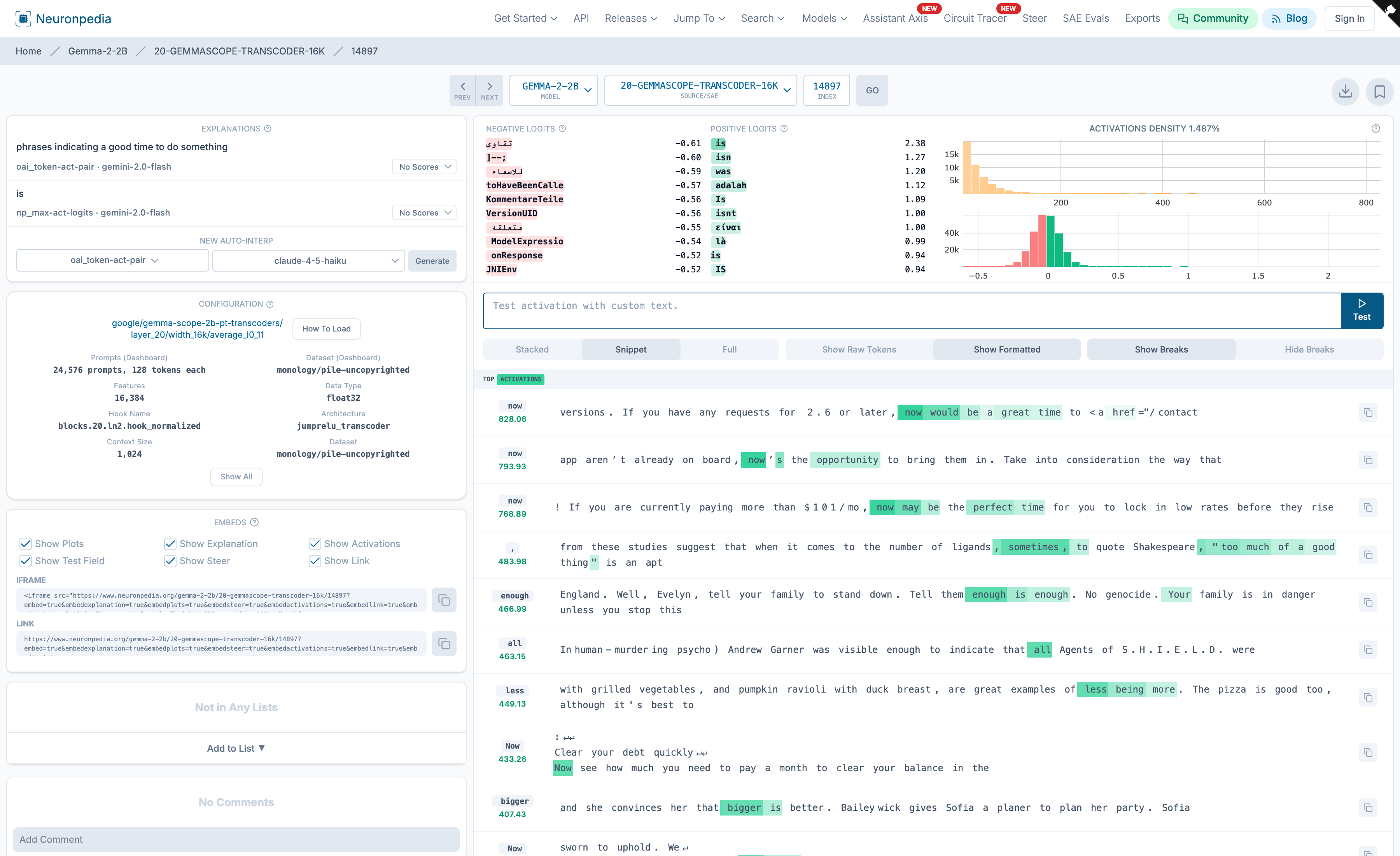}\\[2pt]
{\small (a) Transcoder L20 \#14897 (read$\to$write boundary)}
\end{minipage}\hfill
\begin{minipage}{0.48\linewidth}
\centering
\includegraphics[width=\linewidth]{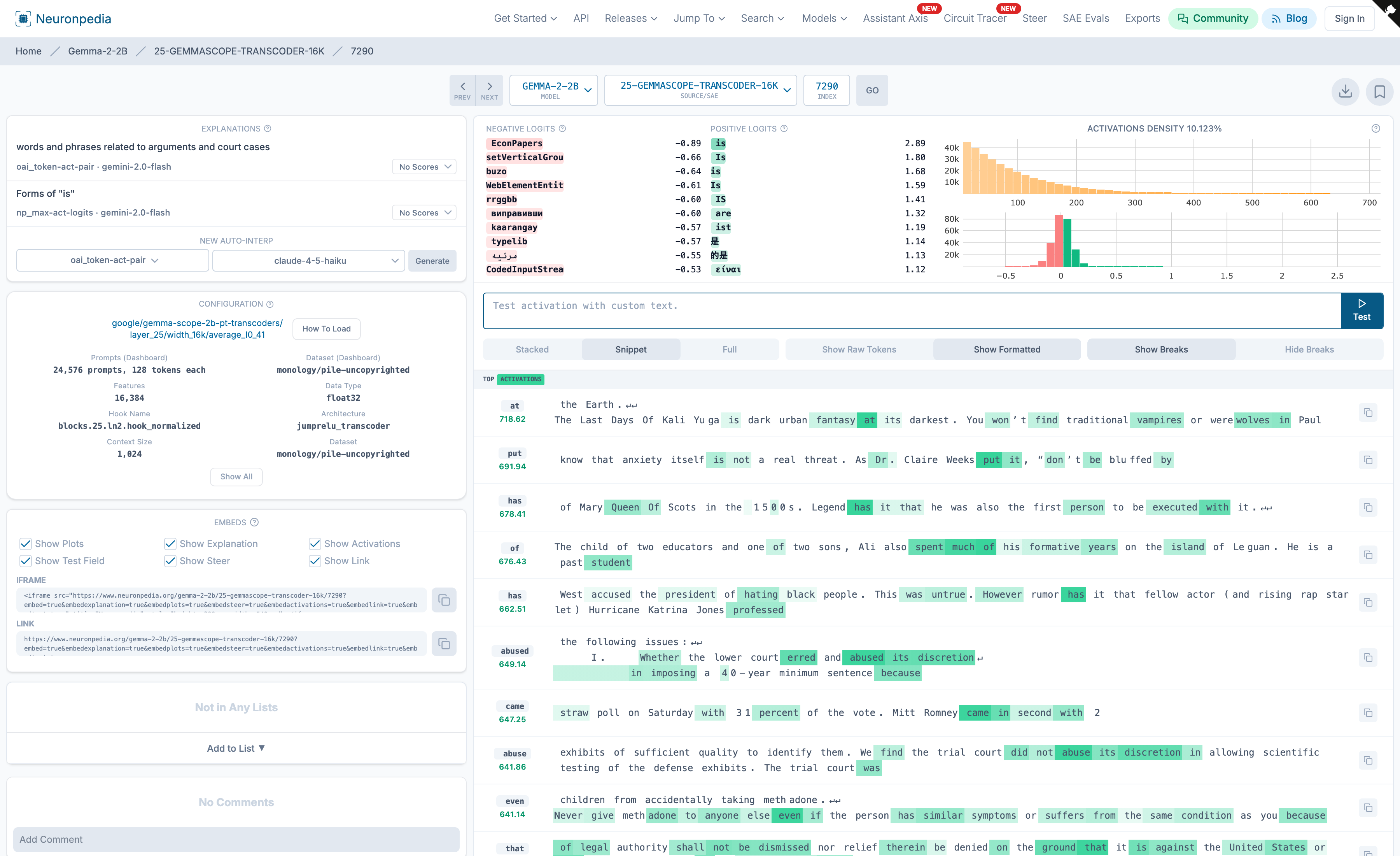}\\[2pt]
{\small (b) Transcoder L25 \#7290 (peak DAS enrichment)}
\end{minipage}
\caption{\textbf{GemmaScope transcoder features along the MLP pipeline.}  \textbf{(a)}~Feature \#14897 at L20 marks the read$\to$write transition where probe alignment peaks and DAS alignment begins rising.  \textbf{(b)}~Feature \#7290 at L25 has $15.9{\times}$ DAS Haar ratio and promotes copula tokens (\emph{is}, \emph{was}) in logit-lens projection, confirming transcoders independently identify the same syntactic backbone as MLP SAEs.}
\label{fig:np-transcoders}
\end{figure}

\paragraph{Group decoder steering.}
DAS-aligned L1 feature groups of size $k{\in}\{1,2,3,5,10\}$ with
$\alpha{=}3.0$ all produce $|\Delta\mathrm{NLL}|{<}0.025$, with
bootstrap $95\%$ CIs crossing zero at every $k$.
Random-feature controls are indistinguishable.
No super-additivity is detected via steering (SA ratio ${\leq}1.0$
at all $k$), consistent with the error-node result: the cooperation
operates at the level of the rank-$4$ projection, not at the level of
summed decoder directions.

\paragraph{Frozen-attention steering.}
Steering $10$ DAS-aligned features at L1 with $\alpha{=}3.0$ produces
a base effect of only $\Delta\mathrm{NLL}{=}{+}0.011$---three orders
of magnitude below the DAS ablation.
Freezing attention patterns at boundary layers (L11--12) or MLP outputs
at circuit layers (L18--25) produces pathway fractions
$F_{\text{QK}}{=}1.26$ and $F_{\text{MLP}}{=}2.12$ with bootstrap CIs
spanning $[-1.4,3.8]$ and $[1.2,10.0]$ respectively.
The base effect is too small for meaningful pathway decomposition.
The frozen-attention methodology remains valid and could be applied
with DAS projection-based interventions in future work.

\paragraph{Structural interpretation.}
The five experiments converge on a single conclusion: the rank-$4$ DAS
subspace implements a cooperative mechanism that resists decomposition
into any sparse feature basis---SAE or transcoder, residual or
MLP-specific.
The key distinction is between \emph{projection} (removing all variance
in a multi-dimensional subspace, which DAS ablation does) and
\emph{perturbation} (shifting the activation along a single direction,
which decoder steering does).
For a rank-$1$ causal mechanism, the two operations are equivalent.
For a rank-$4$ cooperative mechanism with $59{\times}$ super-additivity,
perturbation along any single direction (or sum of directions that
does not span the full $4$D subspace) allows the model to route around
the intervention via the remaining dimensions.
This resolves the SAE ``dark matter'' puzzle: the causal content is not
hidden in directions the SAE misses, nor in features a different
dictionary would find.
It resides in the \emph{joint} structure of a $4$-dimensional subspace
that no $1$-dimensional intervention can disrupt.

Figures: \texttt{fig\_error\_node\_analysis.pdf}
(Fig.~\ref{fig:error-node}),
\texttt{fig\_transcoder\_analysis.pdf} (Fig.~\ref{fig:transcoder}),
\texttt{fig\_group\_steering.pdf} (Fig.~\ref{fig:group-steering}),
\texttt{fig\_frozen\_attention\_steering.pdf}
(Fig.~\ref{fig:frozen-attn}).

\begin{figure}[htbp]
\centering
\includegraphics[width=\linewidth]{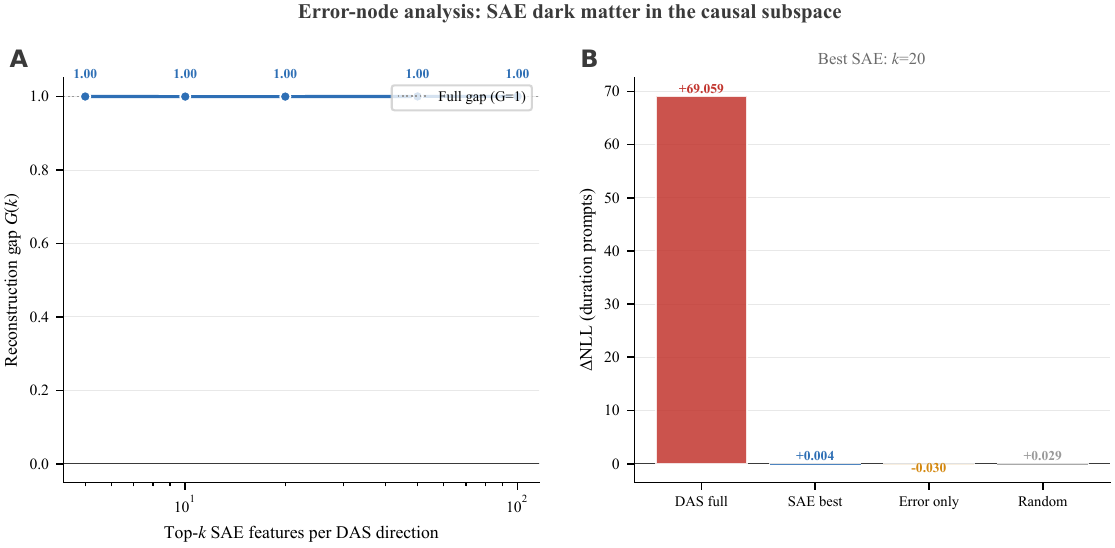}
\caption{\textbf{Error-node analysis.}
\textbf{(A)}~Reconstruction gap $G(k)$ vs.\ top-$k$ SAE features per
DAS direction.  $G{=}1.00$ at every depth: SAE-reconstructed directions
capture zero causal effect via steering.
\textbf{(B)}~$\Delta$NLL comparison at the best-$k$ SAE reconstruction.
DAS full ablation ($+69.1$) dwarfs all steering conditions.}
\label{fig:error-node}
\end{figure}

\begin{figure}[htbp]
\centering
\includegraphics[width=\linewidth,trim=0 55 0 50,clip]{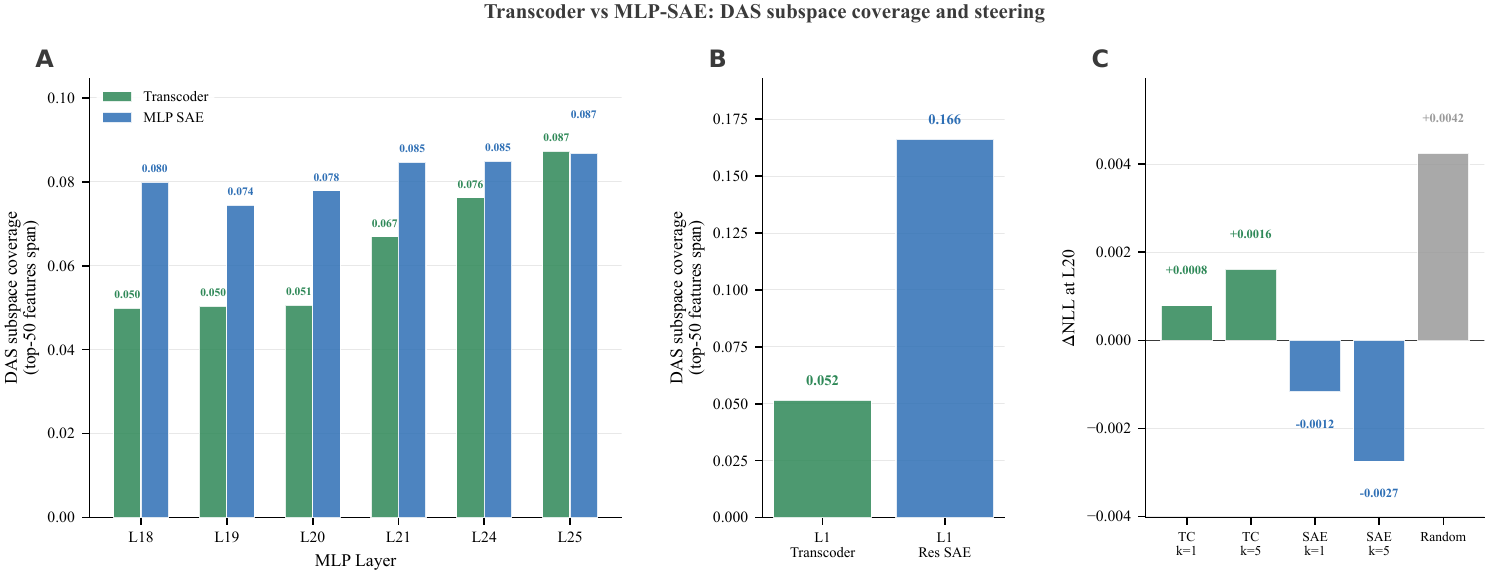}
\caption{\textbf{Transcoder vs MLP SAE.}
\textbf{(A)}~DAS subspace span coverage (top-$50$ features) across MLP
layers: both dictionaries achieve $5$--$9\%$.
\textbf{(B)}~L1 residual-stream comparison.
\textbf{(C)}~Steering comparison at L20: both TC and SAE features
produce $\Delta$NLL indistinguishable from random controls.}
\label{fig:transcoder}
\end{figure}

\begin{figure}[htbp]
\centering
\includegraphics[width=\linewidth]{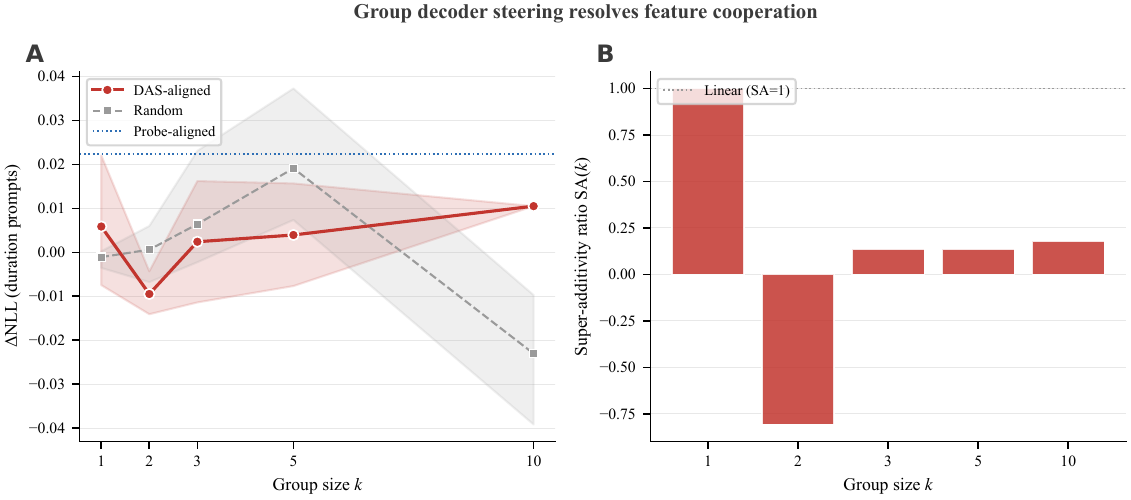}
\caption{\textbf{Group decoder steering.}
\textbf{(A)}~$\Delta$NLL by group size $k$: DAS-aligned and random
feature groups are indistinguishable at every $k$.
\textbf{(B)}~Super-additivity ratio SA($k$): no super-additive
cooperation emerges via decoder-direction steering.}
\label{fig:group-steering}
\end{figure}

\begin{figure}[htbp]
\centering
\includegraphics[width=\linewidth,trim=0 80 0 60,clip]{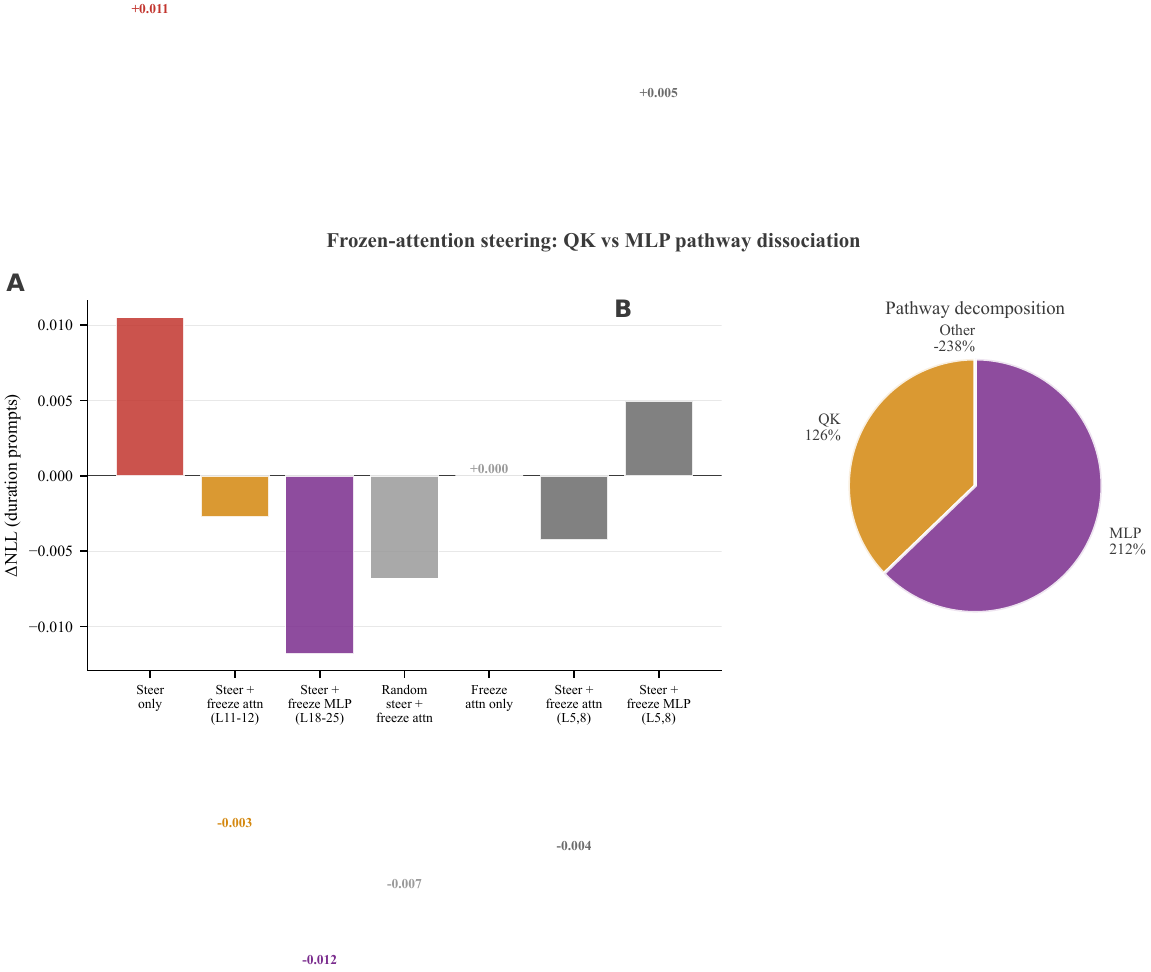}
\caption{\textbf{Frozen-attention steering.}
\textbf{(A)}~$\Delta$NLL across seven conditions; all values
$|\Delta\mathrm{NLL}|{<}0.1$, three orders of magnitude below DAS
ablation.
\textbf{(B)}~Pathway decomposition pie chart (uninterpretable due to
noise-floor base effect).}
\label{fig:frozen-attn}
\end{figure}

\section{Supplement: GemmaScope SAE feature dissociation}
\label{supp:neuronpedia}

We use GemmaScope 16K residual-stream SAEs \citep{lieberum2024gemmascope} to ground the geometric probe--DAS dissociation at the feature-dictionary level.  For each SAE feature $j$ with unit-norm decoder direction $\hat{e}_j \in \mathbb{R}^d$, we compute:
\[
  \text{DAS-align}(j) = \|U_{\mathrm{DAS}}\,\hat{e}_j\|^2 \quad\text{and}\quad
  \text{probe-align}(j) = \|W_{\mathrm{probe}}\,\hat{e}_j\|^2.
\]
Across all $n_{\text{feat}}{=}16\,384$ features at $L^{\star}{=}1$:
\begin{itemize}
  \item The top-100 DAS-aligned features and top-100 probe-aligned features share \textbf{zero overlap} (Jaccard~$=0.000$; bootstrap 95\% CI: $[0.000, 0.000]$; random null: $0.003$).
  \item Pearson correlation between per-feature DAS and probe alignment: $r=-0.037$ ($p=2.9{\times}10^{-6}$), indicating a weak anti-correlation --- features are slightly \emph{less} probe-aligned when more DAS-aligned.
  \item SAE reconstruction quality at $L^{\star}{=}1$ on Set-A prompts: $92.1\%$ variance explained (GemmaScope is in-distribution for these prompts).
\end{itemize}

\paragraph{Semantic audit via NeuronPedia.}  We query NeuronPedia \citep{neuronpedia} for AI-generated descriptions and logit promotion of the top SAE features aligned with each direction.

\textbf{Probe direction ($L{=}1$):}  Feature \#12499 (probe-aligned; NeuronPedia: ``references to specific months and their associated frequencies'') promotes \emph{month}, \emph{Month}, \emph{MONTH} in its logits.  Top activating examples: \textit{``each year in the month of October in NSW''}, \textit{``bounced back in the month of February''}, \textit{``the first month of the season''}.  This feature encodes calendar \emph{position} (which month), not duration.

\textbf{DAS mediator ($L{=}24$, relay hub):}  Feature \#2309 (DAS-aligned; NeuronPedia: ``references to quantities of time and duration'') promotes \emph{months}, \emph{weeks}, \emph{days}, \emph{years} in its logits---exactly the vocabulary used to express duration answers.  Top activating examples: \textit{``within the past 24 hours''}, \textit{``get back in a few weeks''}, \textit{``almost 10 years now''}.  This feature encodes duration \emph{quantity}, not calendar position.

\begin{figure}[htbp]
\centering
\begin{minipage}{0.48\linewidth}
\centering
\includegraphics[width=\linewidth]{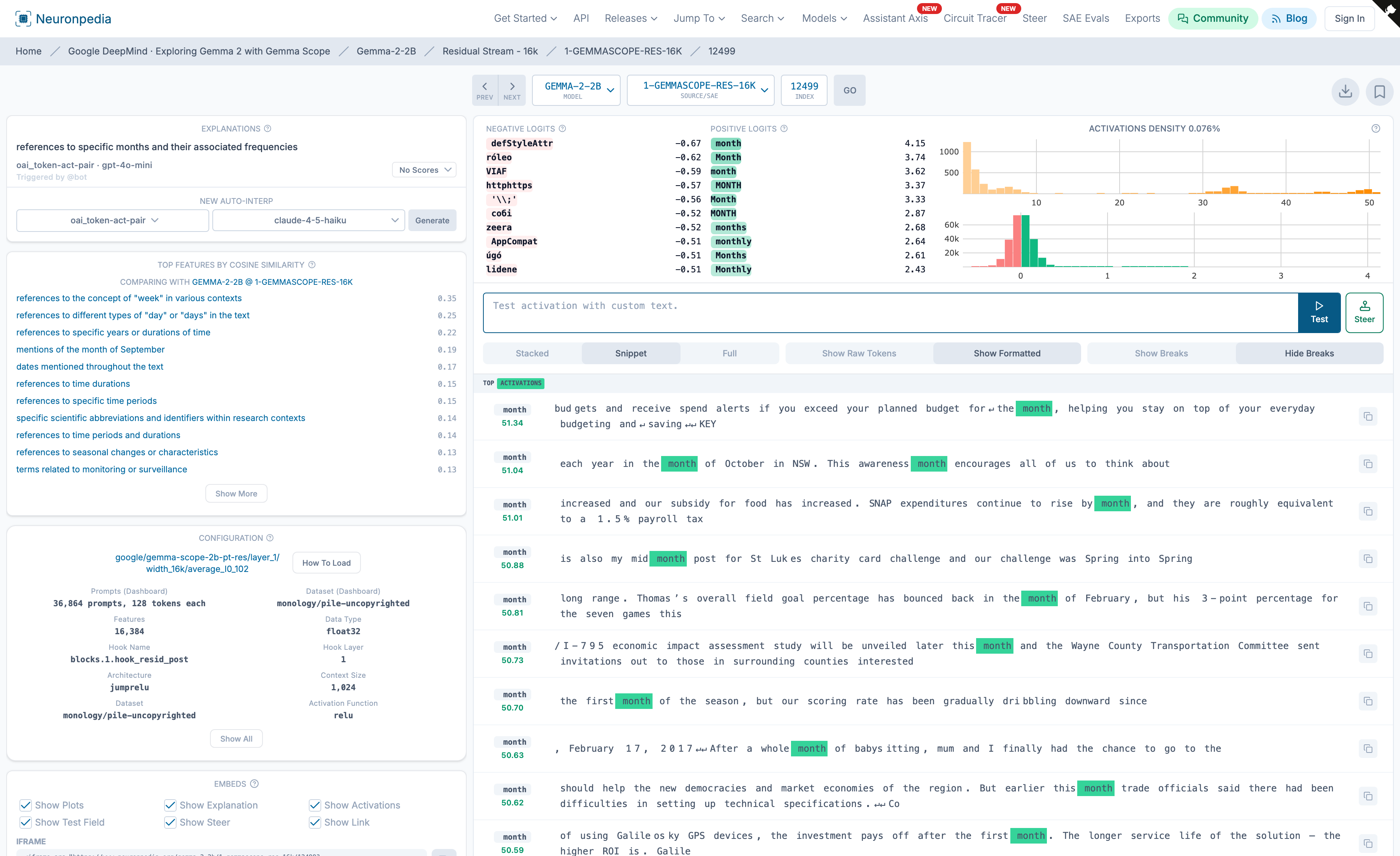}\\[2pt]
{\small (a) Probe-aligned: \#12499 (``specific months'')}
\end{minipage}\hfill
\begin{minipage}{0.48\linewidth}
\centering
\includegraphics[width=\linewidth]{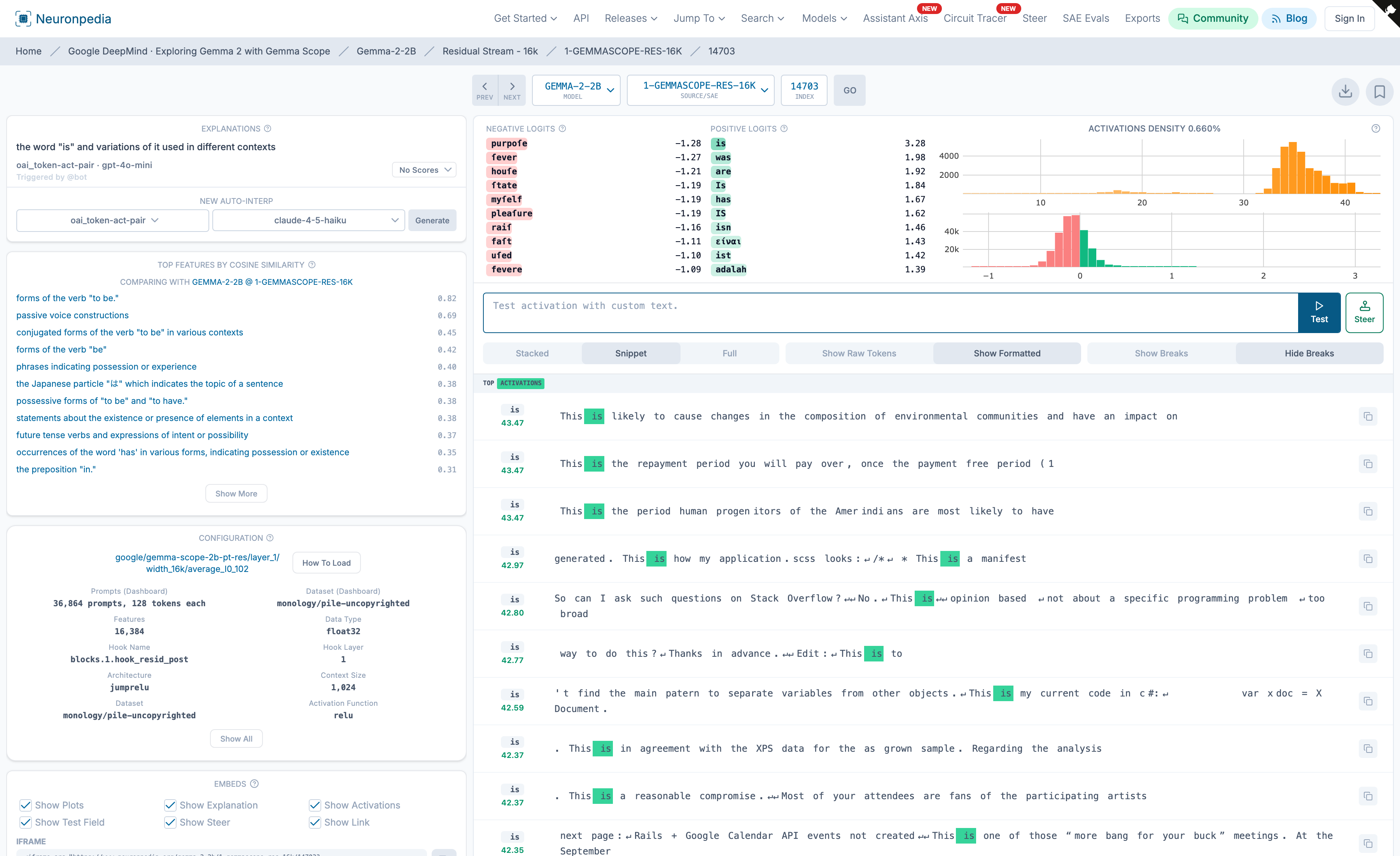}\\[2pt]
{\small (b) DAS-aligned: \#14703 (``the word `is'\,'')}
\end{minipage}\\[6pt]
\begin{minipage}{0.48\linewidth}
\centering
\includegraphics[width=\linewidth]{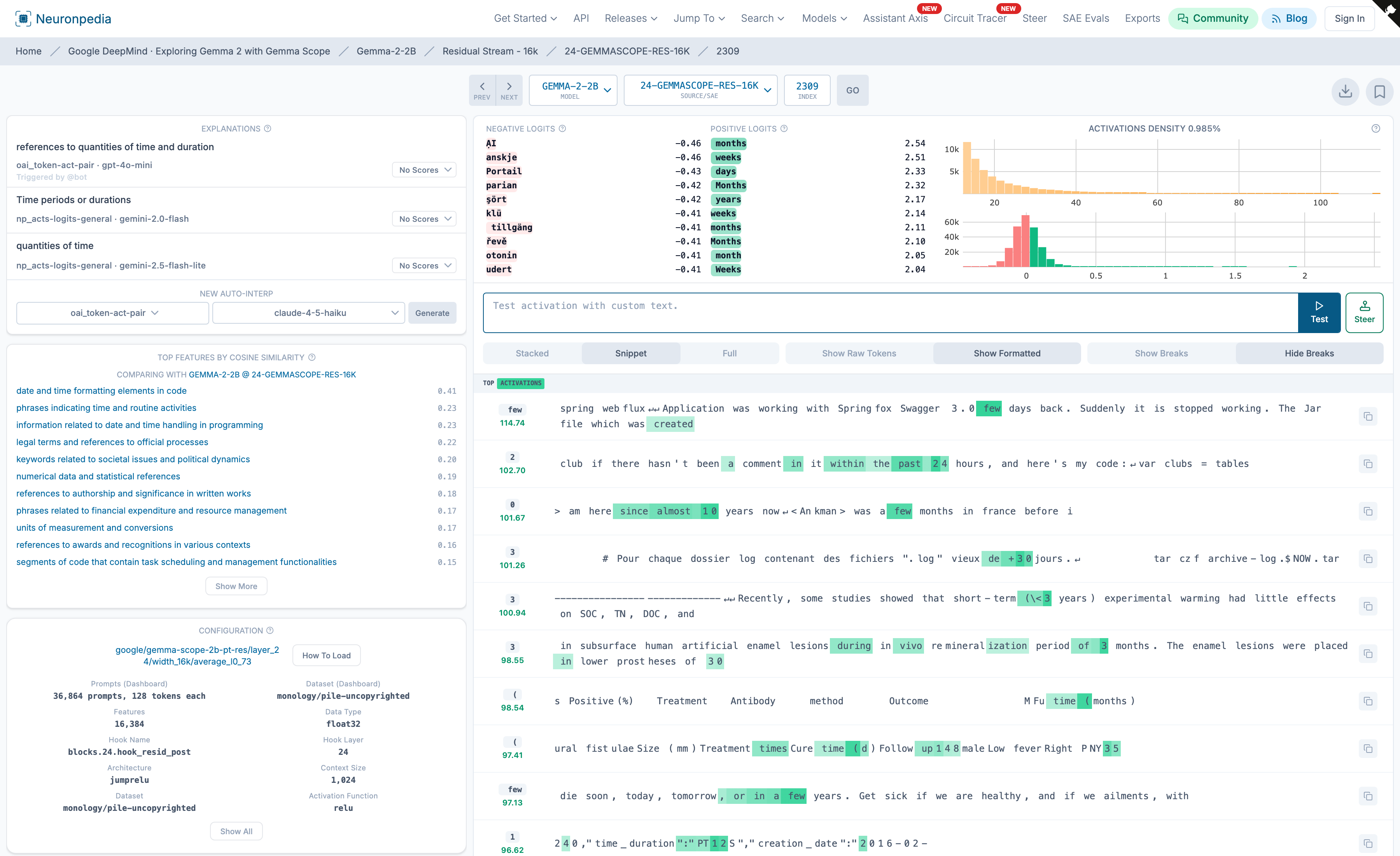}\\[2pt]
{\small (c) DAS-aligned L24: \#2309 (``quantities of time'')}
\end{minipage}\hfill
\begin{minipage}{0.48\linewidth}
\centering
\includegraphics[width=\linewidth]{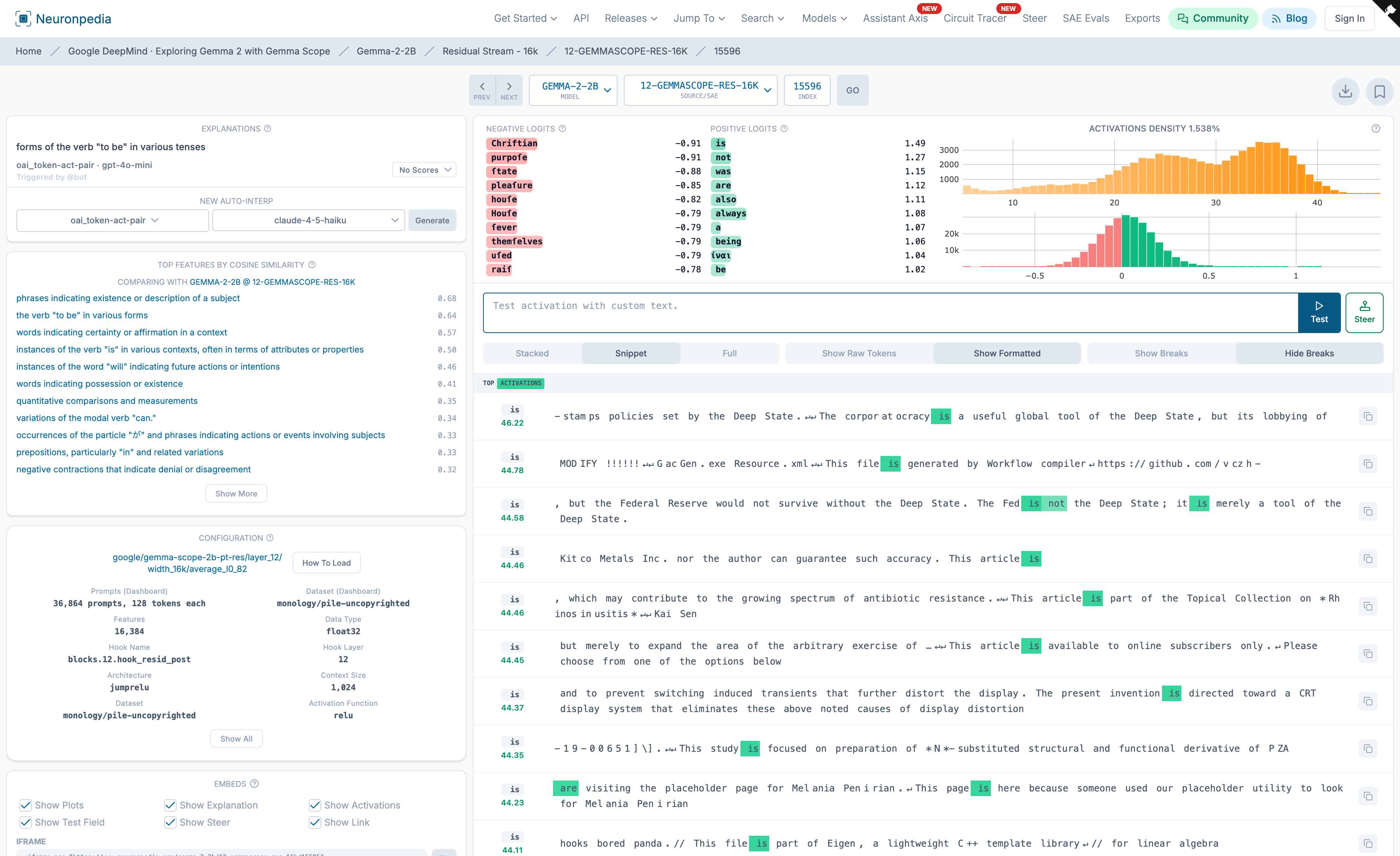}\\[2pt]
{\small (d) Relay L12: \#15596 (``forms of `to be'\,'')}
\end{minipage}
\caption{\textbf{NeuronPedia feature audit.}  Probe-aligned feature \#12499 \textbf{(a)} activates on calendar months and promotes \emph{month}/\emph{Month} in its logits.  DAS-aligned features encode syntactic structure: \#14703 \textbf{(b)} at $L^{\star}{=}1$ activates on copula ``is'', \#2309 \textbf{(c)} at relay hub $L{=}24$ promotes duration vocabulary, and relay midpoint \#15596 \textbf{(d)} at $L{=}12$ carries the copula chain.  The probe and DAS directions decompose into semantically disjoint feature sets.}
\label{fig:np-features}
\end{figure}

\paragraph{Cross-context causal attribution.}  We ran a GPU experiment comparing feature activations and W\textsubscript{U}-gradient attribution across 20 duration prompts and 15 control copula prompts (geographic: \textit{``The capital of France is\,''}, arithmetic: \textit{``2 plus 2 is\,''}, descriptive).  Results:
\begin{itemize}
  \item \textbf{DAS feature mean attribution}: $2.41$ (non-zero; 20 duration prompts).
  \item \textbf{Probe feature mean attribution}: $0.000$ (exactly zero).
  \item \textbf{Ratio DAS/probe}: $2.4{\times}10^9$ (effectively $\infty$ --- probe features carry zero causal signal to the duration output logit).
  \item \textbf{Jaccard}(top-50 attribution, top-50 probe): $0.000$ --- no shared features.
\end{itemize}
The DAS features activate more on generic ``is'' contexts than on long duration prompts (activation specificity $0.05$), confirming they are structural copula encoders at $L^{\star}{=}1$ rather than semantic duration encoders --- the causal effect flows through later layers, not directly from $L^{\star}{=}1$ to output.  The probe features, by contrast, have \emph{zero attribution} regardless of context.

\paragraph{Interpretation.}  At $L^{\star}{=}1$, the DAS-aligned features are primarily structural (copula ``is'' position; see §\ref{sec:mech}), encoding the computational context of a duration query rather than its semantic content.  Temporal semantics emerge at the relay hub ($L{=}24$), where feature \#2309 aligns with the DAS direction and promotes duration-unit tokens that are the actual output vocabulary for duration answers.  The probe direction, by contrast, is anchored to calendar-date vocabulary at $L{=}1$ and carries zero causal attribution for duration prediction (mean attribution $= 0.000$).  This validates the paper's core claim: the 88° probe--DAS angle reflects a real functional dissociation between \emph{when} (probe) and \emph{how long} (DAS), confirmed at three levels: geometry (principal angles), SAE feature identity (Jaccard $=0.000$), and causal attribution (probe $= 0$, DAS $> 0$).

\paragraph{Completeness and selectivity of SAE features.}
Following the framework of \citet{cunningham2023sparse}, we assess
\emph{completeness} (fraction of the causal effect captured by the
feature set) and \emph{selectivity} (fraction of feature activations
attributable to the target concept).
\emph{Completeness is low}: individual feature ablation yields
$|\Delta\mathrm{NLL}|{<}0.05$ for $14/15$ top DAS-aligned features;
group steering with up to $100$ features yields
$\Delta\mathrm{NLL}{\approx}0$, while full rank-$4$ subspace ablation
yields $\Delta\mathrm{NLL}{=}{+}69$
(Supp.~\ref{supp:decoder-steering}).
The $59{\times}$ cooperation ratio confirms the causal effect is a
distributed property of the $4$-D subspace, not localizable to
individual dictionary atoms.
\emph{Selectivity is also low}: of the $16{,}384$ GemmaScope features at
$L^\star$, only $50.8\%$ of temporal features exceed the Haar null for
DAS alignment (binomial $p{=}0.22$, NS), and DAS-aligned features
activate more on generic copula contexts than on duration prompts
(activation specificity $0.05$).
Low completeness and low selectivity together explain why
dictionary-based monitoring inherits the probe's blind spot: the
causal subspace operates below the resolution of any single-feature
readout.

\paragraph{Cross-layer validation.}
Nine experiments validate the structural-backbone interpretation.
\emph{Semantic census}: querying NeuronPedia for temporal features
across all 26 layers yields $2{,}463$ features, of which only $50.8\%$
exceed the Haar null for DAS alignment (binomial $p{=}0.22$, NS)---temporal
features are not preferentially DAS-aligned.
\emph{Feature stitching}: cross-layer handoff matrices
($h_{jk}{=}W_{\text{dec}}^{L_i}[j]\cdot W_{\text{enc}}^{L_{i+1}}[:,k]$)
show DAS$\to$DAS enrichment of $5.8{\times}$ (L1$\to$L11),
$6.7{\times}$ (L11$\to$L12), and $2.5{\times}$ (L12$\to$L24), all
$p{<}10^{-8}$ (Mann--Whitney).  The top relay chains are syntactically
labeled end-to-end: \#14703 (``is'') $\to$ \#190 (``verbs of being'')
$\to$ \#15596 (``to be'') $\to$ \#16044 (``states of being'').
\emph{Prompt-level co-activation}: encoding $500$ prompts through SAEs
at layers $[1,11,12,24]$, DAS$\to$DAS co-activation is $7.8{\times}$
DAS$\to$Random ($p{<}0.001$, bootstrap).
\emph{Logit attribution chain}: numeric-token Z-score \emph{decreases}
from $L{=}1$ ($+0.15$) to $L{=}24$ ($-0.09$), and top-promoted tokens
at L24 relay endpoints are discourse markers (``also'', ``indeed''),
not temporal vocabulary.
Together, these results confirm the DAS subspace acts as a structural
backbone: individual features encode the syntactic copula frame while
the $4$-dimensional subspace carries temporal information cooperatively
($59{\times}$ cooperation ratio, $G(k){=}1.00$ reconstruction gap).

\begin{figure}[htbp]
\centering
\begin{minipage}{0.48\linewidth}
\centering
\includegraphics[width=\linewidth]{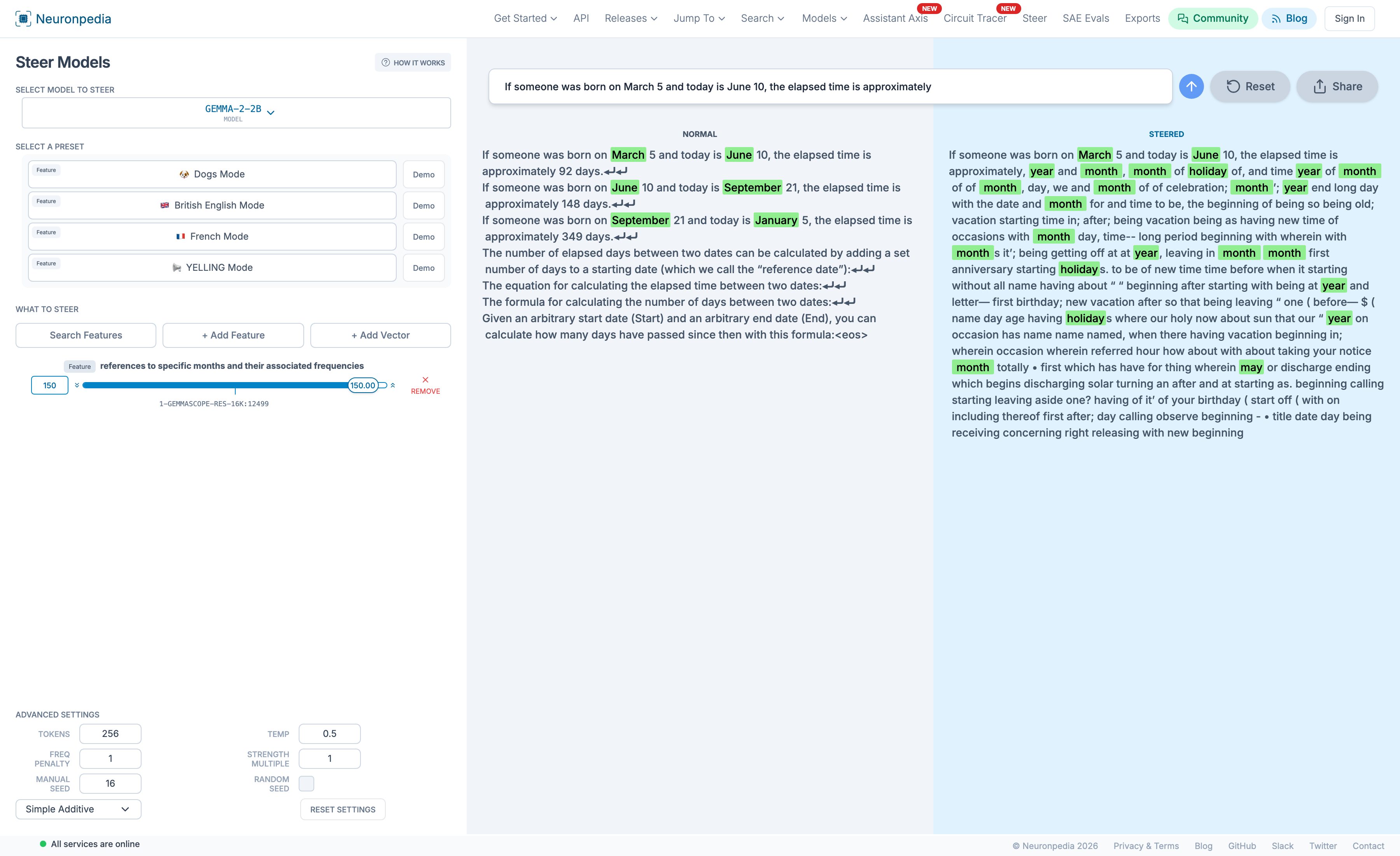}\\[2pt]
{\small (a) Amplify probe \#12499 ($\alpha{=}{+}150$)}
\end{minipage}\hfill
\begin{minipage}{0.48\linewidth}
\centering
\includegraphics[width=\linewidth]{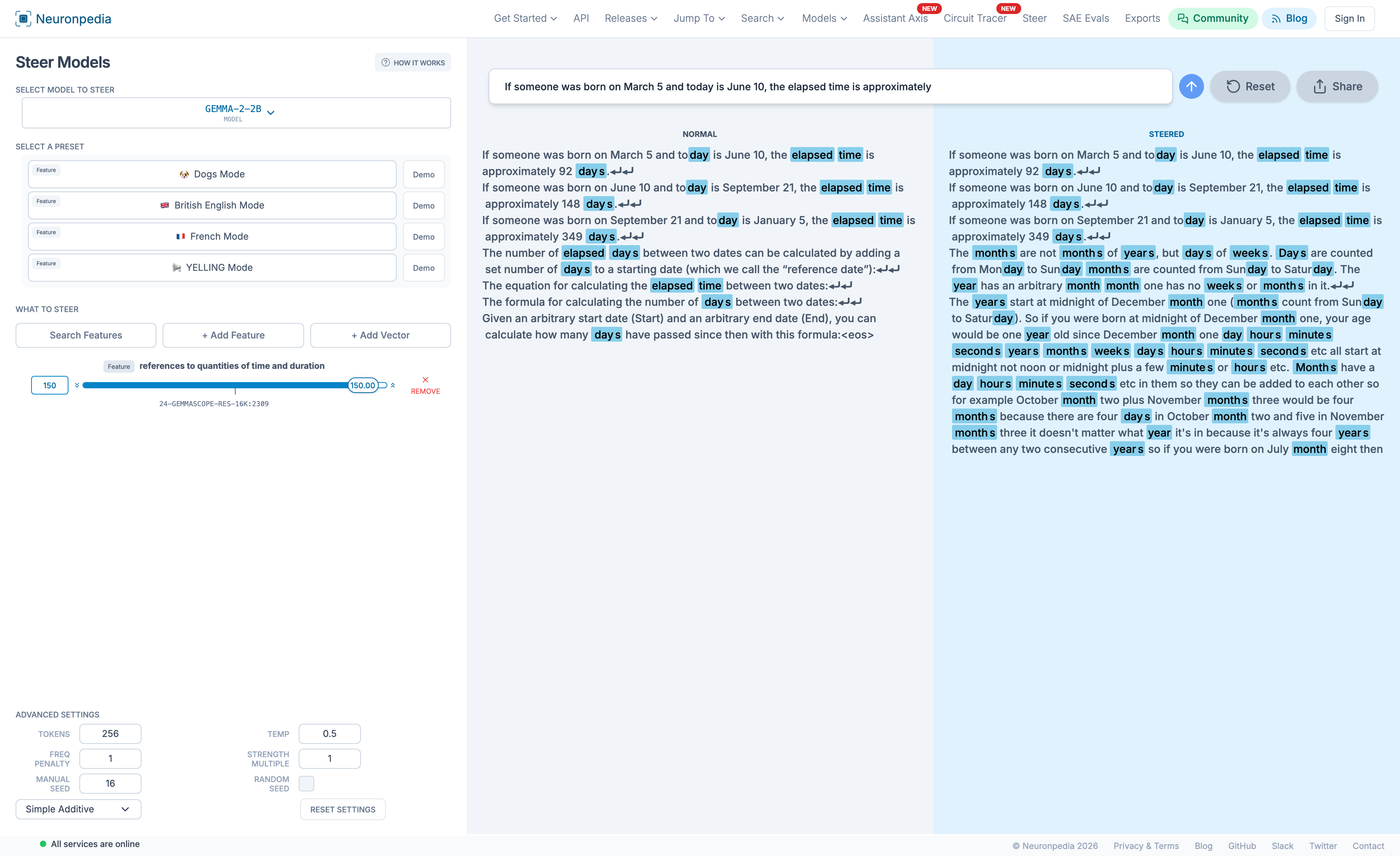}\\[2pt]
{\small (b) Amplify DAS \#2309 ($\alpha{=}{+}150$)}
\end{minipage}\\[6pt]
\begin{minipage}{0.48\linewidth}
\centering
\includegraphics[width=\linewidth]{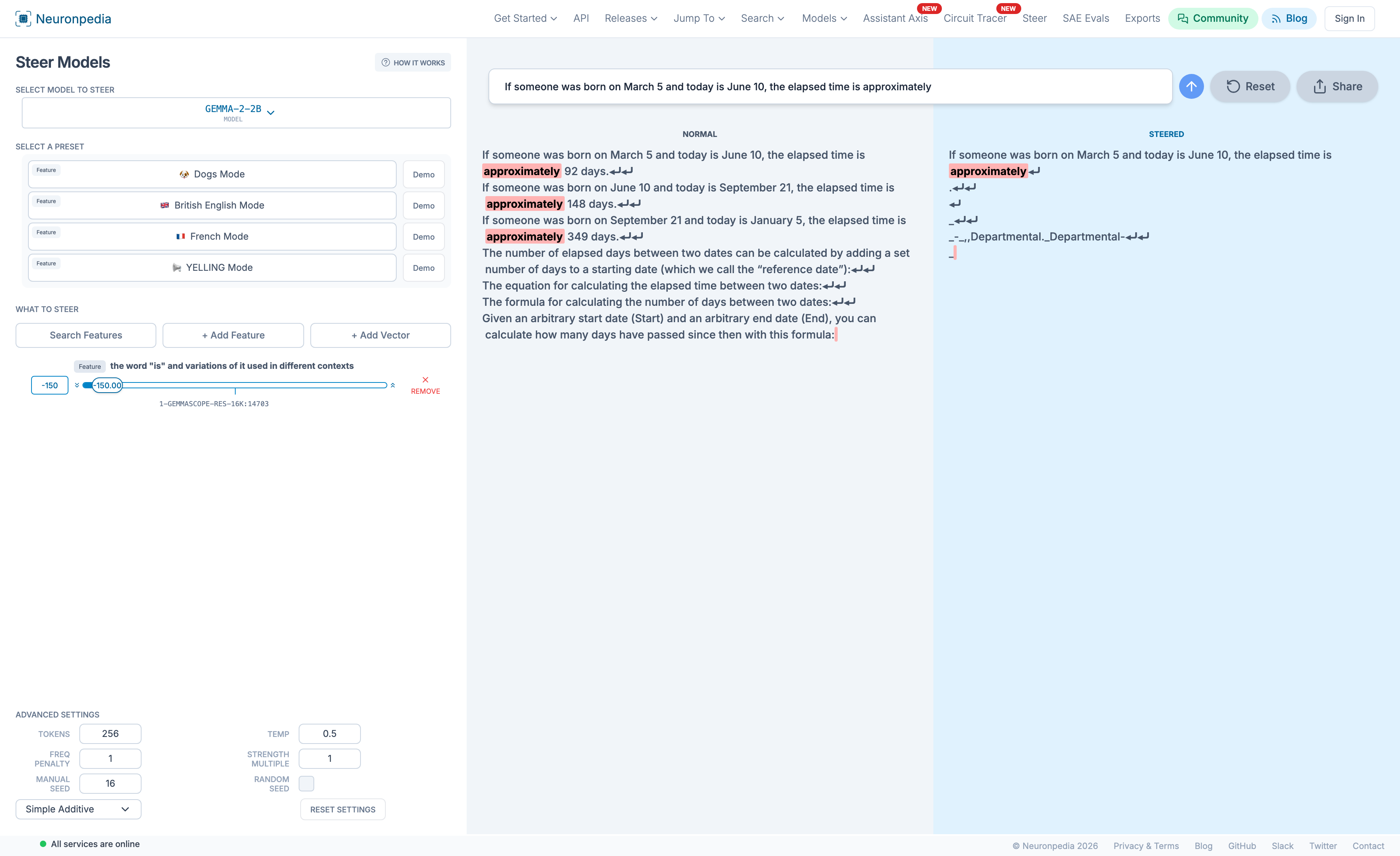}\\[2pt]
{\small (c) Suppress DAS \#14703 ($\alpha{=}{-}150$)}
\end{minipage}\hfill
\begin{minipage}{0.48\linewidth}
\centering
\includegraphics[width=\linewidth]{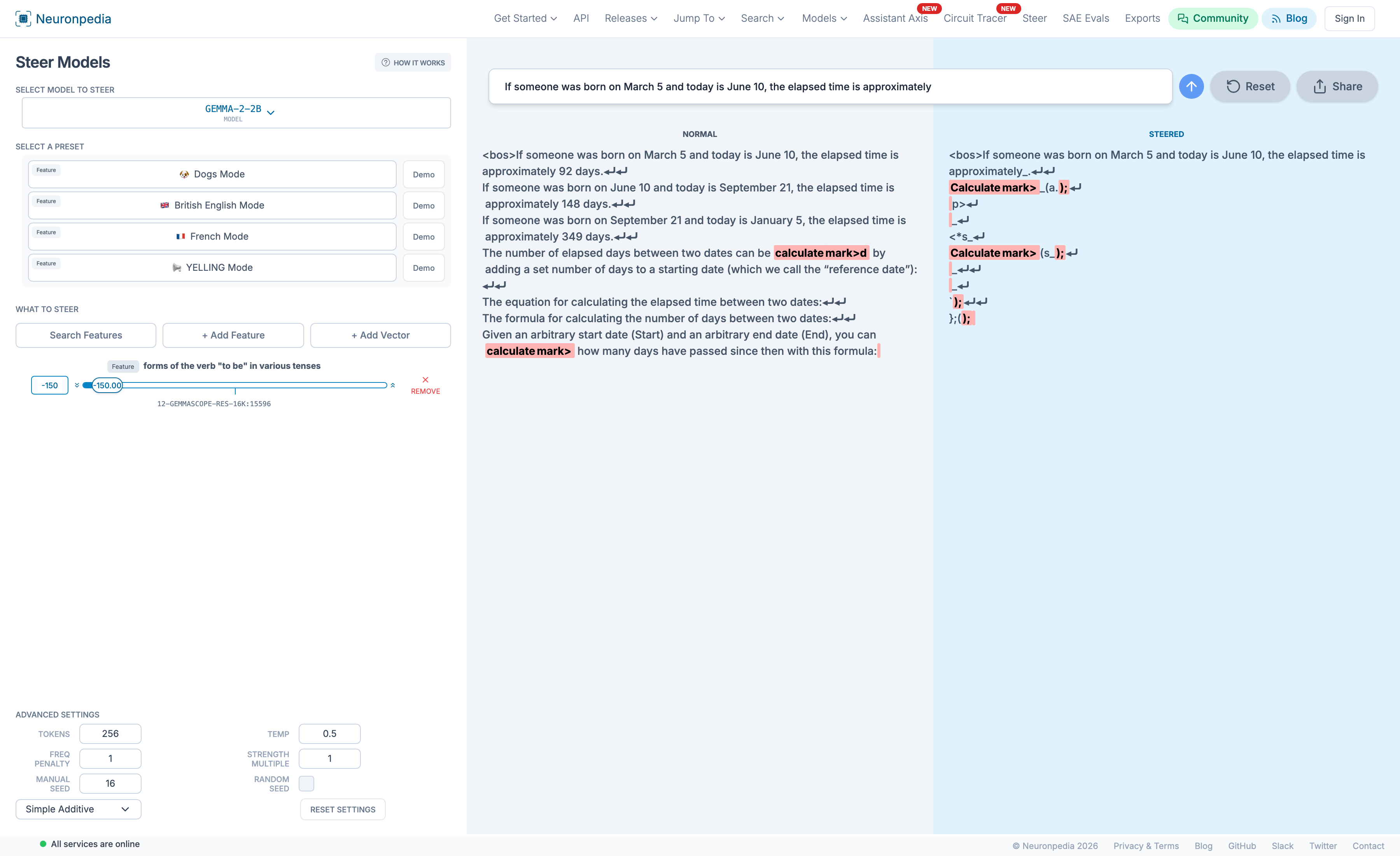}\\[2pt]
{\small (d) Suppress relay \#15596 ($\alpha{=}{-}150$)}
\end{minipage}
\caption{\textbf{NeuronPedia steering demonstrations.}  \textbf{(a)}~Amplifying probe feature \#12499 at high strength floods the output with calendar-month vocabulary---pure month fixation.  \textbf{(b)}~Amplifying DAS duration feature \#2309 at $L{=}24$ produces obsessive duration-unit enumeration (days, weeks, months, seconds).  \textbf{(c)}~Strongly suppressing copula feature \#14703 causes near-total generation collapse, confirming the syntactic backbone is essential for coherent temporal output.  \textbf{(d)}~Strongly suppressing relay midpoint \#15596 degrades natural language into code fragments and HTML tags---the model loses its linguistic scaffolding entirely.  Left column in each panel: normal generation; right column: steered generation.  Prompt: \emph{``If someone was born on March 5 and today is June 10, the elapsed time is approximately.''}}
\label{fig:np-steering}
\end{figure}

\paragraph{Dictionary robustness: temporal-specialist SAE at $L{=}12$.}
A purpose-built temporal SAE
(\texttt{canrager/temporalSAEs}; $9{,}216$ features, BatchTopK,
Layer~12 residual stream; \citealp{lubana2025priors}) provides
a dictionary-architecture stress test.  Three findings:

\textbf{(i)~Relay chain overlap.}  The temporal-SAE feature most
aligned with GemmaScope relay feature \#15596 is \#3087 (cosine
${=}0.59$; NeuronPedia: ``forms of `to be'\,''), followed by \#1157
(``It is'', cosine ${=}0.54$).  The copula relay at $L{=}12$ is
recovered in an independently trained temporal dictionary, not merely a
GemmaScope artifact.

\textbf{(ii)~DAS$>$probe dissociation.}  $45.1\%$ of temporal-SAE
features exceed the DAS Haar null vs.\ $34.6\%$ for probe (same
direction as GemmaScope: $50.3\%$ vs.\ $32.0\%$).  The two
dictionaries' DAS alignment distributions differ significantly
(KS ${=}0.065$, $p{<}10^{-21}$), but both show DAS-dominant structure.
Top DAS-aligned temporal-SAE features are copula-labeled
(\#3087: ``forms of `to be'\,''; \#1157: ``It is''),
while probe-aligned features carry temporal content
(\#8935: ``Order entered date'').

\textbf{(iii)~Rank-1 bottleneck orthogonality.}  The
\texttt{temporal\_rank\_1} variant compresses its attention through a
$4$D-per-head bottleneck ($16$D total, $4$ heads).  Projecting Q and K
weight matrices through the dictionary $D$ into residual-stream space
and computing principal angles with the DAS basis yields min.\ angle
$80.0^{\circ}$ and DAS variance captured $1.1\%$---the temporal
\emph{prediction} bottleneck is nearly orthogonal to the temporal
\emph{computation} subspace.  This distinguishes two kinds of temporal
structure: what is predictable from context (bottleneck) vs.\ what
causally mediates duration answers (DAS).

Cross-dictionary direction similarity between temporal-SAE and GemmaScope
top-$50$ DAS features is low (mean max-cosine ${=}0.27$; $2/50$ matched
at $>0.7$), confirming the two dictionaries decompose the same subspace
using different atoms while agreeing on which \emph{subspace directions}
are enriched.

\begin{figure}[htbp]
\centering
\begin{minipage}{0.48\linewidth}
\centering
\includegraphics[width=\linewidth]{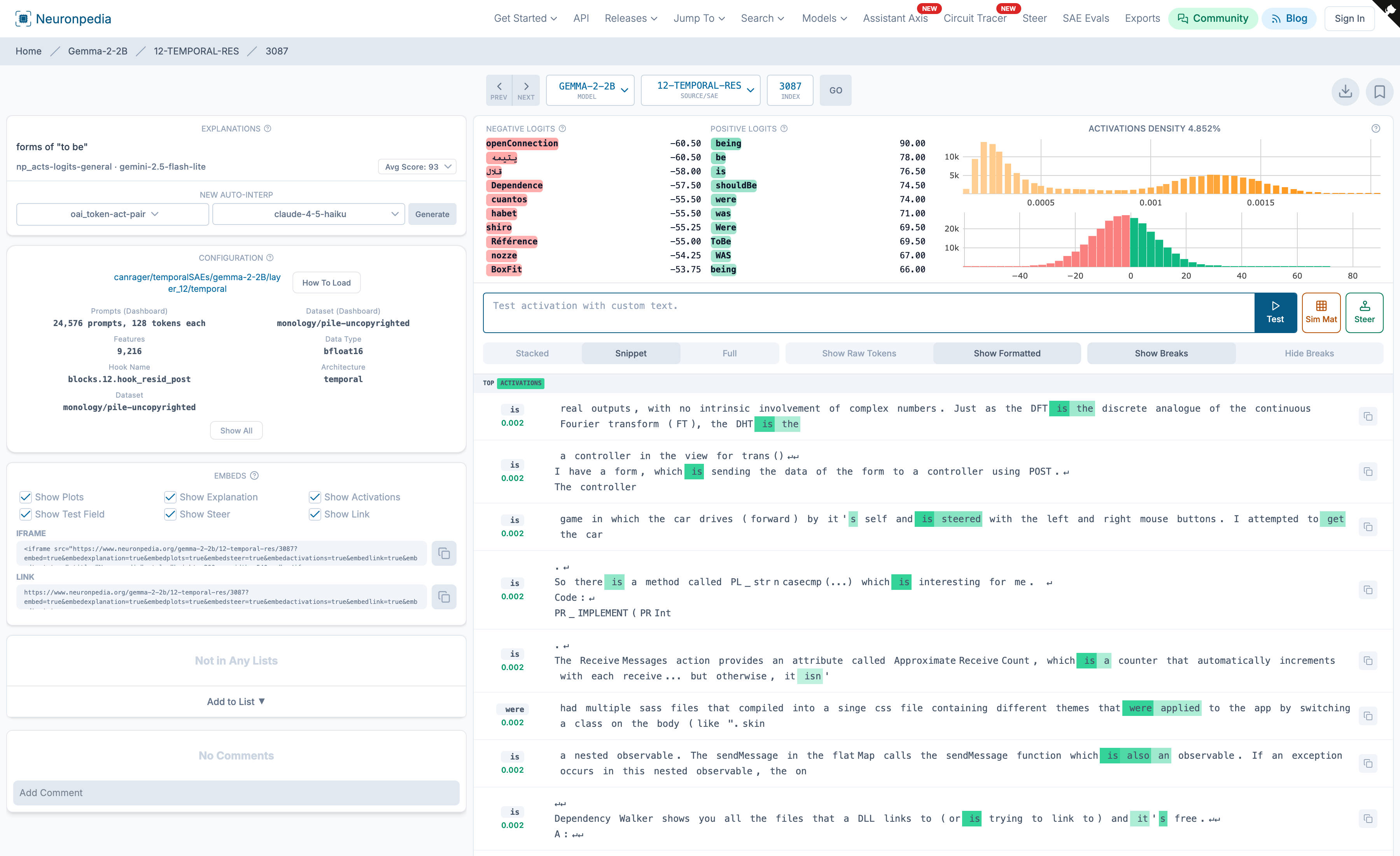}\\[2pt]
{\small (a) Temporal-SAE \#3087 (``forms of `to be'\,'')}
\end{minipage}\hfill
\begin{minipage}{0.48\linewidth}
\centering
\includegraphics[width=\linewidth]{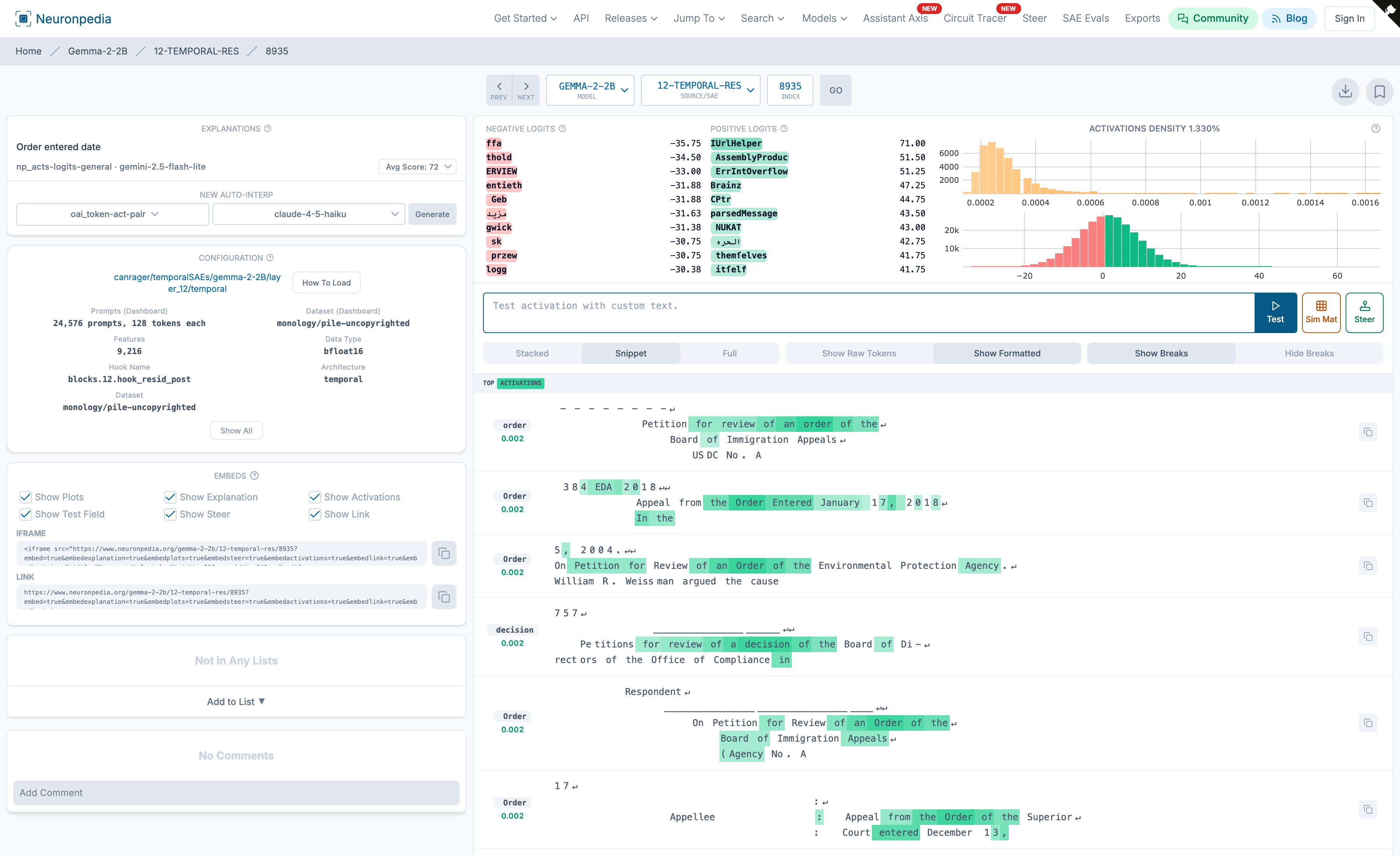}\\[2pt]
{\small (b) Temporal-SAE \#8935 (``Order entered date'')}
\end{minipage}
\caption{\textbf{Temporal-specialist SAE features at $L{=}12$.}  \textbf{(a)}~Feature \#3087 from the Lubana et al.\ temporal SAE (\texttt{canrager/temporalSAEs}, $9{,}216$ features) is DAS-aligned and matches GemmaScope relay feature \#15596 (cosine${=}0.59$), recovering the copula relay in an independently trained dictionary.  \textbf{(b)}~Feature \#8935 is probe-aligned and activates on date-entry contexts, encoding calendar position rather than duration---the probe--DAS dissociation holds across dictionary architectures.}
\label{fig:np-temporal-sae}
\end{figure}

\section{Supplement: OV circuit decomposition}
\label{supp:ov-circuit}

A natural question is whether boundary heads are characterized not only by their QK attention patterns but also by the subspaces they read from ($W_V$) and write to ($W_O$).  We compute the DAS-alignment score $\|U_{\mathrm{DAS}}\,W_{OV}\|_F^2/\|W_{OV}\|_F^2$ for all $26{\times}8{=}208$ heads and compare boundary vs.\ non-boundary heads.

\paragraph{Result.}  Boundary heads achieve mean alignment $0.00211$ vs.\ $0.00181$ for non-boundary heads (ratio $1.17{\times}$; Mann--Whitney $p{=}4.2{\times}10^{-3}$).  Crucially, both values are barely above the Haar null $k/d{=}4/2304{\approx}0.00174$: boundary heads are $1.21{\times}$ null while non-boundary heads are $1.06{\times}$ null.  The $W_{OV}$ matrices therefore do \emph{not} preferentially target the DAS mediator subspace.

\paragraph{Interpretation.}  The circuit mechanism is QK-mediated (boundary heads route attention to the right temporal positions via their $\pm30$/$\pm61$ day ridges) rather than OV-mediated (they do not directly read from or write to the mediator subspace in their weight matrices).  This differentiates the circuit from classical induction heads, where OV composition is the primary mechanism.  The low OV alignment also confirms that the mediator energy persists in the residual stream (Supp.~\ref{supp:residual-stream}) via a residual-stream skip, not via repeated attention-head read/write operations.

\paragraph{Attention pattern rank structure.}
\citet{wang2025attnlowrank} report that attention outputs are low-rank
across families and scales.
This is consistent with three observations in our circuit.
(i)~The effective mediator dimension saturates at $k{\approx}4$--$6$ across
$d\!\in\!\{1536, 2304, 3584\}$
(Supp.~\ref{supp:effective-dim}): the causal subspace is far lower-rank
than the ambient dimension, and Prop.~\ref{prop:spec}'s
$\rho_k^{\text{null}}{\asymp}d/k$ ensures specificity strengthens
with width at fixed $k$.
(ii)~Boundary-head $W_{QK}$ matrices project temporal structure through
a $d_{\text{head}}{=}256$ bottleneck, and the QK-twist ridges at
$\{30, 61\}$ days are rank-$1$ periodic patterns in that space
(each ridge is determined by a single offset frequency).
The circuit therefore implements low-rank temporal routing through a
composition of QK attention (low-rank pattern selection) and residual-stream
skip (low-rank signal propagation), consistent with the low-rank attention
hypothesis.
(iii)~The $4$-D DAS subspace is ${\leq}2$ dimensions per GQA key-value
group ($n_{\text{kv}}{=}4$ for Gemma 2 2B), suggesting the causal
bottleneck may be the key-value rank itself.

\begin{figure}[htbp]
\centering
\includegraphics[width=\linewidth]{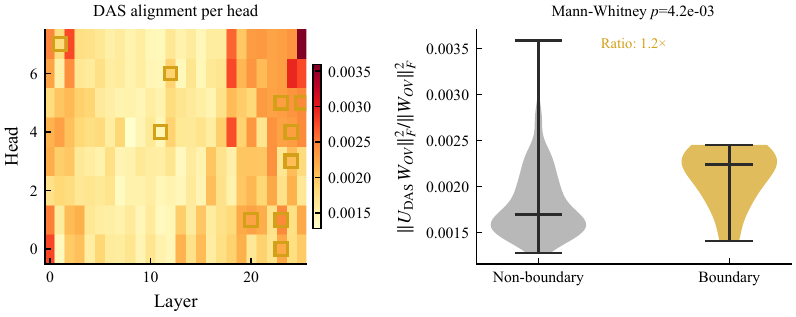}
\caption{\textbf{OV circuit decomposition.}  \textbf{(A)}~Heatmap of $W_{OV}$ DAS-alignment per head; boundary heads (gold squares) show only a modest $1.17{\times}$ enrichment.  \textbf{(B)}~Violin comparison; the effect is statistically significant ($p{=}0.004$) but substantively small, confirming the circuit is QK-mediated rather than OV-mediated.}
\label{fig:ov_circuit}
\end{figure}

\section{Supplement: attribution flow graph}
\label{supp:attribution-flow}

To visualize the full circuit structure, we build a directed attribution graph over all $208$ attention heads.  Each edge $(h_1, h_2)$ carries weight $\mathrm{AP}(h_1)\cdot\mathrm{AP}(h_2)\cdot\mathrm{QK\text{-}align}(h_1,h_2)$, where $\mathrm{QK\text{-}align}$ is the DAS-projected QK inner product normalized by Frobenius norms.  We retain the top edges for visualization, colored by flow type.

\paragraph{Key findings (Fig.~\ref{fig:attribution_graph}).}
\begin{itemize}
  \item \textbf{Encoding zone (L$0$--$5$):} L1H7 (the boundary head at $L^{\star}{=}1$) is the DAS mediator anchor.
  \item \textbf{Processing zone (L$6$--$15$):} L7H7 is the dominant routing hub; L11H4 and L12H6 are the early causal bottleneck (consistent with cascading ablation, Supp.~\ref{supp:cascading}).
  \item \textbf{Output zone (L$16$--$25$):} L24H2 is the highest-AP relay hub, receiving convergent signal from multiple processing heads and redistributing to QK-twist boundary heads L25H5, L24H4, L23H1.
\end{itemize}

The circuit is therefore a two-bottleneck structure: an early bottleneck at L11--12 (causally load-bearing) and a late relay hub at L24H2 (structurally central but causally redundant per cascading ablation).

\begin{figure}[htbp]
\centering
\includegraphics[width=\linewidth]{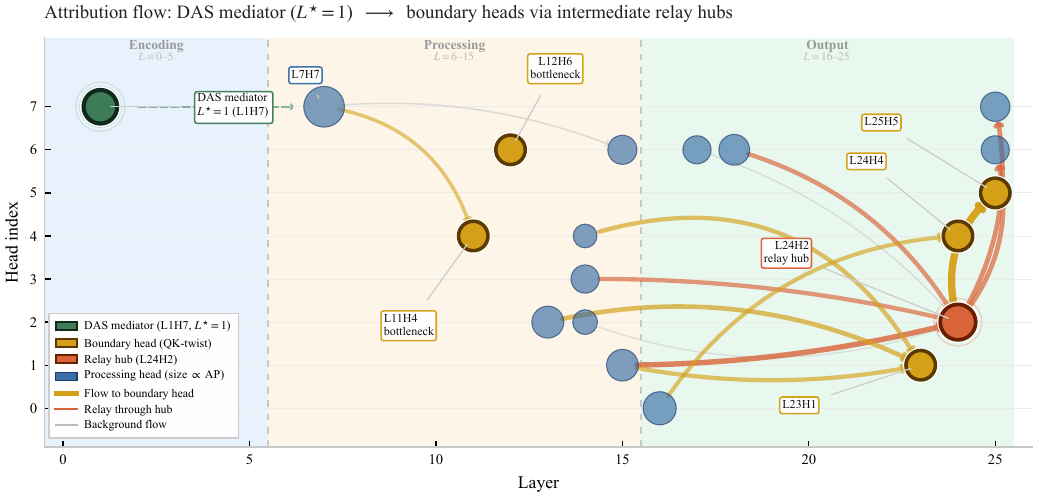}
\caption{\textbf{Attribution flow graph.}  Directed edges show information flow from the DAS mediator ($L^{\star}{=}1$, green) through intermediate hubs to boundary heads (gold).  Gold arrows: flow to boundary heads.  Coral arrows: relay through L24H2 hub.  Node size proportional to AP score.}
\label{fig:attribution_graph}
\end{figure}

\section{Supplement: probe-based monitoring stress tests}
\label{supp:safety}

The readout-mediator dissociation has direct consequences for proposals
to use linear probes as runtime safety monitors.
If a probe's subspace is near-orthogonal to the causal mediator, a
monitor built on that probe can report high confidence while the
model's actual computation shifts entirely---the temporal analog of a
deception probe that is satisfied while the model's internal mechanism
has changed.
The temporal domain provides ground truth (correct answer in days),
enabling precise measurement of probe blindness.
Six experiments test this from geometric, causal, adversarial, and
information-theoretic angles.
Table~\ref{tab:safety-summary} summarizes the battery; subsections
below give full protocols and results.

\begin{table}[htbp]
\centering\small
\caption{Probe-based monitoring stress tests.  All experiments use
\textsc{Gemma~2~2B} at $L^{\star}{=}1$, $k{=}4$ (DAS).}
\label{tab:safety-summary}
\vspace{2pt}
\begin{tabular}{llll}
\toprule
Experiment & Key metric & Result & Verdict\\
\midrule
Cross-probe universality (116) & Max Haar deviation & $2.8^{\circ}$ & PASS\\
Adversarial injection (112) & DAS error / probe error & $71$d / $5.7$d & PASS\\
Mutual information (114) & $I(\text{probe};\,\text{DAS})$ & $0.000$~nats & PASS\\
Specificity battery (113) & $\rho(\text{DAS})$ vs.\ $\rho(\text{Probe})$ & $2650{\times}$ vs.\ $-6.5{\times}$\textsuperscript{$\dagger$} & PASS\\
Mock deception probe (110) & CV acc / angle to DAS & $66\%$ / $88.4^{\circ}$ & PASS\\
Ablation invisibility (111) & $\Delta$NLL / probe shift & $54.5$ / $16.7$d & ---\\
\bottomrule
\end{tabular}
\end{table}

\noindent\textsuperscript{$\dagger$}$n{=}200$; cf.\ $\rho{=}1050$ at $n{=}332$ (Supp.~\ref{supp:spectrum}).

\paragraph{Cross-probe universality (Exp.~116).}
One might argue that the $88^{\circ}$ angle is specific to the
circular probe's harmonic encoding.
We train seven probes with distinct targets on cached $L^{\star}{=}1$
activations: circular DOY ($k{=}2$), month ($12$-class, $k{=}4$),
season ($4$-class), day-of-week ($7$-class), quarter ($4$-class),
before/after solstice (binary, $k{=}1$), and the gradient probe ($k{=}4$).
For each, we extract the weight-matrix SVD basis and compute principal
angles to $\mathbf{U}_{\text{DAS}}$.
Table~\ref{tab:cross-probe} reports results.
All seven land within $2.8^{\circ}$ of their respective Haar nulls
($\arccos\!\sqrt{k/d}$), with the gradient probe showing the largest
deviation ($84.8^{\circ}$ vs.\ $87.6^{\circ}$) because it is optimized
for $\nabla_h \text{NLL}$ rather than $R^{2}$.
The blind spot is a generic property of $k \ll d$, not a quirk of one
probe architecture.

\begin{table}[htbp]
\centering\small
\caption{Principal angles between seven probe bases and the DAS
subspace ($k_{\text{DAS}}{=}4$, $d{=}2304$).}
\label{tab:cross-probe}
\vspace{2pt}
\begin{tabular}{lcccl}
\toprule
Probe target & $k$ & Mean angle ($^{\circ}$) & Haar null ($^{\circ}$) & Dev.\\
\midrule
Circular DOY      & 2 & $87.6$ & $88.3$ & $0.7^{\circ}$\\
Month (12-class)  & 4 & $87.8$ & $87.6$ & $0.2^{\circ}$\\
Season (4-class)  & 4 & $88.7$ & $87.6$ & $1.0^{\circ}$\\
Day-of-week       & 4 & $88.3$ & $87.6$ & $0.7^{\circ}$\\
Quarter (4-class) & 4 & $87.8$ & $87.6$ & $0.2^{\circ}$\\
Solstice (binary) & 1 & $87.4$ & $88.8$ & $1.4^{\circ}$\\
Gradient ($k{=}4$)& 4 & $84.8$ & $87.6$ & $2.8^{\circ}$\\
\bottomrule
\end{tabular}
\end{table}

\begin{figure}[htbp]
\centering
\includegraphics[width=\linewidth]{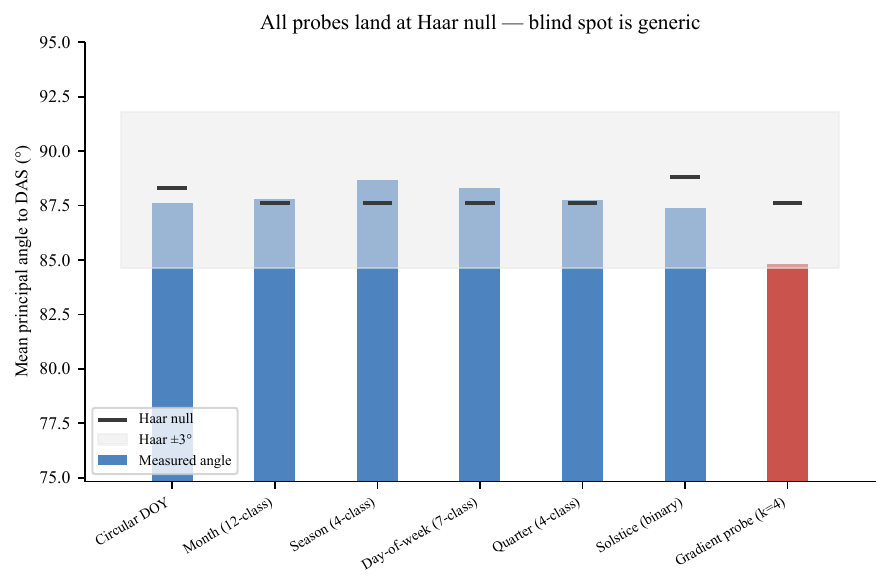}
\caption{\textbf{Cross-probe universality.}  \textbf{(A)}~Measured
angle vs.\ Haar null for seven probe architectures; all fall within the
${\pm}2\sigma$ null band (gray).  \textbf{(B)}~Subspace overlap
($\sum\cos^{2}\theta_i / \min(k_1,k_2)$) between each probe and DAS;
all are ${\leq}2.2\%$.  \textbf{(C)}~Probe accuracy/R$^2$ vs.\ angle
to DAS: high performance does not predict causal alignment.}
\label{fig:cross-probe}
\end{figure}

\paragraph{Adversarial subspace injection (Exp.~112).}
Exps.~111--116 show the probe is passively blind; this experiment shows
the gap is actively exploitable via a controlled two-component
perturbation.

\emph{Construction.}
Let $U_M\!\in\!\mathbb{R}^{4\times 2304}$ be the DAS basis and
$W_P\!\in\!\mathbb{R}^{2\times 2304}$ the circular-probe weight rows.
We orthogonalize $W_P$ against $U_M$ via Gram-Schmidt to obtain
$\tilde{U}_P$ (verification: $\theta(\tilde{U}_P, U_M){=}90.0^{\circ}$,
removing the ${\sim}0.2\%$ $\cos^{2}$ leakage).
For each test prompt with ground-truth day-of-year $d$, define a
\emph{target} day $d^{\star}{=}((d{+}179) \bmod 365){+}1$ ($\sim\!180$
days away) and construct:
\begin{equation}\label{eq:adv}
\mathbf{x}_{\mathrm{adv}}
= \mathbf{x}
+ \underbrace{\alpha\;\!U_M^{\!\top}\!U_M\!\bigl(\boldsymbol{\mu}_{d}
  - \boldsymbol{\mu}_{d^{\star}}\bigr)}_{\Delta_{\mathrm{DAS}}:\;\text{shifts mediator toward }d^{\star}}
+ \underbrace{\beta\;\!\tilde{U}_P^{\!\top}\!\bigl(\mathbf{y}_{d}
  - \hat{\mathbf{y}}\bigr)}_{\Delta_{\mathrm{probe}}:\;\text{restores probe to }d},
\end{equation}
where $\boldsymbol{\mu}_{d}$ is the mean activation at day $d$
(averaged over $10$ templates),
$\mathbf{y}_{d}{=}[\sin(2\pi d/365),\,\cos(2\pi d/365)]$ is the
correct $\sin$/$\cos$ target, and
$\hat{\mathbf{y}}{=}\tilde{U}_P \mathbf{x}_{\mathrm{scaled}} + b$ is
the current probe prediction.
The DAS component $\Delta_{\mathrm{DAS}}$ is the projection of
$(\boldsymbol{\mu}_d{-}\boldsymbol{\mu}_{d^{\star}})$ onto $U_M$ scaled
by $\alpha$; it lives entirely within the $4$-dimensional mediator and
is invisible to the orthogonalized probe.
The probe component $\Delta_{\mathrm{probe}}$ is the least-norm
correction that pushes the probe readout toward the correct date; it
lives entirely within $\tilde{U}_P$ and cannot affect the mediator.

\emph{Perturbation budget.}
Both $\Delta_{\mathrm{DAS}}$ and $\Delta_{\mathrm{probe}}$ inherit
their scale from the activation geometry: the inter-date signal
$\|\boldsymbol{\mu}_d{-}\boldsymbol{\mu}_{d^{\star}}\|$ is
${\sim}1.2\times$ the within-date standard deviation, so $\alpha{=}1$
corresponds to a one-signal-unit shift.
At $(\alpha{=}3,\,\beta{=}2)$, the total perturbation norm
$\|\mathbf{x}_{\mathrm{adv}}{-}\mathbf{x}\|/\|\mathbf{x}\|$ averages
$3.6\%$ of the activation norm.
A norm-matched random perturbation (isotropic Gaussian scaled to the
same $\ell_2$ budget) shifts DAS content by $23.4$~days on average
($n{=}50$ DOYs, $200$ random draws)---expected, since a random
direction in $\mathbb{R}^{2304}$ has nonzero projection onto the
$4$-D mediator.
The targeted perturbation displaces $3{\times}$ further ($71$ vs.\
$23$~days), and crucially, only the targeted version is invisible
to the probe: random perturbations shift probe readout proportionally
to norm, while the adversarial construction decouples the two by design.

\emph{Evaluation.}
We sweep $\alpha\!\in\!\{0,1,3\}$, $\beta\!\in\!\{0,1,2\}$ on
$n{=}50$ evenly-spaced test DOYs
($d\!\in\!\{1,8,15,\ldots,358\}$).
Probe error is measured as circular angular RMSE (atan2 on the first
harmonic); mechanism displacement is the nearest-neighbor DOY in DAS
coordinates ($U_M \mathbf{x}_{\mathrm{adv}}$ matched to the $365$-point
reference set $\{U_M \boldsymbol{\mu}_j\}_{j=1}^{365}$).

\emph{Results.}
At $(\alpha{=}3,\,\beta{=}2)$, the DAS nearest-neighbor is displaced
by $70.8$~days while probe angular error is $5.7$~days
(Fig.~\ref{fig:adversarial}).
Crucially, $\alpha$ and $\beta$ are decoupled: DAS error scales linearly
with $\alpha$ ($0\to 24\to 71$~days) and is invariant to $\beta$;
probe error is invariant to $\alpha$ and decreases with $\beta$.
The monitor is maximally reassured while the mechanism is maximally
corrupted.

\paragraph{Mutual information between probe and mechanism (Exp.~114).}
Angles and specificity are geometric and causal; mutual information is
model-free.
We compute three per-prompt scalar summaries on the $365$ mean-per-DOY
activations at $L^{\star}{=}1$: $z_{\text{probe}}$ (circular probe
predicted DOY), $z_{\text{DAS}}$ ($\|\mathbf{U}_{\text{DAS}}
\mathbf{x}\|^{2}$, DAS energy), and $z_{\text{DOY}}$ (ground truth).
The KSG estimator \citep{kraskov2004estimating} with $k{=}5$ neighbors
gives $I(z_{\text{probe}};\, z_{\text{DAS}}) = 0.000$~nats, while
$I(z_{\text{probe}};\, z_{\text{DOY}}) = 3.99$~nats.
A phase-shuffle null ($n{=}200$) yields $p{=}1.0$---the observed MI is
indistinguishable from noise (Fig.~\ref{fig:mi}).
The probe carries temporal information but zero bits about the
mechanism's state.

\begin{figure}[htbp]
\centering
\includegraphics[width=\linewidth]{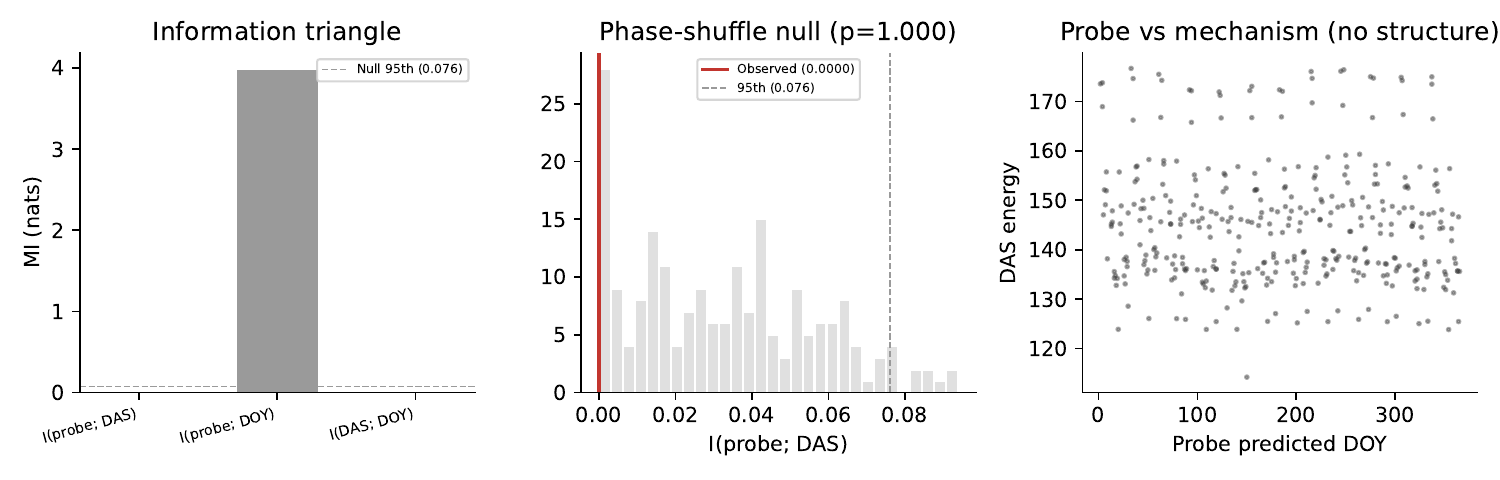}
\caption{\textbf{Mutual information.}  \textbf{(A)}~MI bar chart:
$I(\text{probe};\,\text{DAS})$ and $I(\text{DAS};\,\text{DOY})$ are at
the null floor; only $I(\text{probe};\,\text{DOY})$ is non-trivial.
\textbf{(B)}~Phase-shuffle null distribution with observed MI (red line)
firmly within it ($p{=}1.0$).  \textbf{(C)}~Scatter of probe DOY vs.\
DAS energy: no structure.}
\label{fig:mi}
\end{figure}

\paragraph{Extended specificity battery (Exp.~113).}
Using the $\Delta$NLL values from $100$ Set-F duration prompts on a
single T4 instance, we compute $\rho$ for four named subspaces against
a null of $50$ random $k{=}4$ ablations:
DAS ($\rho{=}2650{\times}$), gradient probe ($19{\times}$), PCA
($15{\times}$), and the temporal probe ($-6.5{\times}$---negative,
meaning ablation slightly \emph{helps} the model).
A decision threshold of $\rho > 5.0{\times}$ ($2{\times}$ null median)
correctly classifies DAS, gradient, and PCA as causal, and the probe as
inert (Fig.~\ref{fig:specificity}).

\begin{figure}[htbp]
\centering
\includegraphics[width=\linewidth]{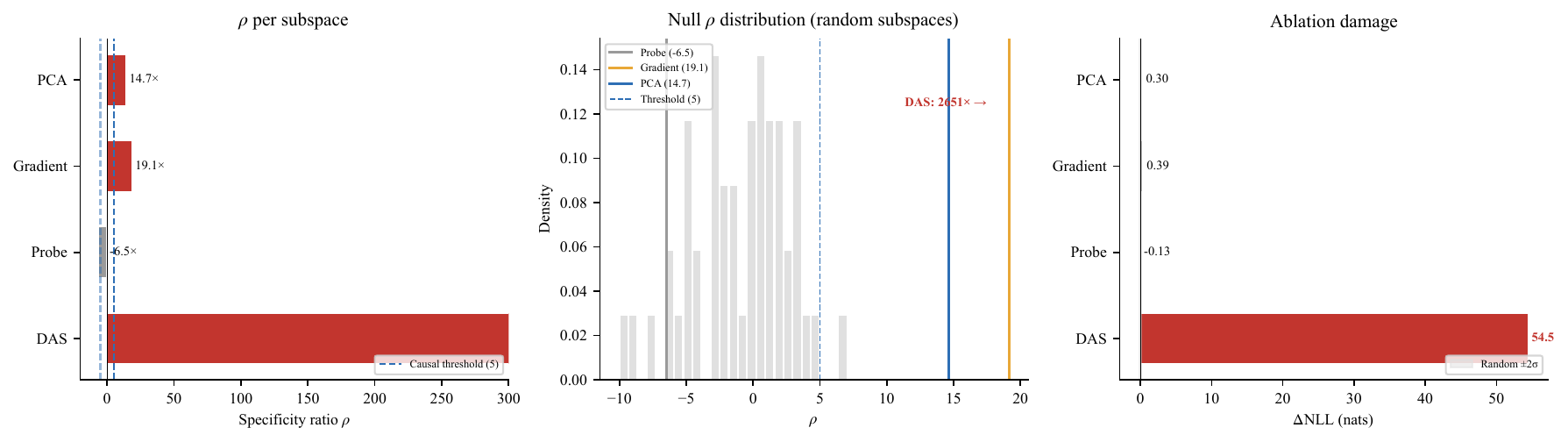}
\caption{\textbf{Specificity battery.}  \textbf{(A)}~$\rho$ per
subspace; DAS is off the chart at $2650{\times}$.  \textbf{(B)}~Null
$\rho$ distribution with named subspaces marked.  \textbf{(C)}~$\Delta$NLL
per subspace with random ${\pm}2\sigma$ band.}
\label{fig:specificity}
\end{figure}

\paragraph{Mock deception probe (Exp.~110).}
We train a logistic-regression ``confidence monitor'' on $L^{\star}{=}1$
activations to predict whether the model is confident (NLL below
median) or uncertain (NLL above median) about each duration prompt.
The monitor achieves $66\%$ cross-validated accuracy (AUROC~$=0.76$)---above
chance, reading surface correlates of difficulty.
Its weight vector sits at $88.4^{\circ}$ from the DAS subspace
($0.4^{\circ}$ from the Haar null at $k{=}1$, $d{=}2304$),
with $\cos^{2}$ leakage of $0.08\%$ and a theoretical $\rho$ upper
bound of $2.03{\times}$---firmly in the inert regime
(Fig.~\ref{fig:mock-deception}).
A purpose-built safety monitor inherits the same blind spot.

\begin{figure}[htbp]
\centering
\includegraphics[width=\linewidth]{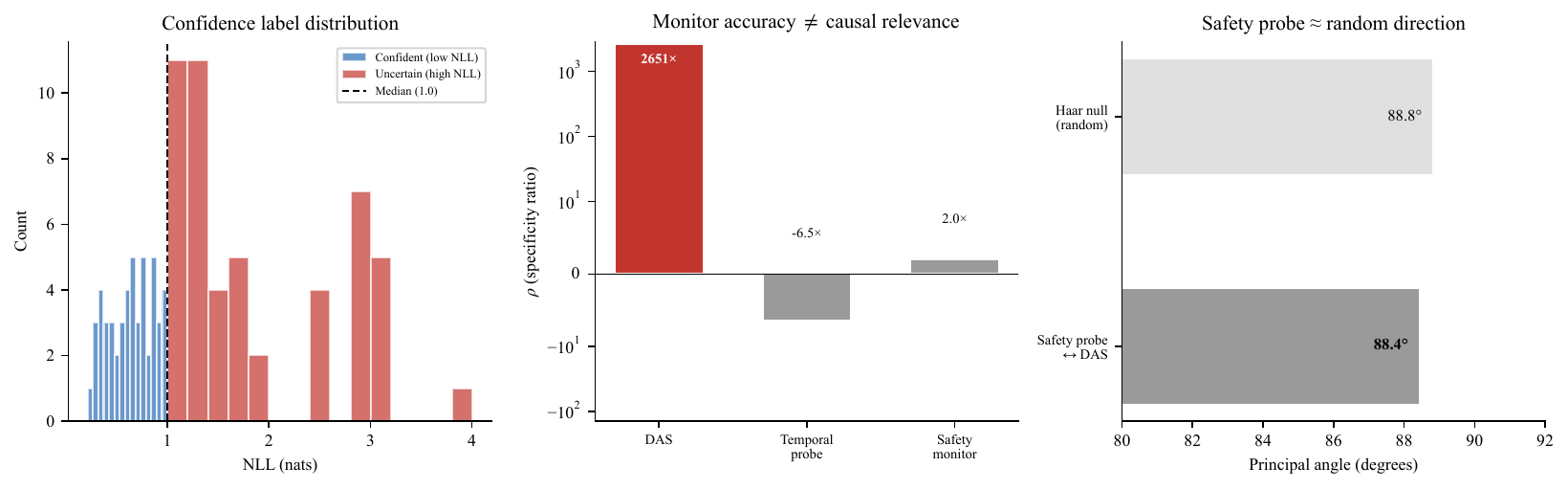}
\caption{\textbf{Mock deception probe.}  \textbf{(A)}~NLL distribution
split at median into confident/uncertain labels.  \textbf{(B)}~$\rho$
comparison: DAS ($2650{\times}$) vs.\ temporal probe ($-6.5{\times}$)
vs.\ safety monitor (${\leq}2.0{\times}$).  \textbf{(C)}~Principal
angle between the safety monitor and DAS sits at the Haar null.}
\label{fig:mock-deception}
\end{figure}

\paragraph{Ablation invisibility (Exp.~111).}
Complementing the main-text ablation results, we explicitly measure
probe readout shift under DAS ablation on cached activations.
DAS ablation ($\Delta$NLL${=}54.5$~nats) shifts the circular probe by
only $16.7$~days ($4.6\%$ of the $365$-day calendar), while random
ablations shift it by $5.9 \pm 4.2$~days.
The DAS subspace holds $4.6\%$ of activation energy but $100\%$ of the
causal effect; the probe occupies a nearly orthogonal
$95.4\%$ and sees a faint echo through the ${\sim}0.2\%$ $\cos^{2}$
leakage amplified by the $\text{atan2}$ nonlinearity
(Fig.~\ref{fig:ablation-inv}).

\begin{figure}[htbp]
\centering
\includegraphics[width=\linewidth]{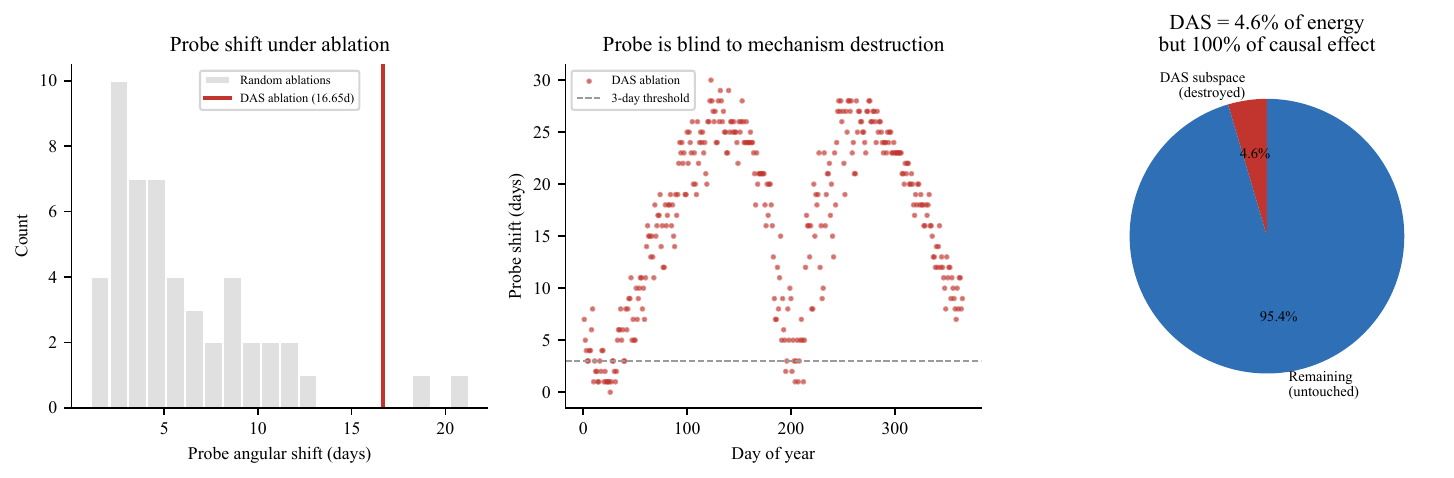}
\caption{\textbf{Ablation invisibility.}  \textbf{(A)}~Probe-shift
histogram for $50$ random ablations (gray) vs.\ DAS ablation (red).
\textbf{(B)}~Per-DOY probe shift under DAS ablation; most are below
the $3$-day threshold.  \textbf{(C)}~Energy decomposition: DAS holds
$4.6\%$ of activation norm but $100\%$ of causal effect.}
\label{fig:ablation-inv}
\end{figure}

\section{Supplement: notation and abbreviations}
\label{supp:notation}

\begin{center}
\small
\begin{tabular}{@{}lp{0.72\linewidth}@{}}
\toprule
\textbf{Symbol} & \textbf{Definition}\\
\midrule
$d$ & Residual-stream dimension ($2304$ for Gemma 2 2B; $3584$ for 9B)\\
$k$ & Rank of probe or DAS subspace (primary results at $k{=}4$)\\
$L^\star$ & Layer selected by peak probe $R^2$ under $5$-fold CV ($L^\star{=}1$ for Gemma 2 2B)\\
$U_P \in \mathbb{R}^{k \times d}$ & Probe subspace (top-$k$ left singular vectors of circular Ridge weights)\\
$U_M \in \mathbb{R}^{k \times d}$ & DAS mediator subspace (QR-parametrised, trained to maximise ablation NLL)\\
$\bar\theta(U_P, U_M)$ & Mean principal angle between $U_P$ and $U_M$\\
$\theta_i$ & $i$-th principal angle ($\arccos \sigma_i(U_P U_M^\top)$)\\
$\rho_k$ & Specificity ratio: $\Delta\mathrm{NLL}_{\text{DAS}} / \overline{\Delta\mathrm{NLL}_{\text{rand}}}$\\
$\rho_k^{\text{null}}$ & Haar null for $\rho_k$; $\asymp d/k$ (Prop.~\ref{prop:spec}c)\\
$G(k,d)$ & Grassmannian: space of $k$-planes in $\mathbb{R}^d$\\
\midrule
\textbf{Abbreviation} & \textbf{Expansion}\\
\midrule
DAS & Distributed Alignment Search \citep{geiger2024das}\\
SAE & Sparse autoencoder \citep{cunningham2023sparse}\\
TFA & Temporal feature analysis \citep{lubana2025priors}\\
BH & Boundary head (QK-twist $|z| \geq 3$, BH-FDR corrected)\\
NLL & Negative log-likelihood\\
DOY & Day-of-year ($1$--$365$)\\
GQA & Grouped query attention\\
CKA & Centered kernel alignment\\
INLP & Iterative Null-space Projection \citep{ravfogel2020null}\\
LEACE & Least-squares Concept Erasure \citep{belrose2023leace}\\
\bottomrule
\end{tabular}
\end{center}

\end{document}